\documentclass[journal]{IEEEtran} 
\usepackage{cite}
\usepackage{microtype}
\usepackage{graphicx}
\usepackage{subfigure}
\usepackage{booktabs} % for professional tables

\usepackage[pagebackref,breaklinks,colorlinks]{hyperref}

\usepackage{amsmath}
\usepackage{amssymb}
\usepackage{mathtools}
\usepackage{amsthm}

\theoremstyle{plain}
\newtheorem{theorem}{Theorem}[section]
\newtheorem{proposition}[theorem]{Proposition}
\newtheorem{lemma}[theorem]{Lemma}

\theoremstyle{definition}

\theoremstyle{remark}
\newtheorem{remark}[theorem]{Remark}

\usepackage[textsize=tiny]{todonotes}

\usepackage{makecell}
\usepackage{color}
\usepackage[capitalize]{cleveref}
\usepackage[numbers]{natbib}
\usepackage{multirow}
\usepackage{verbatim}
\usepackage{algorithm}
\usepackage{algorithmic}
\newcommand{\eat}[1]{}
\newcommand{\etal}{{\em et al.~}} % \etal
\newcommand{\eg}{{\em e.g.,~}}     % e.g.
\newcommand{\ie}{{\em i.e.,~}}      % i.e.
\usepackage{amssymb}
\usepackage{threeparttable}

\begin{document}

\title{Feature Noise Boosts DNN Generalization under Label Noise}
\author{Lu Zeng, Xuan Chen, Xiaoshuang Shi, Heng Tao Shen,~\IEEEmembership{Fellow,~IEEE}}
\maketitle

\begin{abstract}
The presence of label noise in the training data has a profound impact on the generalization of deep neural networks (DNNs). In this study, we introduce and theoretically demonstrate a simple feature noise method, which directly adds noise to the features of training data, can enhance the generalization of DNNs under label noise.
Specifically, we conduct theoretical analyses to reveal that label noise leads to weakened DNN generalization by loosening the PAC-Bayes generalization bound,
and feature noise results in better DNN generalization by imposing an upper bound on the mutual information between the model weights and the features, which constrains the PAC-Bayes generalization bound.
Furthermore, to ensure effective generalization of DNNs in the presence of label noise, we conduct application analyses to identify the optimal types and levels of feature noise to add for obtaining desirable label noise generalization.
Finally, extensive experimental results on several popular datasets demonstrate the feature noise method can significantly enhance the label noise generalization of the state-of-the-art label noise method.
\end{abstract}


% \listoftodos

\IEEEpeerreviewmaketitle

\section{Introduction}
\label{sec:intro}

Label noise, a prevalent problem in machine learning, refers to the presence of mislabeled samples in the training dataset used to train machine learning models. In the context of deep learning, where over-parameterization of deep neural networks (DNNs) is common, the label noise issue becomes particularly significant. It is worth noting that DNNs with a sufficient number of parameters can potentially fit random label data given adequate training \cite{zhang2021understanding}. 
There are two popular strategies to address the label noise issue, \ie denoising-based methods, and robustness-based methods.

Denoising-based methods, which aim to identify mislabeled samples in the training set and mitigate their negative effects on DNNs, are widely applied to handling the label noise issue.
The denoising-based methods rely on the information they have to differentiate the mislabeled samples from the correctly labeled samples. 
Based on the information they utilize, they can be roughly classified into clean-sample-based methods and small-loss-sample-based methods. 
Clean-sample-based methods utilize a small set of clean samples to provide information to distinguish the wrongly and correctly labeled samples. 
The information in the clean samples can be extracted by curriculum learning \cite{jiang2018mentornet}, meta-learning \cite{ren2018learning,shu2019meta}, or the anchor points in the label transition matrix \cite{patrini2017making,yao2020dual}. 
One limitation of the clean-sample-based methods is that the clean samples cannot be easily obtained in some datasets, such as medical images \cite{karimi2020deep}.
Small-loss-sample-based methods rely on the early robustness characteristic of DNNs to identify the wrongly and correctly labeled samples. They are inspired by the observation that, without any regularization, DNNs themselves can be fairly robust to even massive label noise during the early training stage \cite{rolnick2017deep}. 
These methods rely on the judgment of DNNs in the early training stage. 
For instance, co-teaching methods \cite{han2018co,chen2021boosting} rely on the small-loss samples in the early training stage to train the networks, early-learning-regularization \cite{liu2020early} utilizes a regularization term to penalize the samples that have large prediction variation between the early and late training stages.
Overall, the denoising-based methods are mostly heuristic in nature, and provide little insight into the understanding of the label noise issue.

Robustness-based methods, which can enhance label noise generalization without paying special attention to noisy labels, aim to boost the overall robustness of DNNs.
They often involve noisy training, \ie adding noise to the training process.
The random noise can be introduced to either the training data or the training algorithm.
Adding noise to training data is a form of data augmentation that is widely recognized for its effectiveness in enhancing model robustness. This is mainly attributed to the enlarged data diversity, which helps reduce model overfitting. However, there lacks a comprehensive study about the theoretic relationship between feature noise and model overfitting.
On the other hand, model robustness techniques in deep learning involve adding noise to the training algorithm.
For example, stochastic gradient descent (SGD) and dropout can boost DNNs' robustness in general.
Studies have shown that the effectiveness of these robustness techniques is closely related to the generalization bound \cite{wu2022does,achille2018information,rippel2014learning}.
In this paper, we explore the impact of adding noise to training features on the generalization bound. Through theoretical and empirical analysis, we demonstrate that the simple feature noise (FN) method greatly improves DNN generalization in the presence of massive label noise. Our study has two notable contributions compared to previous robustness-based methods. First, we establish that feature noise, a commonly used data augmentation technique, significantly enhances label noise generalization, which has not been proposed before. Second, we uncover the close relationship between the effectiveness of the FN method and the generalization bound, which coincides with the theoretical analysis of the model robustness techniques.

This paper provides thorough theoretical analyses, along with extensive empirical evidence, to demonstrate the positive effect of FN on DNNs' label noise generalization. 
Specifically, we initially establish that DNNs trained with label noise possess loose generalization bounds. Subsequently, we prove that the FN method can effectively constrain the generalization bounds. Additionally, we conduct application analyses to explore the types and levels of feature noise that should be injected to achieve favorable label noise generalization of DNNs.

%contributions of our work
Our contributions are listed as follows:

\begin{itemize}
    \item We propose an innovative approach, \ie adding feature noise, to increase the generalization of DNNs under label noise. This straightforward approach offers a simple and cost-effective solution to address the label noise issue, while also being adaptable to different models. % an innovative approach with pros
    \item The effectiveness of feature noise in boosting the label noise generalization of DNNs sheds new light on the relationship between the generalization characteristics of DNNs and the impact of noise. % insights into the relationship between generalization and noise
    \item We theoretically prove the effectiveness of the FN method in constraining the generalization bound of DNNs under label noise, providing a solid theoretical foundation for using feature noise in the label noise issue. % theory
    \item Extensive experiments demonstrate the simple FN method can obtain comparable and even superior performance over existing state-of-the-art methods.  
    This emphasizes the practicality of using the FN method to improve the label noise generalization of DNNs. % experiments verify the effectiveness
\end{itemize}

The rest of this paper is organized as follows. 
In Section \ref{sec:rw}, we review the current label noise methods, including denoising-based methods and robustness-based methods, and the generalization bound studies. 
In Section \ref{sec:pre}, we introduce some preliminaries of our theoretical analysis.
We introduce our method in Section \ref{sec:meth} and analyze how it affects the PAC-Bayes generalization bound in Scetion \ref{sec:theo}. 
In Section \ref{sec:app}, we conduct qualitative and quantitative analyses to discuss the appropriate type and level of feature noise to be added.
In Section \ref{sec:ex}, we conduct extensive experiments on four real-world datasets to validate the effectiveness and the theoretical advantages of the FN method.
Finally, we conclude this paper in Section \ref{sec:con}.

\section{Related Work}\label{sec:rw}
In this section, we comprehensively review two categories of label noise methods, \ie denoising-based methods and robustness-based methods, and show their relationships with the FN method.
Additionally, we review the studies of the generalization bound to highlight the contribution of this work.

\subsection{Denoising-based methods}
Most label noise methods focus on denoising, which reduces the impact of mislabeled samples during model training. Denoising-based methods rely on specific information to accurately distinguish between wrongly and correctly labeled samples. The success of these methods depends on the accuracy of this information, which should effectively capture the distinctions between the two types of samples. By penalizing the wrongly labeled samples using this information, denoising-based methods can achieve favorable performance. These methods can be categorized into two groups based on the source of information used: clean-sample-based methods and small-loss-sample-based methods.

Clean-sample-based methods employ a limited number of correctly labeled samples, known as clean samples, to effectively denoise the noisy labels. 
For example, some clean-sample-based methods first pre-train a teacher model on the auxiliary clean samples and then correct the wrong labels by the teacher model's predictions \cite{li2017learning},  or reweight the noisy samples in the total loss by a sample weighting scheme provided by the teacher model \cite{jiang2018mentornet}.

Small-loss-sample-based methods are inspired by the observation that DNNs exhibit inherent resilience to label noise in the early training stage \cite{rolnick2017deep,arpit2017closer,zhang2021understanding}, thereby the small-loss samples in the early training stage are likely to be the clean ones.
Leveraging this phenomenon, small-loss-sample-based methods employ various strategies to enhance label noise generalization. For instance, they might reduce the influence of samples with significant prediction variation between the early and late training stages \cite{liu2020early}, discard mislabeled samples in the late training stage when the DNNs' robustness to label noise diminishes \cite{han2018co,yu2019does,chen2021boosting}, or terminate the training process early to maintain the early robustness of DNNs \cite{li2020gradient,bai2021understanding}.

However, existing denoising-based methods still have certain limitations. First, obtaining clean samples can be challenging in real-world applications. Second, the identification of mislabeled samples often requires complex algorithms and additional computational costs. For instance, clean-sample-based methods involve meta learning \cite{ren2018learning,shu2019meta,zheng2021meta} to harness information from clean samples, while some small-loss-sample-based methods involve semi-supervised learning \cite{liu2020early,bai2021understanding} to leverage early robustness.
Furthermore,  although denoising-based methods are effective in addressing the label noise issue, they possess heuristic nature. 
These methods have limited interrelation and provide little insight into the mechanisms that govern generalization in the presence of label noise. 
The FN method and denoising-based methods share a common goal of addressing the label noise issue. 
However, there are differences between them. First, the FN method is notable for its simplicity and efficiency, as it doesn't need extra clean samples or additional computation costs. Second, denoising-based methods often rely on heuristic approaches, but the FN method directly constrains the generalization bound.

\subsection{Robustness-based methods}
Robustness-based methods boost the overall robustness of DNNs without paying special attention to the noisy labels.
They often involve noisy training, \ie adding noise to the training process.
The methods that add noise to the training data belong to noisy data augmentation, and the methods that add noise to the training algorithm belong to model robustness techniques.

%review da
Noisy data augmentation methods involve adding random or specific noise, \eg adversarial noise, to the features in order to enhance model robustness. These methods demonstrate that models trained on samples with feature noise exhibit greater resilience to feature noise compared to those trained on clean samples. For instance, Yin \etal \cite{yin2019fourier} showed that DNNs trained with feature noise tend to be robust to other feature noises with similar frequency domain properties. Additionally, many studies on adversarial samples have shown that models trained with adversarial noise are more resilient to adversarial samples \cite{goodfellow2015explaining,hendrycks2021natural,cohen2019certified}.
These studies in data augmentation primarily suggest that feature noise can boost robustness towards the same type or similar types of feature noise. Remarkably, in this study, we demonstrate that feature noise can also enhance DNN robustness against label noise. 
This observation is counterintuitive because the supervised learning scheme does not reward this behavior. Therefore, we assume the robustness gains of DNNs trained with feature noise is a topic that needs closer investigation.

%review ir
The effectiveness of model robustness techniques in boosting the robustness of DNNs has many explanations. One popular explanation is that the noise introduced to the training process results in a flat minima of the loss landscape \cite{wu2022does,keskar2016large,zhang2022implicit}.
The theory of flat minima is conducted from the optimization point of view. While statistically, the effectiveness of model robustness techniques can also be explained by the generalization bound theory.
For example, the effectiveness of many common DNN robustness techniques, \ie dropout \cite{mcallester2013pac}, $l$-2  regularization \cite{mcallester2013pac}, and SGD \cite{dziugaite2018entropy}, can all be explained by the tightened PAC-Bayes generalization bound.
On the contrary, the effectiveness of the other noisy training method, \ie the data augmentation, lacks a comprehensive explanation.
To the best of our knowledge, our study is the first to provide an explanation of how feature noise influences the label noise generalization of DNNs within the PAC-Bayes framework.
In the following, we thoroughly review the studies of the generalization bound to further explain the contribution of this work.

\subsection{Generalizaiton bound}
%%% introduce generalizaiton bound and reduce bound
Boosting the generalization of DNNs often involves reducing the generalization bound.
There are  two types of generalization bounds: uniform and non-uniform.
Uniform bounds, like the VC-bound proposed by Vapnik \etal \cite{vapnik1991principles}, cover the entire hypothesis space and are designed to constrain the generalization gap for the worst-case hypothesis. 
Non-uniform bounds, like the PAC-Bayes bound introduced by McAllester \cite{mcallester1999some}, are defined at the level of individual hypotheses. 
The generalization bound serves two main purposes: evaluating the generalization of a learned hypothesis and guiding the design of learning algorithms that yield hypotheses with small generalization bounds. In the realm of deep learning, the latter function is particularly important as it involves trading the complexity of DNNs for better generalization. We refer to these algorithms that explicitly or implicitly reduce the generalization bound during training as bound-reducing methods.
Studies have shown that bound-reducing methods can effectively boost the generalization of DNNs \cite{dziugaite2017computing,neyshabur2018towards}. 
Only the non-uniform bound can be reduced by the training algorithms. Among the various non-uniform generalization bounds, the PAC-Bayes bound is known to be the tightest \cite{foong2021tight}.

%%% introduce reducing pb specifically
The PAC-Bayes generalization bound  provides an upper bound on the gap between the expected risk and the empirical risk of machine learning models. 
Studies have demonstrated that reducing the PAC-Bayes generalization bound can improve the generalization of DNNs by controlling the overfitting level \cite{xu2017information,neyshabur2017exploring,xie2021artificial}. The value of the PAC-Bayes bound is primarily determined by the Kullback-Leibler (KL) divergence between the prior distribution $P$ and the posterior distribution $Q$ over the hypothesis space, denoted as ${\rm D}(Q||P)$. In deep learning, without any regularization, the PAC-Bayes bound can become vacuous, since ${\rm D}(Q||P)$ can be very large \cite{dziugaite2017computing}. Bound-reducing algorithms aim to find a posterior distribution $Q$ that minimizes ${\rm D}(Q||P)$. This can be achieved either by directly minimizing the bound or by constraining it with a tighter upper bound.

% minimize methods review
Algorithms that minimize the PAC-Bayes bound typically employ implicit regularization techniques, which reduce ${\rm D}(Q||P)$ during training. For instance, weight decay (also known as $l$-2 regularization) and dropout implicitly minimize the PAC-Bayes bound, by assuming a standard Gaussian distribution for the prior $P$ \cite{mcallester2013pac}. 
However, these bound-minimizing algorithms have a drawback, \ie the requirement to choose a prior $P$, which is often done randomly for convenience. Consequently, the bound can become loose, rendering it unreliable for assessing whether the algorithm can effectively reduce the generalization error by minimizing the generalization bound.
Achille \etal \cite{achille2018emergence} demonstrate that $I(W;S)$ provides the tightest bound for ${\rm D}(Q||P)$ among all possible priors $P$. 
Nevertheless, computing the tightest bound $I(W;S)$ is still challenging if one aims to minimize it. As such, many studies resort to minimizing an upper bound of this term instead. For instance, the limiting label information memorization in training (LIMIT) method \cite{harutyunyan2020improving}  constrains $I(W;S)$ by substituting the model weights with the gradient of the loss with respect to the weights. Nevertheless, approximating the upper bound of $I(W;S)$ still requires many computation costs.

% constrain methods review? introduce ours
Fortunately, the computation costs can be bypassed by constraining the bound instead of minimizing it. In this paper, we introduce a simple method that constrains the upper bound of the sharpest PAC-Bayes bound by adding noise to the training features. To the best of our knowledge, this is the first study that employs a tighter upper bound to constrain the PAC-Bayes bound for achieving improved label noise generalization.
The FN method overcomes the limitations of bound-minimizing methods, namely the bias introduced by manually selecting the prior $P$ and the computational effort required to compute the PAC-Bayes bound.

\section{Preliminaries}\label{sec:pre}
In this section, we introduce some information theoretic quantities that are frequently used in this paper.
Then we introduce the PAC-Bayes generalization bound theorem.
Finally, we propose a novel decomposition of the cross entropy loss inspired by \cite{achille2018emergence}.

\subsection{Information theoretic quantities}
The discrete Shannon entropy $H(X)=-{\mathbb E}\ln{p(X)}$; The differential entropy $h(f)=-\int_Af(X)\ln f(X) dx$, where $f$ and $A$ denote the probability density function and the support set of $X$; The conditional entropy $H(Y|X)=-{\mathbb E}\ln{p(Y|X)}$; The  Kullback-Leibler (KL) divergence ${\rm D}(q|| p)$ from the distribution $p$ to $q$; The chain rule for mutual information $I(X;Y|Z)=H(X|Z)-H(X|Y,Z)$; The Markov chain $X\to Y\to Z$ that ensures $p(x,y,z)=p(x)p(y|x)p(z|y)$, which means $X$ and $Z$ are conditionally independent given $Y$; The data processing inequality (DPI) $I(X;Y)\geq I(X;Z)$ for a Markov chain $X\to Y\to Z$; The (noisy) communication channel that transmits signals from the input $X$ to the output $Y$; The channel capacity $C={\max_{p(X)}}I(X;Y)$ of a discrete memoryless communication channel that denotes the logarithm of the number of distinguishable signals. For consistency, we utilize $\ln$ to measure the Shannon information, \ie the unit in ``nats''.

\subsection{PAC-Bayes generalization bound}
\begin{lemma}[PAC-Bayes generalization bound  \cite{mcallester1999some,guedj2019primer}]
\label{lm:pac}
For any real $\delta \in (0,1)$, with probability at least $1-\delta$, over the draw of the training dataset $S$, the generalization gap between the expected risk $L(Q)$ and the empirical risk $\hat{L}(Q)$ has the following inequality hold for any distribution $Q$:
\begin{equation}
\label{eq:gen_bound}
|L(Q) - \hat L(Q)| \le \sqrt {\frac{{\rm D}(Q||P)+\ln(\frac{2\sqrt{m}}{\delta})}{2m}},
\end{equation}
where $m$ denotes the number of training samples.
\end{lemma}

The PAC-Bayes generalization bound is determined by the KL-divergence from the prior distribution $P$ to the posterior distribution $Q$ over the hypothesis space, \ie ${\rm D}(Q|| P)$.
${\rm D}(Q|| P)$ has the following inequality hold:
\begin{equation}
    I(W;S)\le\mathbb{E}_{S\sim \mathbb{S}}{\rm D}(Q|| P),
\end{equation}
where $S$ and $\mathbb{S}$ denote the training samples and the unknown sample space, respectively. The equality holds when $P$ equals the marginal distribution $q(W)$ \cite{achille2018emergence}.
Therefore, the sharpest PAC-Bayes generalization bound is achieved by $I(W;S)$. This means that the best prior distribution $P$ is the one that predicts the posterior.
Therefore, limiting $I(W;S)$ is an effective way to tighten the PAC-Bayes generalization bound.
By replacing ${\rm D}(Q|| P)$ with $I(W;S)$ in Eq. (\ref{eq:gen_bound}), we define the sharpest PAC-Bayes generalization bound to be 
\begin{equation}
\label{eq:sharp}
    {\cal B}^*=\sqrt {\frac{I(W;S)+\ln(\frac{2\sqrt{m}}{\delta})}{2m}}.
\end{equation}

Furthermore, $I(W;S)$ can be decomposed to $I(W;S)=I(W;Y|X)+I(W;X)$ in Eq. (\ref{eq:sharp}) to have
\begin{equation}
\label{eq:bounddecom}
    {\cal B}^*=\sqrt {\frac{I(W;Y|X)+I(W;X)+\ln(\frac{2\sqrt{m}}{\delta})}{2m}}.
\end{equation}

An important application of the generalization bound is to help design a learning algorithm that outputs $Q$ with a small generalization bound. 
This purpose can be achieved by reducing the PAC-Bayes generalization bound, which involves minimizing or constraining $I(W;S)$.

The PAC-Bayes generalization bound in Eq. (\ref{eq:gen_bound}), denoted as $f(Q,P,m,\delta)$, is a function that depends on various factors, including $Q$, $P$, $m$ and $\delta$. 
% Here, $Q$ and $P$ represent the posterior and prior distributions over the hypothesis space, respectively. $m$ denotes the number of training samples, and $\delta$ is the probability parameter associated with PAC-learning \cite{valiant1993view}.
$f(Q,P,m,\delta)$ exhibits monotonic relationships with variables $m$ and $\delta$, while it is more difficult to predict for $Q$ and $P$. Previous research has shown that the characterization of the PAC-Bayes generalization bound is strongly influenced by the choice of prior $P$. Selecting an inappropriate prior can result in a vacuous bound \cite{dziugaite2017computing}. 
Achille \etal \cite{achille2018emergence} demonstrate that the prior $P$ that yields the tightest PAC-Bayes bound is the marginal distribution of $W$ from the joint distribution $q(W,S)$ learned by the model. Although computing this prior remains challenging, we can now substitute the KL-divergence term ${\rm D}(Q||P)$ in the original PAC-Bayes bound with $I(W;S)$ to obtain the tightest generalization bound. Furthermore, minimizing $I(W;S)$ offers a more philosophically meaningful interpretation than ${\rm D}(Q||P)$ in the context of DNN training, as $I(W;S)$ quantifies the amount of information from the training data stored in the DNN weights.
While both $Q$ and $P$ are probability distributions over the hypothesis space, they differ fundamentally in practice. $P$ is manually chosen, whereas $Q$ can only be learned as an output of a learning algorithm.
Given that the other three variables in $f(Q,P,m,\delta)$ are fixed, the value of $f(Q,P,m,\delta)$ is solely determined by $Q$. When considering the sharpest bound, namely $I(W;S)$, we can express $f(Q,P,m,\delta)$ as $f(Q(W))= g(I(W;S))$. This quantity possesses a natural upper bound, $g(\min(H(W),H(S)))$. By introducing a tighter upper bound, we can ensure that the learned hypothesis results in a smaller generalization bound.

\subsection{Cross-entropy decomposition}
According to \cite{achille2018emergence}, the expected cross-entropy  loss for a model parametrized by $W$ and trained on the dataset $S=(X,Y)$ can be written as:
\begin{align}
\label{eq:ce}
    &H_{p,q}(Y|X,W)=\\\nonumber
    &H_p(Y|X,W)+\mathbb{E}_{X,W}{\rm D}(p(Y|X,W)||q(Y|X,W)),
\end{align}
where $p$ denotes the real class distribution and $q$ denotes the approximate class distribution optimized by the DNNs.

Different from the decomposition in \cite{achille2018emergence}, Eq. (\ref{eq:ce}) can be decomposed by the following Lemma.

\begin{lemma}\label{lm:cedecommy}
    The expected cross-entropy loss can be decomposed as:
    \begin{align}
    \label{eq:cedecommy}
         H_{p,q}(Y|X,W)=&H_p(Y)-I(Y;X|W)-I(Y;W|X)\\\nonumber
         &-I(W;X)+I(W;X|Y)\\\nonumber
         &+\mathbb{E}_{X,W}{\rm D}(p(Y|X,W)||q(Y|X,W)).
    \end{align}
\end{lemma}
\begin{proof}
    Based on Eq. (\ref{eq:ce}), proving Eq. (\ref{eq:cedecommy}) is equivalent to proving $H_p(Y|X,W)=H_p(Y)-I(Y;X|W)-I(Y;W|X)-I(W;X)+I(W;X|Y)$.
    This can be easily obtained by using the chain rule for information twice:
    \begin{align}
        H_p(Y|X,W)=&H_p(Y)-I(Y;X|W)-I(Y;W|X)\\\nonumber
        &-I(W;Y;X)\\
        =&H_p(Y)-I(Y;X|W)-I(Y;W|X)\\\nonumber
        &-I(W;X)+I(W;X|Y).
    \end{align}
\end{proof}

Eq. (\ref{eq:cedecommy}) has two important meanings. 
First,  
Eq. (\ref{eq:cedecommy}) can be rewritten as $H_{p,q}(Y|X,W)=H_p(Y)-I(Y;X|W)-I(W;S)+I(W;X|Y)+\mathbb{E}_{X,W}{\rm D}(p(Y|X,W)||q(Y|X,W))$, where $I(W;S)$ is a negative term, which means that minimizing the cross-entropy loss is maximizing the sharpest PAC-Bayes generalization bound.
Therefore, training DNNs with the cross-entropy loss will lead to loosened generalization bounds and cause bad generalization.
Thus it is important to design bound-reducing algorithms to reduce the generalization bound when training DNNs with the cross-entropy loss.
Second, although the three negative terms, \ie $I(Y;X|W)$, $I(Y;W|X)$, and $I(W;X)$, are all maximized while the cross-entropy loss is minimized, $I(W;X)$ is the most crucial term  because it has the largest upper bound.
$I(Y;X|W)$ and $I(Y;W|X)$ are both smaller than $H(Y)$. Since $Y$ is a discrete random variable that represents the labels of the data, it has $I(Y;X|W)\le H(Y) \le \ln c$ and $I(Y;W|X)\le H(Y) \le \ln c$, where $c$ is the number of all classes. However, the training feature $X$ and the model weight $W$ are two continuous random variables, whose differential entropy can be infinite, and the mutual information can also be infinite. Therefore, $I(Y;X|W)$ and $I(Y;W|X)$ are bounded by $\ln c$, while $I(W;X)$ can grow to infinity.

In the following, we introduce the feature noise method, which injects noise into the training features and reduces the generalization bound by constraining $I(W;X)$.

\section{Method}\label{sec:meth}

In this paper, we introduce a novel bound-reducing method, \ie the feature noise (FN) method, that adds noise to the features of the training data to constrain the generalization bound through the mutual information $I(W;X)$.
Specifically, given a training dataset $\widetilde{S} = (X, \widetilde{Y})$  with
$m$ i.i.d samples drawn from the unkown data distribution, where $X \in \mathbb{R}^{m \times n}$ denotes the $n$-dimentional features and $\widetilde{Y} \in \mathbb{R}^{m \times 1}$ denotes the labels (with label noise)  of $m$ samples. The FN method adds  noise  (\eg  the multivariate Gaussian noise $\mathbf{z} \sim {\cal N}(\mathbf{0},\sigma^2I_{n \times n})$) into each feature vector $\mathbf{x}$ of the original feature matrix by
\begin{equation}
    \label{eq:method}
    \widetilde{\mathbf{x}} = \mathbf{x} + \mathbf{z}.
\end{equation}
After that, we train a DNN on the dataset $\hat{S} = (\widetilde{X},\widetilde{Y})$.

% advantages
The FN method has two major advantages. First, it is very simple, requiring little extra computation costs, and can be easily applied to various models.
Second, limiting the information in the training features is a novel methodology to increase the label noise generalization, which can be generalized to build many other methods.

In addition to the FN method, in this paper, we mainly focus on the  theoretical analysis of the effectiveness of the FN method under label noise, and application analysis of how to ensure good label noise generalization, in Sections \ref{sec:theo} and \ref{sec:app}, respectively.

\section{Theoretical Analysis}\label{sec:theo}
%% introducing the analysis AND  the analysis settings
In this section, we first prove that adding label noise to the training data loosens the PAC-Bayes generalization bound, which causes worse generalization than training without label noise in Section \ref{sec:ln}.
And we prove that adding Gaussian noise to the training features can improve the generalization of DNNs under label noise by constraining the PAC-Bayes generalization bound in Section \ref{sec:fn}.
We establish theoretical proofs within the following framework. We consider the training dataset denoted as $S=(X,Y)$, where the training features $X$ can be viewed as a Gaussian mixture model (GMM). Each component of the GMM represents an independent multivariate distribution associated with a particular class $y\in Y$.

We add symmetric noise to $Y$ to simulate the label noise issue.
Specifically, denoting $\widetilde{Y}=\{\widetilde{y}^{(1)},\cdots,\widetilde{y}^{(m)}\}$, the element of  $\widetilde{Y}$ is
\begin{equation}
\label{eq:symnoise}
{\widetilde y^{(i)}} = \left\{
{\begin{array}{ccc}
y^{(i)}\quad&\text{with probability }& 1-\Delta\\
U      \quad&\text{with probability }& \Delta
\end{array}} \right.
\end{equation}
where $\Delta \in [0,1]$ denotes the corruption rate. $U\sim {\cal U}\{ 0,c - 1\}$ is a random variable drawn from a discrete uniform distribution, where $c$ denotes the number of all classes.

\subsection{Label noise}\label{sec:ln}
In this part, we conduct theoretical analysis to demonstrate the symmetric label noise loosens the PAC-Bayes generalization bound, and will harm DNNs' generalization in Theorem \ref{thm:1}.

%%% label noise loosens gb
Lemma \ref{lm:sym_increase} establishes that adding symmetric label noise increases the entropy $H(Y)$ of the labels $Y$.

\begin{lemma}%[{\cm proved in Appendix B}]%\ref{ap:ta}
\label{lm:sym_increase}
Adding symmetric noise to 
the labels $Y$ increases the entropy, i.e., $H(Y) \le H(\widetilde{Y})$.
\end{lemma}
\begin{proof}
Letting $\widetilde{Y}_t \in \mathbb{R}^{m \times 1}$ be the set of $m$ labels with $t$ labels corrupted by the symmetric noise (\ie $t=[\Delta m]$), we have a Markov chain as follows:
\begin{equation}
\label{eq:inrr}
    \widetilde{Y}_0 \to \widetilde{Y}_1 \to \cdots \to \widetilde{Y}_m.
\end{equation}
 
The relative entropy between a distribution $\mu_{t}$ on the states at $\widetilde{Y}_t$  and the stationary distribution $\mu$ on the states at $\widetilde{Y}_m$  decreases with $t$, \ie ${\rm D}(\mu_{t}|| \mu) \ge {\rm D}(\mu_{t+1}|| \mu)$, because $\mu$ is the uniform distribution \cite{thomas2006elements}.
Since the relative entropy is the difference between the cross entropy $H(\mu_{t},\mu)$ and the entropy $H(\mu_{t})$, the inequality can be further expressed as
\begin{align}
    {\rm D}(\mu_{t}|| \mu) &=H(\mu_{t},\mu)-H(\mu_{t})\\
\label{eq:mu}
                           &=\mathbb{E}_{\mu_{t}}\log|{\cal Y}|-H(\mu_{t})\\
\label{eq:yt}
                           &= \log|{\cal Y}|-H(\widetilde{Y}_{t})\\
                           &\ge {\rm D}(\mu_{t+1}|| \mu)\\
\label{eq:yt1}
                           &= \log|{\cal Y}|-H(\widetilde{Y}_{t+1}),
\end{align}
where $|{\cal Y}|$ denotes the number of all possible states of $\widetilde{Y}$.
Hence, it has $H(\widetilde{Y}_{t}) \le H(\widetilde{Y}_{t+1})$ based on Eq. (\ref{eq:yt}) and Eq. (\ref{eq:yt1}), which means the entropy of the labels increases monotonically with $t$.
Therefore, the entropy of  clean labels $H(Y)=H(\widetilde{Y}_0)$ is smaller than the entropy of the noisy labels $H(\widetilde{Y})$ with arbitrary noise rate, \ie $H(Y) \le H(\widetilde{Y})$.
\end{proof}

Next, we prove that DNNs trained with symmetric label noise have a looser generalization bound than those trained without label noise in Theorem \ref{thm:1}.

\begin{theorem}
\label{thm:1}
DNNs trained with symmetric label noise have a loose constraint of the generalization bound than those trained without label noise.
\end{theorem}
\begin{proof}

Considering the Markov chain $X \to Y \to \widetilde{Y}$, it has $ I(X;\widetilde{Y}) \le I(X,Y)$ according to the DPI.
Based on $I(X;Y)=H(Y)-H(Y|X)$, we can obtain
\begin{equation}
\label{eq:entropy}
    H(\widetilde{Y})-H(\widetilde{Y}|X)\leq H(Y)-H(Y|X).
\end{equation}

According to Lemma \ref{lm:sym_increase}, it has $H(\widetilde{Y})\geq H(Y)$, which is applied to  Eq. (\ref{eq:entropy})  to have
\begin{align}
     &0\leq H(\widetilde{Y})-H(Y)\leq H(\widetilde{Y}|X)-H(Y|X),\\
     \label{eq:naturaline}
     &H(Y|X)\leq H(\widetilde{Y}|X).
\end{align}

Letting $S=(X,Y)$ be a clean dataset without label noise, based on Eq. (\ref{eq:bounddecom}), the generalization bound ${\cal B}_S^*$ of DNNs trained on $S$ is upper bounded by
\begin{align}
\label{eq:6}
    {\cal B}_S^*&=\sqrt {\frac{I(W;Y|X)+I(W;X)+\ln(\frac{2\sqrt{m}}{\delta})}{2m}}\\
\label{eq:7}
              &\le \sqrt {\frac{H(Y|X)+I(W;X)+\ln(\frac{2\sqrt{m}}{\delta})}{2m}}.
\end{align}
Eq. (\ref{eq:7}) holds  because of the equality $I(W;Y|X)=H(Y|X)-H(Y|X,W)$ and the non-negativity of the entropy. 
Similarly, for the dataset  ${\widetilde S}=(X,\widetilde{Y})$  with symmetric label noise, it has
\begin{align}
\label{eq:9}
    {\cal B}_{\widetilde S}^*&=\sqrt {\frac{I(W;\widetilde{Y}|X)+I(W;X)+\ln(\frac{2\sqrt{m}}{\delta})}{2m}}\\
    &\le \sqrt {\frac{H(\widetilde{Y}|X)+I(W;X)+\ln(\frac{2\sqrt{m}}{\delta})}{2m}}.
\end{align}

Denoting ${\cal C}_S = \sqrt {\frac{1}{2m}(H(Y|X)+I(W;X)+\log(\frac{2\sqrt{m}}{\delta}))}$ and ${\cal C}_{\widetilde S} = \sqrt {\frac{1}{2m}(H(\widetilde{Y}|X)+I(W;X)+\log(\frac{2\sqrt{m}}{\delta}))}$, the inequality $ {\cal C}_S \le {\cal C}_{\widetilde S}$ is easily obtained from Eq. (\ref{eq:naturaline}). 
\end{proof}

According to Lemma \ref{lm:cedecommy},  because $I(W;Y|X)$ is a negative term, the ERM algorithm minimizing  Eq. (\ref{eq:cedecommy}) will eventually maximize  $I(W;Y|X)$  until it reaches its upper bound $H(Y|X)$.
Similarly, $I(W;\widetilde{Y}|X)$ will reach $H(\widetilde{Y}|X)$.
Therefore, without regularization, the generalization gap of   DNNs trained with label noise will eventually reach the loosened generalization bound ${\cal C}_{\widetilde S}$, which will cause worse generalization than  training without label noise.

Theorem \ref{thm:1} suggests DNNs trained with label noise will have bad generalization  because of the loosened constraint ${\cal C}_{\widetilde S}$ on the generalization bound. 
In the following, we will prove that adding Gaussian noise to the training features induces a constant upper bound of $I(W;X)$, which results in better generalization of DNNs under label noise than training without feature noise.

\begin{figure}[!tb]
    \centering
    \includegraphics{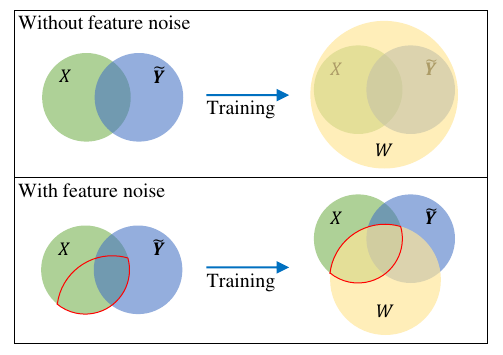}
    \caption{Information diagrams of the relationship between the training dataset (including the features $X$ and the noisy labels $\widetilde{Y}$) and the model weights $W$ of DNNs trained with or without feature noise.
    The circles represent the entropy, and the area of $h(W)$ overlapping with $h(X)$ and $H(\widetilde{Y})$ is $I(W;S)$, which
    is positively correlated to the value of the generalization bound. % in Eq. (\ref{eq:iws})
    The red line indicates the constraint of $I(W;X)$ induced by the feature noise. 
    Therefore, given adequate training epochs, the DNN trained with feature noise tends to have a smaller generalization bound than that trained without feature noise because $I(W;X)$ has a constant upper bound under feature noise.}
    \label{fig:info}
\end{figure}

\subsection{Feature noise}\label{sec:fn}
In this part, we conduct theoretical analysis to demonstrate the Gaussian feature noise constrains the PAC-Bayes generalization bound, and will boost DNNs' generalization in Theorem \ref{thm:fn_mem}.
Fig. \ref{fig:info} illustrates the influence of training DNNs with feature noise on their generalization.
%analyze adding fn limit the pbgb

\begin{theorem}%[Feature noise constrains the generalization bound]
\label{thm:fn_mem}
Training DNNs with Gaussian feature noise constrains  the generalization bound by constraining the mutual information $I(W;X)$ between the DNN weights  $W$ and the  features $X$.
\end{theorem}
\begin{proof}
Letting ${\widetilde S}=(X,\widetilde{Y})$ be a  dataset with symmetric label noise, we
add Gaussian noise into the original features $X$ to obtain the new features $\widetilde{X}$.
% We consider the process of adding feature noise as a Gaussian channel, which transmits signals from the original feature $\mathbf{X}$ to the new features $\widetilde{\mathbf{X}}$.
We have a Markov chain $X \to \widetilde{X} \to W$ because the model weights $W$ depend only on $\widetilde{X}$.
The process of adding Gaussian feature noise  can be regarded as a Gaussian channel that transmits signals from the original features $X$ to the new features $\widetilde{X}$.
Thus, according to the DPI and the Gaussian channel capacity with mean constraint \cite{thomas2006elements}, it has
\begin{equation}
\label{eq:dpiandcc}
I(W;X)\leq I(X;\widetilde{X}) \leq
\frac{1}{2}\ln (1 + \frac{{{\mathbb E}(X)}}{\sigma^2}),
\end{equation}
where ${\mathbb E}(X)$ denotes the expectation of $X$.

Applying  Eq. (\ref{eq:dpiandcc}) to Eq. (\ref{eq:bounddecom}), 
it can obtain:
\begin{align}
\label{eq:ii}
{\cal B}^*&=\sqrt {\frac{I(W;\widetilde{Y}|X)+I(W;X)+\ln(\frac{2\sqrt{m}}{\delta})}{2m}}\\
\label{eq:noiseboundbefore}
&\le \sqrt {\frac{H(\widetilde{Y}|X)+I(X;\widetilde{X})+\ln(\frac{2\sqrt{m}}{\delta})}{2m}}\\
\label{eq:noisebound}
&\le \sqrt {\frac{H(\widetilde{Y}|X)+\frac{1}{2}\ln (1 + \frac{{{\mathbb E}(X)}}{\sigma^2})+\ln(\frac{2\sqrt{m}}{\delta})}{2m}}.
\end{align}
\end{proof}

Eq. (\ref{eq:noisebound}) is the constraint of the generalization bound  of the DNNs trained on ${\widetilde S}$ with Gaussian feature noise. 
This constraint is a constant determined by $H(\widetilde{Y}|X)$, \ie the information of $\widetilde{Y}$ that cannot be inferred from $X$, ${\mathbb E}(X)$, \ie the expectation of $X$, the variance of the Gaussian noise $\sigma^2$, the number of the training samples  $m$ and the arbitrary value $\delta \in (0,1)$.
This constraint is unrelated to the model weights $W$. Next, we will explain why this constant upper bound  can  boost the generalization of DNNs under label noise.

The value of $I(X;\widetilde{X})$ can be calculated. According to Eq. (\ref{eq:dpiandcc}), assuming we have a dataset with ${\mathbb E}(X)=0.5$ and $\sigma^2=0.5$ (the variance of the Gaussian noise), the upper bound of $I(X;\widetilde{X})$ is approximately 0.35 nats. \footnote{The values of this upper bound in other cases will not be far from this value because we have ${\mathbb E}(X)\approx \sigma^2$.}
According to Lemma \ref{lm:cedecommy}, $I(W;X)=h(X)-h(X|W)$  is being maximized during the training of the cross-entropy loss. Theoretically, $I(W;X)$  will grow to positive infinity as $W$ contains all the information in $X$ (\ie $h(X|W)\to -\infty$) if there is no regularization. However, DNNs often have implicit regularizations, \eg SGD algorithm adds stochasticity to the training process, and $I(W;X)$ will not be infinite. Nevertheless, by adding Gaussian noise to the features, we can restrict the value of $I(W;X)$ to a considerably small value, \ie approximately 0.35 nats, which significantly tightens the original upper bound of $I(W;X)$ and will boost the generalization of DNNs under label noise. 

\begin{remark}
The upper bound of $I(W;Y|X)$ in Eq. (\ref{eq:noisebound}) is to be further discussed for two reasons. 
First, it does not take into account the influence of the feature noise. 
To analyze the quantitative relationship between the generalization bound and the feature noise, we need to replace $H(\widetilde{Y}|X)$ in Eq. (\ref{eq:noisebound}) with an upper bound that is related to $\sigma^2$ (proved in Proposition \ref{prop:quantify}).
% The constraint in  Eq. (\ref{eq:noisebound}) is negatively correlated with the Gaussian noise variance $\sigma^2$, which implies that with $\sigma^2$ growing, the generalization of DNNs becomes better.
% But intuitively, DNNs will not generalize well when $\sigma^2$ is extremely large and overwhelms the original features. This conflict is due to that the upper bound $H(\widetilde{Y}|X)$ is not correlated with $\sigma^2$.
Second, similar to the feature noise, the injection of label noise also poses a restriction on $I(W;Y|X)$ because of the DPI of the Markov chain $Y\to \widetilde{Y} \to W$. There is $I(W;Y|X) \le I(W;Y) \le I(Y;\widetilde{Y})$. However, it is unclear if $I(Y;\widetilde{Y})$ tightens the original upper bound $H(Y|X)$, because they are both likely to be very small.
\end{remark}

Theorem \ref{thm:fn_mem} proves that DNNs trained on $\widetilde{S}$ with
Gaussian feature noise has a tighter generalization bound than DNNs trained without feature noise,
because the adding of Gausssian noise constrains $I(W;X)$, \ie the feature information stored in the model weights.
However, we  still need to investigate what types of feature noise and how much feature noise  should be added to ensure good label noise generalization.

\section{Application Analysis}\label{sec:app}%Qualitative and quantitative
In this section, we conduct application analysis to define what types of feature noise, and  how much feature noise can achieve good label noise generaliztion of DNNs in Proposition \ref{prop:noisetype} and \ref{prop:quantify}, respectively.

\begin{proposition}%[The randomness of feature noise]
\label{prop:noisetype}
The constraint of the generalization bound is tighter for DNNs trained with feature noise with higher randomness.
\end{proposition}
\begin{proof}
Given an arbitrary noise type $K$ added to the original features $X$, we define the randomness ${\cal R}_K$ of $K$ as the conditional entropy of the original features $X$ with the new features $\widetilde{X}_K$ known:
\begin{equation}
\label{eq:random}
    {\cal R}_K=h(X|\widetilde{X}_K).
\end{equation}
${\cal R}_K$ indicates the  information required to recover $X$ from $\widetilde{X}_K$.

Recall that Eq. (\ref{eq:dpiandcc}) is obtained by regarding the adding of Gaussian feature noise as a Gaussian channel.
More generally, according to \cite{thomas2006elements}, the adding of any feature noise can be regarded as a noisy channel with capacity (\ie the maximum mutual information between $X$ and $\widetilde{X}_K$ among all possible input distributions $p(x)$)  of 
\begin{equation}
\label{eq:gencapa}
    C ={\max_{p(x)}} 
\ I(X;\widetilde{X}_K).
\end{equation}

Combining Eq. (\ref{eq:random}) with Eq. (\ref{eq:gencapa}), it can obtain:
\begin{align}
    C&={\max_{p(x)}} (h(X)-h(X|\widetilde{X}_K))\\
\label{eq:capacity}
     &={\max_{p(x)}} (h(X)-{\cal R}_K).
\end{align}
Applying Eq. (\ref{eq:capacity}) to Eq. (\ref{eq:noiseboundbefore}), it has
\begin{equation}
\label{eq:chan_capa}
{\cal B}^* 
\le \sqrt {\frac{H(\widetilde{Y}|X)+{\max_{p(x)}} (h(X)-{\cal R}_K)+\ln(\frac{2\sqrt{m}}{\delta})}{2m}}.
\end{equation}

Since $h(X)$ is constant for any noise type, the constraint in Eq. (\ref{eq:chan_capa}) depends on the value of ${\cal R}_K$.
% ${\cal R}_K=H(\mathbf{X}|\mathbf{X}+K))$ is an intrinsic quality of a noise type.
Obviously, the constraint of the generalization bound becomes tight with  the increased  ${\cal R}_K$.
Therefore, the constraint of the generalization bound
is tighter for DNNs trained with feature noise with higher
${\cal R}_K$.
\end{proof}

\begin{remark}\label{rem:random}
According to the definition of feature noise randomness in Eq. (\ref{eq:random}), 
% the randomness of a feature noise is defined as the extra information it needs to recover the original features from the noised features.
the feature noise that is difficult to be recovered will have large randomness.
For example, Gaussian blur can be easily recovered, while Gaussian noise is difficult to be recovered because of the information loss \cite{geman1992constrained}, which makes the randomness of Gaussian blur smaller than that of Gaussian noise. 
Therefore, Proposition \ref{prop:noisetype} suggests DNNs trained with Gaussian noise will have better label noise generalization than DNNs trained with Gaussian blur (verified by Fig. \ref{fig:fns}).
\end{remark}

Lemma \ref{lm:prop2} demonstrates the conditional entropy $H(\widetilde{Y}|X)$ becomes larger after adding Gaussian feature noise, \ie $H(\widetilde{Y}|X)\le H(\widetilde{Y}|\widetilde{X})$,
and $H(\widetilde{Y}|\widetilde{X})$ has an upper bound of $H(\widetilde{Y})$ when $\sigma^2$ approaches to infinity.

\begin{lemma}%[proved in Appendix B]%\ref{ap:ta}
\label{lm:prop2}
After adding Gaussian feature noise with variance $\sigma^2$ to the dataset $\widetilde{S}=(X,\widetilde{Y})$, the following inequalities hold:
$H(\widetilde{Y}|X)\le H(\widetilde{Y}|\widetilde{X}) \le \lim\limits_{\sigma^2 \to \infty}  H(\widetilde{Y}|\widetilde{X}) = H(\widetilde{Y})$.%, where $H(\widetilde{\mathbf{y}}|\widetilde{\mathbf{X}}) \propto \sigma^2$.
\end{lemma}
\begin{proof}
First, we have a Markov chain $\widetilde{Y} \to X \to \widetilde{X}$ because $\widetilde{X}$ depends only on $X$. According to the DPI, it has
\begin{equation}
\label{eq:dpi2}
    I(\widetilde{X};\widetilde{Y}) \le I(X;\widetilde{Y}).
\end{equation}
By adding $H(\widetilde{Y})$ to both sides of Eq. (\ref{eq:dpi2}), it becomes
\begin{align}
    H(\widetilde{Y}) + I(\widetilde{X};\widetilde{Y}) &\le H(\widetilde{Y}) + I({X};\widetilde{Y}),\\
    H(\widetilde{Y}) -I(X;\widetilde{Y}) &\le H(\widetilde{Y}) - I(\widetilde{X};\widetilde{Y}),\\
    H(\widetilde{Y}|X) &\le H(\widetilde{Y}|\widetilde{X}).
\end{align}

Under the extreme scenario of $\sigma^2$ approaches to infinity, the values of the original features $X$ can be ignored, and thus $\widetilde{X}$ can be considered as random Gaussian noise $Z$, that is
\begin{equation}
    H(\widetilde{Y}|\widetilde{X}) \le H(\widetilde{Y})=H(\widetilde{Y}|Z)=\lim\limits_{\sigma^2 \to \infty}  H(\widetilde{Y}|\widetilde{X}).
\end{equation}
\end{proof}

To conduct quantitative analysis, we need to define the algebraic relationship between $H(\widetilde{Y}|\widetilde{X})$ and $\sigma^2$.
Based on Lemma \ref{lm:prop2}, we know that $H(\widetilde{Y}|\widetilde{X})$ has its upper bound when $\sigma^2$ approaches infinity.
Assuming  $H(\widetilde{Y}|\widetilde{X})$ increases linearly with $\sigma^2$, \ie  $H(\widetilde{Y}|\widetilde{X})=a\sigma^2$, we  conduct quantitative analysis of how much feature noise, \ie the variance of Gaussian noise, should be added to the original features to achieve good generalization of DNNs In Proposition \ref{prop:quantify}.

\begin{proposition}%[The quantity of feature noise]
\label{prop:quantify}
DNNs trained with Gaussian feature noise with variance $\sigma^2$ 
have the tightest  constraint of the generalization bound at $\sigma^2 = \frac{1}{2}({\sqrt {{{\mathbb E}(X)^2} + \frac{{2{\mathbb E}(X)}}{{a}}}  - {\mathbb E}(X)})$.

\end{proposition}
\begin{proof}
Based on Lemma \ref{lm:prop2}, Eq. (\ref{eq:noisebound}) can be rewritten as:

\begin{equation}
\label{eq:quantity2}
{\cal B}^*\le \sqrt {\frac{H(\widetilde{Y}|\widetilde{X})+ \ln ({(1 + \frac{{{\mathbb E}(X)}}{\sigma^2})^{\frac{1}{2}}}{(\frac{2\sqrt{m}}{\delta })})}{m}}.
\end{equation}

Thus, the tightest constraint of ${\cal B}^*$ is obtained by the minimum of the right-hand side in Eq. (\ref{eq:quantity2}).
Based on $H(\widetilde{Y}|\widetilde{X})=a\sigma^2$, it becomes

\begin{equation}
\label{eq:mini}
    {\cal B}^*\le \sqrt {\frac{a\sigma^2+ \ln ({(1 + \frac{{{\mathbb E}(X)}}{\sigma^2})^{\frac{1}{2}}}{(\frac{2\sqrt{m}}{\delta })})}{2m}}.
\end{equation}

Let $f(\sigma^2)=a\sigma^2+ \ln ({(1 + \frac{{{\mathbb E}(X)}}{\sigma^2})^{\frac{1}{2}}}{(\frac{2\sqrt{m}}{\delta })})$, it has
\begin{equation}
    f'(\sigma^2)=a - \frac{{\mathbb E}(X)}{{{\sigma^4}(2 + \frac{2{\mathbb E}(X)}{\sigma^2})}}.
\end{equation}
    
The minimum of the right-hand side of Eq. (\ref{eq:mini}) is obtained by $f'(\sigma^2)=0$. That is
\begin{equation}
\label{eq:root}
    {\sigma^4} + {\mathbb E}(X)\sigma^2 - \frac{{\mathbb E}(X)}{2a}=0.
\end{equation}
 
For $\sigma^2 \in \left( {0,\infty } \right)$ and $a> 0$, Eq. (\ref{eq:mini}) has its minimum at the root of Eq. (\ref{eq:root}):
\begin{equation}
\label{eq:minin}
    \sigma^2 = \frac{1}{2}({\sqrt {{{\mathbb E}(X)^2} + \frac{{2{\mathbb E}(X)}}{{a}}}  - {\mathbb E}(X)}).
\end{equation}

\end{proof}

Eq. (\ref{eq:minin}) suggests that the level of Gaussian feature noise achieving the tightest generalization bound is determined by the mean of the features ${\mathbb E}(X)$ and the coefficient $a$. Note that $a$ is changing as the label noise rate $\Delta$ varies, because we know intuitively that the information of $\widetilde{Y}$ that cannot be inferred from $\widetilde{X}$ (\ie $H(\widetilde{Y}|\widetilde{X})$) becomes larger as  $\Delta$ increases.
Therefore, Proposition \ref{prop:quantify} suggests that the level of the feature noise achieving the tightest generalization bound is a certain value influenced by the label noise rate (verified by Fig. \ref{fig:ablvar}).

\section{Experiment}\label{sec:ex}

In this section, we present a comprehensive set of experiments to validate the effectiveness of the FN method and to provide empirical evidence that corroborates our theoretical findings.
In Section \ref{ss:ca} and \ref{ss:int}, we compare the performance of the FN method against ten other state-of-the-art methods in terms of accuracy on four real-world datasets. Additionally, we compare the interpretability of the FN methods and the denoising-based methods in Appendix \ref{ss:int}.
Furthermore, in Section \ref{ss:bv}, we perform experiments aimed at visualizing the PAC-Bayes bound. The goal of this analysis is to empirically verify our primary theoretical conclusion, which posits that the utilization of feature noise can lead to tighter PAC-Bayes generalization bound of DNNs under label noise.
Additionally, we verify the application analysis in Section \ref{ss:fn} and \ref{ss:qu} and demonstrate the effectiveness of FN method on two other label noise types in Section \ref{ss:ln}.

\begin{table}[!htb]
\caption{Statistics of the datasets used in the experiments.}
\label{tb:dataset}
\resizebox{\linewidth}{!}{%
\begin{tabular}{ccccc}\toprule
         & \#(training) & \#(validation) & \#(testing) & \#(class) \\\midrule
MNIST    & 50000          & 10000            & 10000         & 10          \\
FashionMNIST& 50000       & 10000            & 10000         & 10          \\
miniImageNet & 800       & 200            & 200         & 2          \\
ADNI       & 164            & 124              & 124          & 2          \\\bottomrule 
\end{tabular}%
}
\end{table}

\begin{table*}[!tb]
\centering
\caption{ACC and standard deviations of all methods on grayscale image datasets under all label noise rates.}
\label{tb:gray}
\begin{threeparttable}
\begin{tabular}{ccccccccc}
\toprule
Dataset            & \multicolumn{4}{c}{MNIST}                                                                                    & \multicolumn{4}{c}{FMNIST}                                                                                    \\\cmidrule(r){2-5}\cmidrule(r){6-9}
$\Delta$ & 0                         & 0.2                       & 0.5                       & 0.8                       & 0                         & 0.2                       & 0.5                       & 0.8                       \\ \midrule
BL                 & 98.9$\pm$0.1 & 97.1$\pm$1.8 & 92.2$\pm$3.9 & 42.5$\pm$5.7 & 99.0$\pm$0.0 & 83.1$\pm$6.9 & 71.3$\pm$3.5 & 34.8$\pm$0.6 \\
SR                 & 98.1$\pm$0.5 & 96.4$\pm$0.2 & 91.8$\pm$0.6 & 83.3$\pm$0.6 & 88.2$\pm$0.2 & 85.9$\pm$0.7 & 84.2$\pm$0.1 & 76.6$\pm$2.0 \\
CT                 & 99.5$\pm$0.1 & 97.2$\pm$0.1 & 92.4$\pm$0.1 & 82.2$\pm$1.0 & \textbf{99.6$\pm$0.0} & 97.2$\pm$0.1 & 92.7$\pm$0.6 & 82.3$\pm$0.6 \\
CT+                & 99.1$\pm$0.1 & 98.7$\pm$0.0 & 97.2$\pm$0.2 & 48.7$\pm$9.3 & 99.1$\pm$0.1 & 98.6$\pm$0.1 & 97.1$\pm$0.1 & 65.2$\pm$7.6 \\
ES                 & 99.0$\pm$0.0 & 98.0$\pm$0.5 & 95.8$\pm$3.1 & 88.5$\pm$4.8 & 99.0$\pm$0.2 & 98.0$\pm$0.2 & 94.9$\pm$0.2 & 55.3$\pm$3.2 \\
ELR                & 99.6$\pm$0.0 & 99.3$\pm$0.1 & 98.8$\pm$0.1 & 85.9$\pm$1.8 & 93.9$\pm$0.3 & 92.2$\pm$0.5 & 86.9$\pm$0.5 & 55.6$\pm$8.8 \\
ELR+               & 99.3$\pm$0.4 & \textbf{99.5$\pm$0.0} & 99.1$\pm$0.1 & 85.5$\pm$0.2 & \textbf{99.6$\pm$0.0} & \textbf{99.1$\pm$0.6} & 97.2$\pm$0.3 & 84.1$\pm$0.5 \\
PES                & 90.8$\pm$3.1 & 98.9$\pm$0.2 & 79.1$\pm$7.7 & 31.2$\pm$8.8 & 68.2$\pm$5.4 & 74.3$\pm$4.4 & 77.9$\pm$4.1 & 24.0$\pm$1.5 \\
NCT                & 99.5$\pm$0.2 & 96.9$\pm$0.2 & 75.4$\pm$0.9 & 25.2$\pm$1.0 & 91.9$\pm$2.9 & 88.8$\pm$0.2 & 67.7$\pm$1.0 & 25.2$\pm$0.5 \\
LIMIT              & 99.3$\pm$0.1 & 98.7$\pm$1.5 & 97.3$\pm$4.3 & 83.1$\pm$3.2 & 99.3$\pm$1.0 & 98.7$\pm$2.1 & 97.8$\pm$3.2 & 85.3$\pm$1.3 \\\midrule
BL+FN\tnote{*}     & 99.0$\pm$0.1 & 98.4$\pm$0.0 & 97.9$\pm$0.1 & 95.7$\pm$0.4 & 99.2$\pm$0.0 & 97.6$\pm$0.1 & 96.1$\pm$0.1 & 78.3$\pm$4.3 \\
ES+FN              & 99.2$\pm$0.0 & 98.5$\pm$0.2 & 97.8$\pm$0.2 & 95.8$\pm$0.2 & 99.1$\pm$0.0 & 98.5$\pm$0.2 & 97.7$\pm$0.0 & 91.2$\pm$0.7 \\
ELR+FN             & 99.6$\pm$0.0 & \textbf{99.5$\pm$0.0} & 98.8$\pm$0.1 & 94.0$\pm$0.2 & 98.5$\pm$0.0 & 98.3$\pm$0.2 & 95.4$\pm$0.4 & 80.7$\pm$0.2 \\
ELR++FN            & \textbf{99.7$\pm$0.3} & \textbf{99.5$\pm$0.3} & \textbf{99.2$\pm$0.2} & \textbf{99.2$\pm$0.3} & 98.6$\pm$1.2 & 98.5$\pm$1.1 & \textbf{97.9$\pm$2.0} & \textbf{92.4$\pm$2.9} \\ \bottomrule
\end{tabular}

\begin{tablenotes}
        \footnotesize
        \item[*] Feature noise method is denoted as FN, it can be combined with different methods to boost their performance.
\end{tablenotes}
\end{threeparttable}

\end{table*}

\begin{figure*}[!tb]
    \centering
	\subfigure[Gaussian noise]{
        \begin{minipage}[t]{.23\linewidth}
			\centering
			\includegraphics[width=1\linewidth]{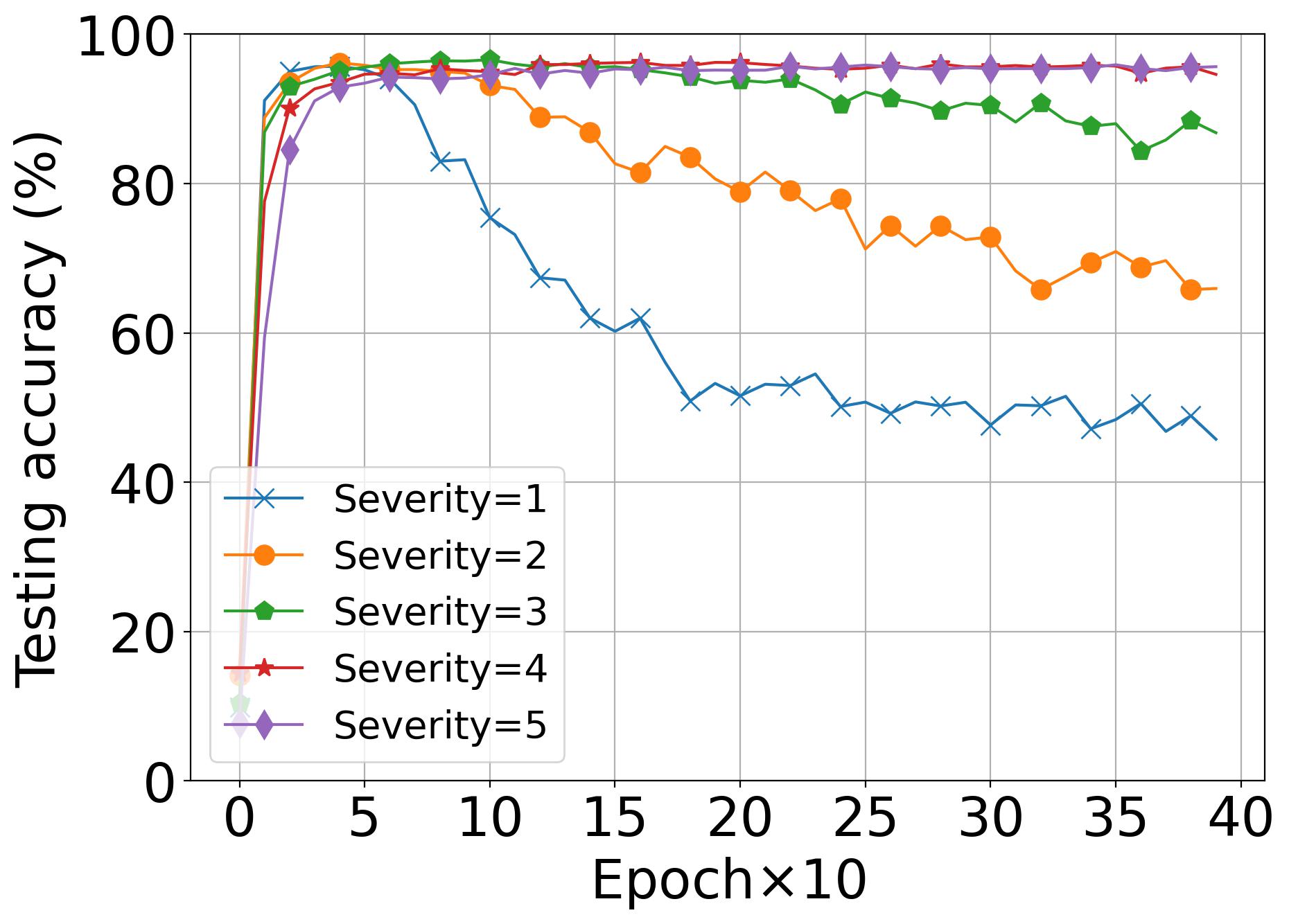}
		\end{minipage}
		}
	\subfigure[Impulse noise]{
		\begin{minipage}[t]{.23\linewidth}
			\centering
			\includegraphics[width=1\linewidth]{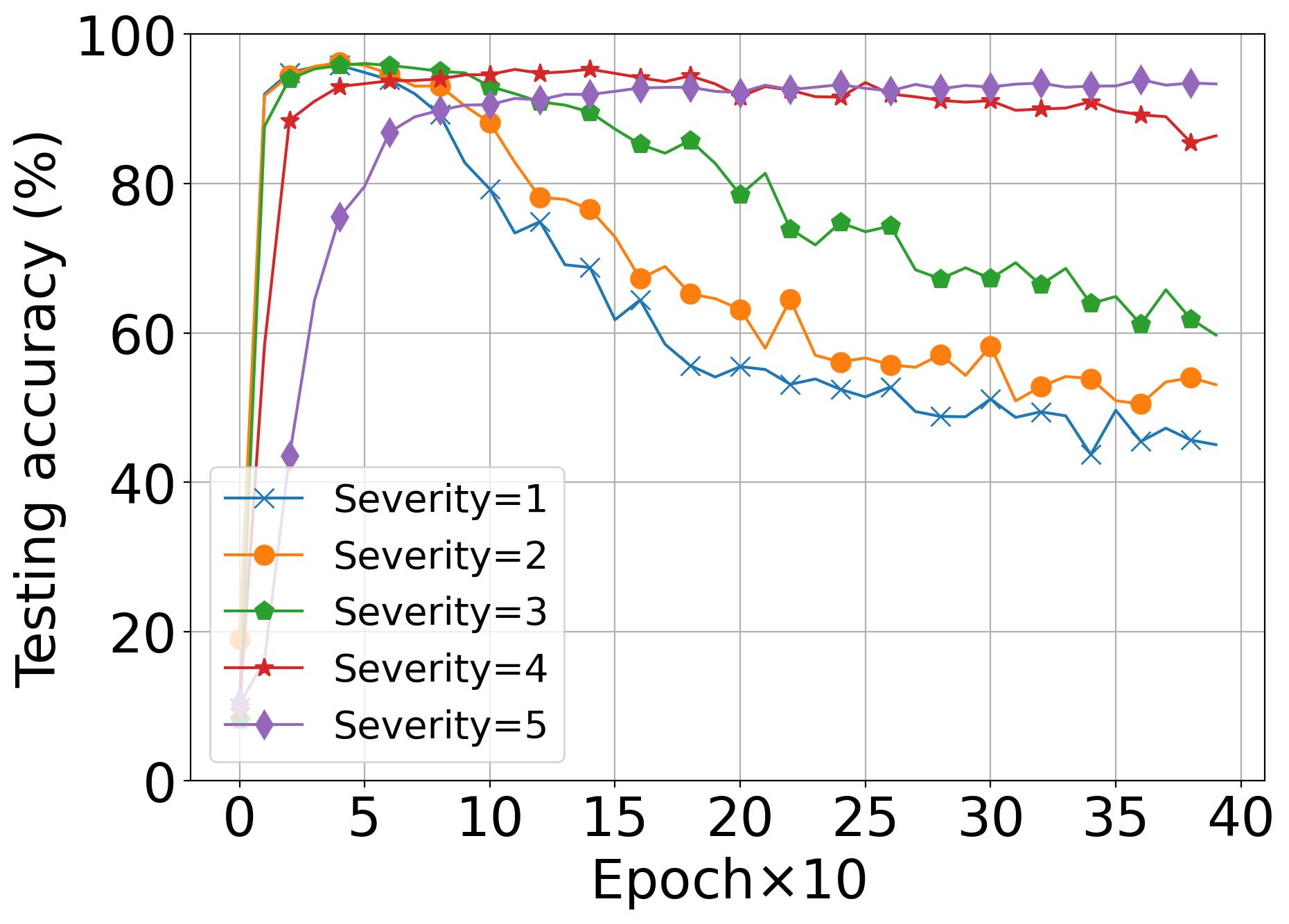}
		\end{minipage}
	}
	\subfigure[Gaussian blur]{
		\begin{minipage}[t]{.23\linewidth}
    		\centering
    		\includegraphics[width=1\linewidth]{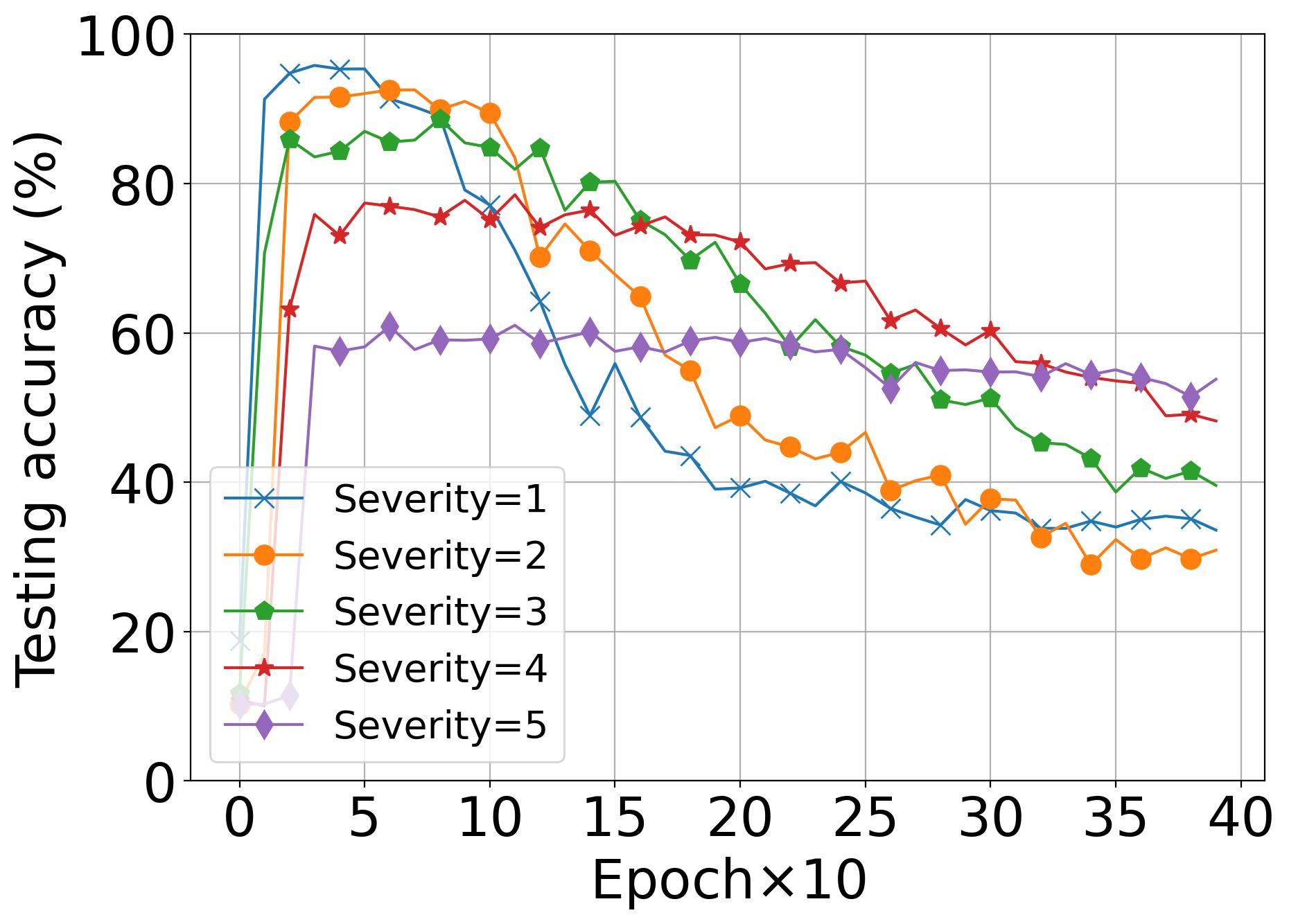}
		\end{minipage}
		}
	\subfigure[Snow]{
		\begin{minipage}[t]{.23\linewidth}
			\centering
			\includegraphics[width=1\linewidth]{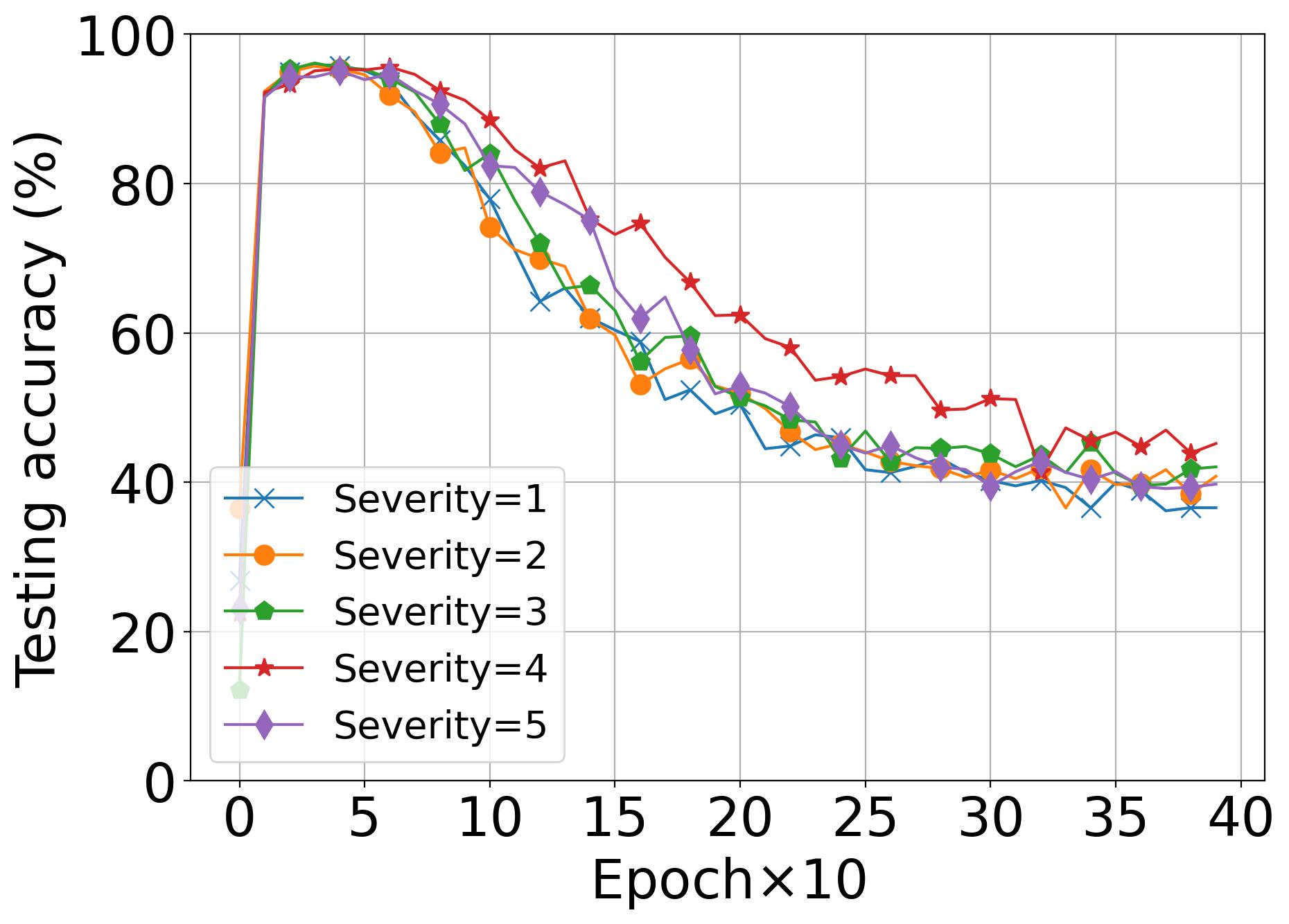}
		\end{minipage}
	}

    \caption{ACC curves of LeNet-5 trained on MNIST with four types of feature noise  of increasing severity levels  under $\Delta=0.8$. The severity levels of the feature noise are generated by the Imagecorruptions toolbox \cite{michaelis2019dragon}.}
    \label{fig:fns}
\end{figure*}

\subsection{Experimental Setup}
\subsubsection{Datasets}

To evaluate the effectiveness of the FN method on different types of datasets, we employ three image datasets (\ie MNIST \cite{deng2012mnist}, FashionMNIST \cite{xiao2017fashion} and miniImageNet \cite{vinyals2016matching}), and one graph dataset (\ie ADNI \cite{jack2008alzheimer}). 
Their statistics are presented in Table \ref{tb:dataset}. 
Detailed descriptions of all datasets are listed in Appendix \ref{ap:data}.

\subsubsection{Comparison methods}
The comparison methods include  baseline method (BL) (\ie convolutional neural network (CNN) \cite{lecun2015lenet,he2016deep} and graph convolutional network (GCN) \cite{welling2016semi}), one clean-sample-based method (\ie sample reweighting (SR) \cite{ren2018learning}), eight small-loss-sample-based methods (\ie co-teaching (CT) \cite{han2018co}, co-teaching+ (CT+) \cite{yu2019does}, early stopping (ES) \cite{li2020gradient}, early learning regularization (ELR) and ELR+ \cite{liu2020early}, progressive early stopping (PES) \cite{bai2021understanding} and nested co-teaching (NCT) \cite{chen2021boosting}, and one bound-reducing method (\ie LIMIT \cite{harutyunyan2020improving}).
To conduct a fair comparison, we delete all modules that increase the robustness of DNNs in general, \eg dropout and early stopping \footnote{We cancel the early stopping in all methods except ES.},  in all comparison methods and the FN method.
Detailed descriptions of all comparison methods are listed in Appendix \ref{ap:cm}.

\subsubsection{Setting-up}
All experiments are implemented in PyTorch and conducted on NVIDIA GeForce 3090 GPU (24GB memory). We repeat all experiments three times until the models converge and report their averaged results. We employ the author-verified codes for all comparison methods and achieve their best performance using the recommended parameters or a random search method for parameter selection. 
The source codes of the FN method are available on \url{https://github.com/zlzenglu/FN}.

\subsubsection{Criterion}
We add symmetric label noise with corruption rates $\Delta=\{0,0.2,0.5,0.8\}$ into the datasets and evaluate the effectiveness of all methods by the classification accuracy (ACC) on the clean testing set. It shows the model's generalization under different label noise rates.
% LNR is obtained by computing the reciprocal of the variance of the ACCs across all label noise rates. It indicates the model's performance stability, \ie the robustness, under different label noise rates.
Besides, we evaluate the interpretability of several methods trained under massive label noise by visualizing the saliency map ${\cal M}$, which is generated by computing the gradient of the class activation score w.r.t the input features following the literature \cite{simonyan2013deep}. It highlights the most important features for the model to make predictions. %Specifically, for an image $x_i$, we have ${\cal M}_i=\frac{\partial S_c}{\partial x_i}|_{x}$, where ${\cal M}_i$ has the same dimensionality as $x_i$. 

\subsection{Classification accuracy}\label{ss:ca}

We compare the performance of the FN method with all comparison methods in terms of ACC on gray-scale image datasets (\ie MNIST and FMNIST), high-resolution RGB image datasets (\ie MiniImageNet), and medical image graph datasets (\ie ADNI).

Table \ref{tb:gray} shows the ACC of all methods under different label noise rates on MNIST and FMNIST.  
% As we can see, the FN method has competitive ACC under low label noise rates ($\Delta = \{0, 0.2\}$), while it has significantly better ACC than all comparison methods under high label  noise rates ($\Delta = \{0.5, 0.8\}$) on both datasets with different backbone models. 
To further show the strengths, we evaluate the FN method on more complex datasets, \ie miniImageNet and ADNI. 
We select the baseline model and three methods (\ie ES, ELR, and SR) that outperform the others in the previous experiments  or are suitable to deal with graph datasets to be the comparison methods.
In the Appendix, Table \ref{tb:highrgb} displays the results of these comparison methods and the FN methods.

Results show that the FN method improves by  2.73\% (\{0.1\%, 0.0\%, 0.1\%, 10.7\%\} for $\Delta = \{0,0.2,0.5,0.8\}$), 1.55\% (\{-0.4\%, -0.6\%, 0.1\%, 7.1\%\} for $\Delta = \{0,0.2,0.5,0.8\}$),  4.75\% (\{7.5\%, 3.6\%, 5.7\%, 2.2\%\} for $\Delta = \{0,0.2,0.5,0.8\}$),  3.65\% (\{-3.3\%, 4.6\%, 5.3\%, 8.0\%\} for $\Delta = \{0,0.2,0.5,0.8\}$) on average,  compared to the best comparison methods on MNIST, FMNIST, miniImageNet and ADNI, respectively.
These results show that the FN method can achieve significantly better ACC under high label noise rates, \eg $\Delta=0.8$.
However, under low label noise rates, \eg $\Delta=\{0,0.2\}$, the FN method cannot always achieve significantly better performance.
This is possibly due to that the constraint of FN has on the generalization bound works better under high label noise rates, when the generalization bound is quite loose. For the low label noise rate, the generalization bound is tight, so that the effect of FN becomes trivial.

We must clarify that the value of the FN method is not competing with the SOTA label noise methods in terms of accuracy. 
The value of the FN method lies in its simplicity. The classification results prove that this simple trick can boost the generalization of many backbone models under label noise scenarios, with little extra computation costs and high flexibility on different datasets.
Based on these results, we want to emphasize the possibility of using FN as a standard technique on top of any model to boost the performance under severe label noise.

\subsection{Bound visualization}\label{ss:bv}

\begin{figure}[!ht]
    \centering
    \includegraphics[width=.75\linewidth]{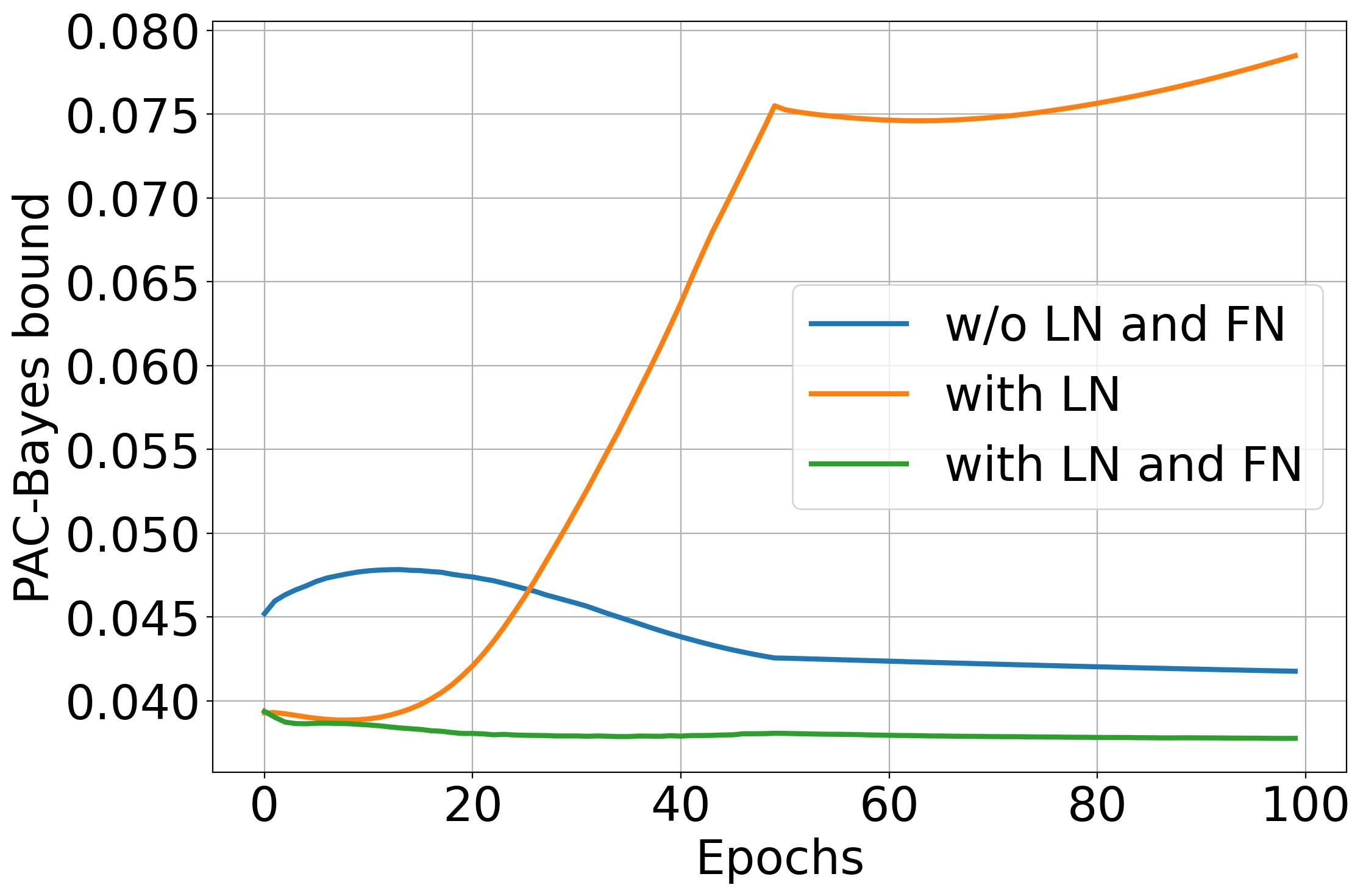}
    \caption{The PAC-Bayes generalization bound ($\delta=0.05$) of LeNet-5 trained on MNIST under three scenarios, \ie training without label noise and feature noise, training with label noise ($\Delta=0.8$), and training with label noise and feature noise (Gaussian noise with $\sigma^2=1$). The sharp change of the curves at epoch 50  is due to the learning rate change.}
    \label{fig:pbbvisual}
\end{figure}

To verify our conclusion in the theoretical analysis that adding feature noise tightens the generalization bound of DNNs under label noise, we visualize the PAC-Bayes bound of LeNet-5 trained on MNIST.
Specifically, we assume the prior and posterior distributions over the hypothesis space are two Gaussians. We use 30,000 clean samples and 30,000 samples with label noise or feature noise to train the DNN 100 times to generate the weight matrix $\mathbf{W}_t \in \mathbb{R}^{100 \times k}$ and the mean weight vector $\mathbf{w}_t \in \mathbb{R}^{1 \times k}$ over the 100-time runs, where $k$ is the dimension of the weights and $t$ is the number of the epoch. Then we have the multivariate Gaussian distribution parametrized by mean $\mathbf{w}_t$ and covariance matrix $\mathbb{E}\{(\mathbf{w}_t^{(i)}-\mathbf{w}_t)^\mathsf{T}(\mathbf{w}_t^{(i)}-\mathbf{w}_t)\}$. 
We calculate the PAC-Bayes bound in each epoch and show the result in Fig. \ref{fig:pbbvisual}.
Note that this approximation of the generalization bound can be quite loose than the true bound, because the Gaussian prior is not optimal. The theoretical upper bound with both label noise and feature noise computed by Eq. (\ref{eq:noisebound}) is approximately 0.014.

Fig. \ref{fig:pbbvisual} shows that training DNNs with both label noise and feature noise yields the tightest bound compared with training DNNs on the clean dataset and with label noise. 
This verifies our theoretical analysis that label noise loosens the generalization bound and feature noise tightens the generalization bound under label noise. 
Interestingly, although the bound increases as the training proceeds in the case of training with label noise, it can be very tight at the early training stage. This observation coincides with the early robustness characterization of DNNs reported in many studies \cite{rolnick2017deep,li2020gradient}.

\subsection{Feature noise type}\label{ss:fn}

To evaluate the influence of feature noise type on the FN method,  
we evaluate the ACC of DNNs trained with different  feature noise types under $\Delta=0.8$ in Fig. \ref{fig:fns}.

The feature noise types with high randomness (as defined in Remark \ref{rem:random}), \eg Gaussian noise and impulse noise, have better performance as the severity levels grow high. In contrast, the feature noise types with low randomness, \eg Gaussian blur and snow, have worse performance as the severity levels grow high.  
These results verify the qualitative analysis in Proposition \ref{prop:noisetype}, which suggests  
the higher randomness in feature noise for model training, the better model generalization.

\subsection{Label noise type}\label{ss:ln}

\begin{figure}[!hbt]
    \centering    \includegraphics[width=.65\linewidth]{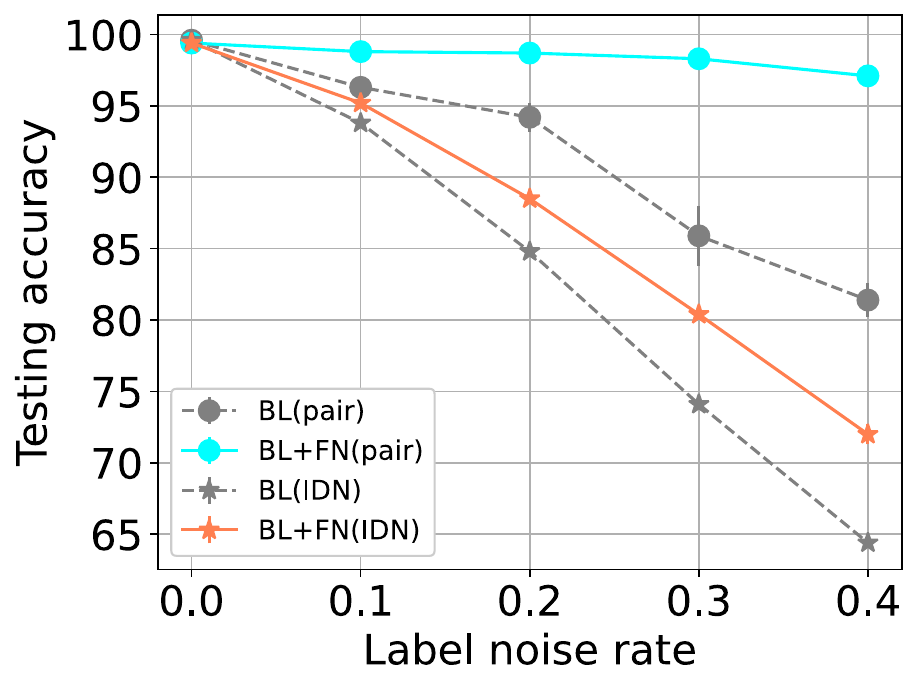}
    \caption{ACC curves of LeNet-5 (BL) trained on MNIST with multiple levels of  pair  noise and instance-dependent noise (IDN).}
    \label{fig:ln}
\end{figure}

In this part, we investigate the influence of different label noise types on the FN method.
To do that, we compare the performance of the FN method combining with the baseline model on MNIST with two label noise types other than symmetric noise, \ie  pair noise and instance-dependent noise, and five label noise rates $\Delta=\{0,0.1,0.2,0.3,0.4\}$.
We display the results in Fig. \ref{fig:ln}. It shows the FN method achieves significantly better generalization with both label noise types. This indicates the FN method is also applicable to pair noise and instance-dependent noise.

\subsection{Quantitative relationship between feature noise and label noise}\label{ss:qu}

\begin{figure}[!hbt]
    \centering
    \includegraphics[width=.95\linewidth]{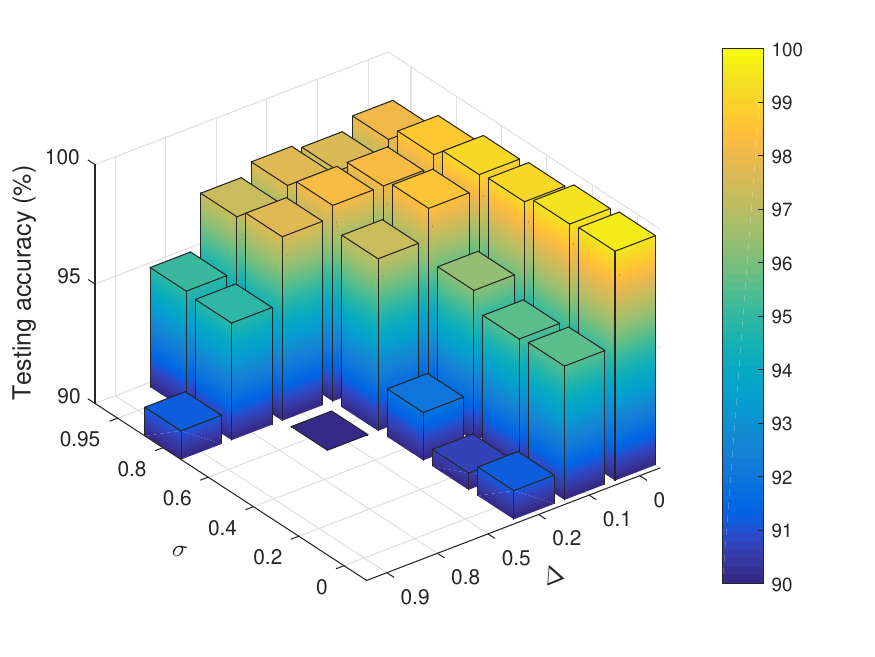}
    \caption{ACC bars of LeNet-5 trained on MNIST with multiple levels of Gaussian noise std ($\sigma$) and symmetric label noise rates ($\Delta$).}
    \label{fig:ablvar}
\end{figure}

We evaluate the relationship between the amount of feature noise and label noise in the training data that achieves the best performance of the model. Specifically, we evaluate the ACC of the baseline model trained under Gaussian feature noise with six levels of standard deviation ($\sigma\in[0,1]$)  and symmetric label noise with six levels of corruption rate ($\Delta\in[0,1)$). The results are shown in Fig. \ref{fig:ablvar}.

According to  Fig. \ref{fig:ablvar}, with the label noise rate ($\Delta$) increasing, Gaussian noise standard deviations also increase to obtain the best testing accuracy.
However, a higher $\sigma$ is not necessarily better under each $\Delta$, because the testing accuracy will reach a peak and then decrease as $\sigma$ increases.
These results are consistent with Proposition \ref{prop:quantify}, which suggests the best performance of the models trained with feature noise is achieved at a specific value of $\sigma^2$, and the value is influenced by $\Delta$.

\section{Conclusion}\label{sec:con}
In this work, we introduce a simple method of adding feature noise into the training data to boost the generalization of DNNs trained with label noise. 
Additionally, we provide theoretical analysis to investigate the impact of training DNNs with label noise and feature noise on generalization. Specifically, we demonstrate that training DNNs with label noise adversely affects generalization by loosening the generalization bound. In contrast, training DNNs with feature noise improves generalization by constraining the generalization bound. To complement the theoretical analysis, we conduct application analysis to identify the appropriate types and levels of feature noise that yield optimal DNN generalization.
To evaluate the effectiveness of the feature noise method, we perform extensive experiments on various real-world datasets. Our experimental results demonstrate that the feature noise method enhances the performance of state-of-the-art methods trained under label noise, leading to significant improvements in classification accuracy and interpretability under severe label noise. Furthermore, we systematically validate our theoretical findings through these experiments.

However, there are several aspects in which our work can be further explored:
(i) The development of a theoretical framework that extends our theorems to encompass other types of label noise, like pair noise and instance-dependent noise.
(ii) A theoretical framework that elucidates the relationship between the feature noise method and other related methods, like dropout and data augmentation.
(iii) Determining a threshold for the randomness of feature noise that guarantees favorable generalization in the presence of label noise, as outlined in Proposition \ref{prop:noisetype}.
To address these open questions and enhance the generalization of DNNs under label noise, we aim to thoroughly investigate these aspects in the future work.

\ifCLASSOPTIONcaptionsoff
  \newpage
\fi

\ifCLASSOPTIONcaptionsoff
  \newpage
\fi

\small{
\bibliographystyle{IEEEtranN}
\bibliography{main}
}
\normalsize{
}
\newpage
\onecolumn
\section{Appendix --- Feature Noise Boosts DNN Generalization under Label Noise}

\begin{table*}[!htb]
\centering
\caption{ACC and standard deviations of all methods on miniImageNet and ADNI under all label noise rates.}
\label{tb:highrgb}
\begin{tabular}{ccccccccc}
\toprule
Dataset            & \multicolumn{4}{c}{mini-ImageNet (binary)}                                                                                    & \multicolumn{4}{c}{ADNI}                                                                                    \\\cmidrule(r){2-5}\cmidrule(r){6-9}
$\Delta$ & 0                         & 0.2                       & 0.5                       & 0.8                       & 0                         & 0.2                       & 0.5                       & 0.8                       \\ \midrule
BL                 & 85.7$\pm$2.9 & 79.2$\pm$2.2 & 67.8$\pm$2.5 & 57.1$\pm$2.2 & \textbf{84.9$\pm$0.7} & 75.7$\pm$0.3 & 70.5$\pm$0.7 & 56.3$\pm$0.8 \\
ES                 & 84.7$\pm$2.1 & 79.5$\pm$2.8 & 60.0$\pm$2.9 & 56.0$\pm$2.9 & 83.9$\pm$0.5 & 77.6$\pm$1.6 & 75.7$\pm$1.5 & 71.5$\pm$5.1 \\
ELR                & 79.1$\pm$1.7 & 69.9$\pm$1.8 & 61.7$\pm$1.8 & 47.7$\pm$1.9 & 75.3$\pm$3.8 & 77.6$\pm$3.5 & 74.9$\pm$1.4 & 71.1$\pm$0.3 \\\midrule
BL+FN\tnote{*}     & 91.9$\pm$1.0 & 82.9$\pm$3.6 & \textbf{73.5$\pm$3.1} & \textbf{59.3$\pm$1.9} & 81.6$\pm$0.5 & \textbf{82.2$\pm$0.5} & 78.9$\pm$0.5 & 71.5$\pm$1.2 \\
ES+FN              & 91.7$\pm$0.2 & 79.4$\pm$1.7 & 68.7$\pm$3.3 & 57.7$\pm$2.6 & 79.9$\pm$1.2 & 79.1$\pm$1.0 & 76.2$\pm$1.4 & 72.2$\pm$3.4 \\
ELR+FN             & \textbf{93.2$\pm$1.6} & \textbf{83.1$\pm$2.2} & 61.4$\pm$1.8 & 53.3$\pm$3.6 & 79.9$\pm$2.9 & 80.3$\pm$1.0 & \textbf{81.0$\pm$0.5} & \textbf{79.5$\pm$1.0} \\\bottomrule
\end{tabular}
\end{table*}

\begin{figure*}[!tb]
	\centering

    %fist row
	\subfigure{
		\begin{minipage}[t]{0.081\linewidth}
			\centering
			\includegraphics[width=1\linewidth]{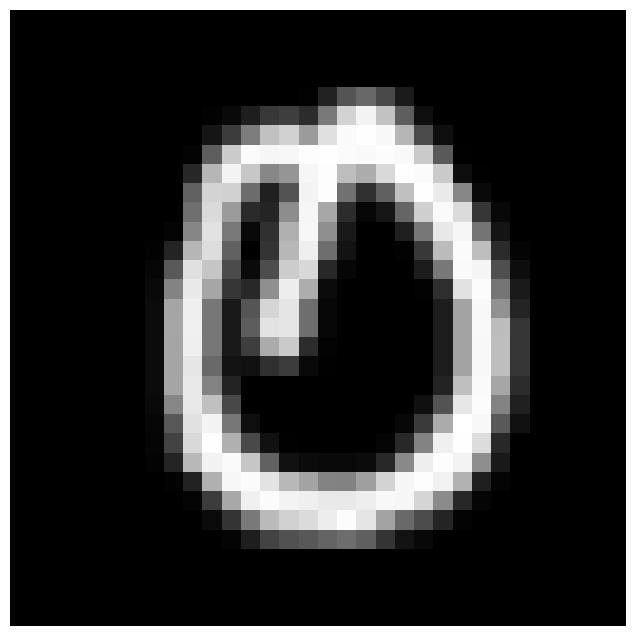}
		\end{minipage}
	}
	\subfigure{
		\begin{minipage}[t]{0.081\linewidth}
			\centering
			\includegraphics[width=1\linewidth]{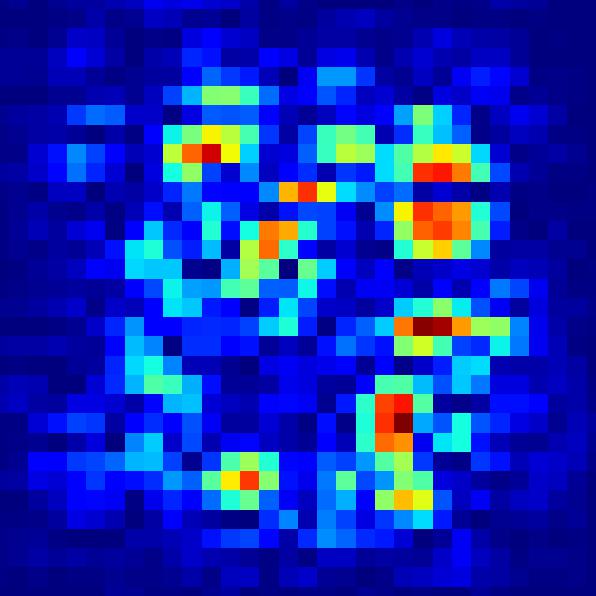}
		\end{minipage}
	}
	\subfigure{
		\begin{minipage}[t]{0.081\linewidth}
			\centering
			\includegraphics[width=1\linewidth]{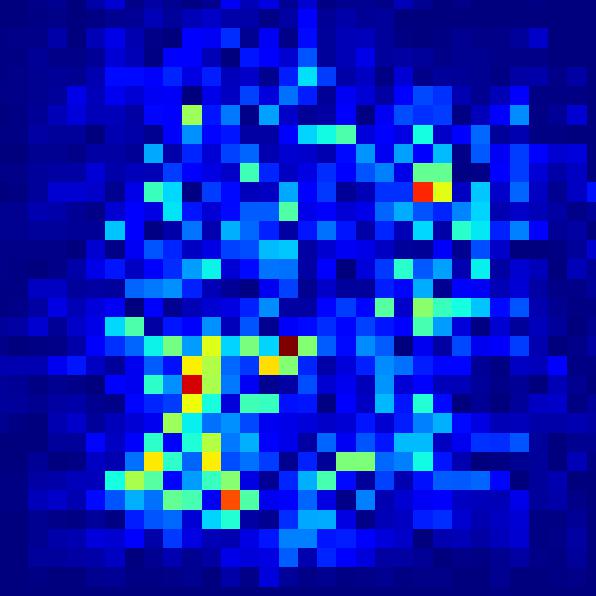}
		\end{minipage}
	}
	\subfigure{
		\begin{minipage}[t]{0.081\linewidth}
			\centering
			\includegraphics[width=1\linewidth]{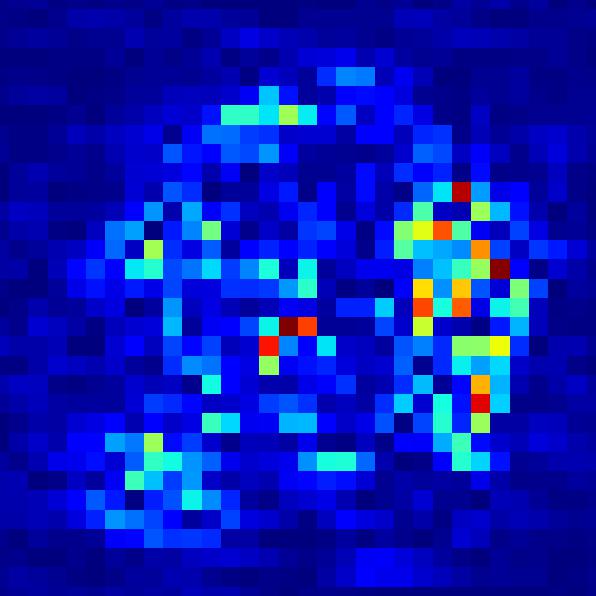}
		\end{minipage}
	}
	\subfigure{
		\begin{minipage}[t]{0.081\linewidth}
			\centering
			\includegraphics[width=1\linewidth]{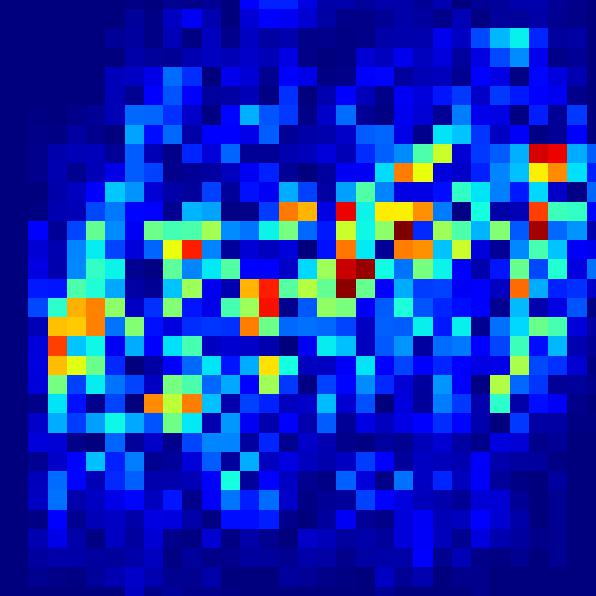}
		\end{minipage}
	}
	\subfigure{
		\begin{minipage}[t]{0.081\linewidth}
			\centering
			\includegraphics[width=1\linewidth]{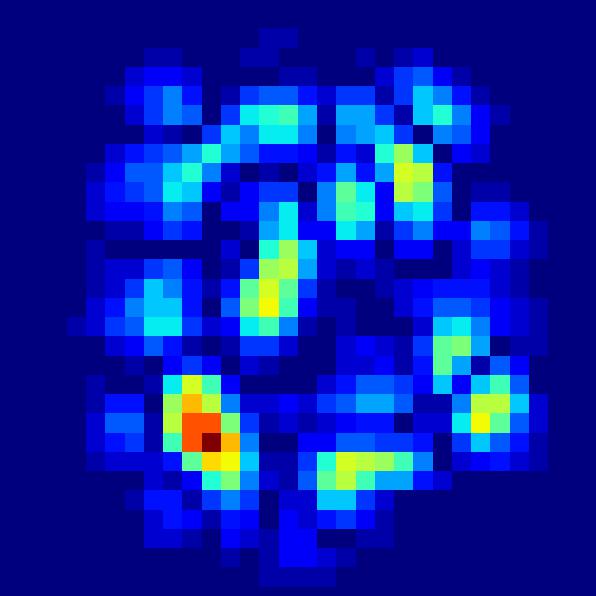}
		\end{minipage}
	}
	\subfigure{
		\begin{minipage}[t]{0.081\linewidth}
			\centering
			\includegraphics[width=1\linewidth]{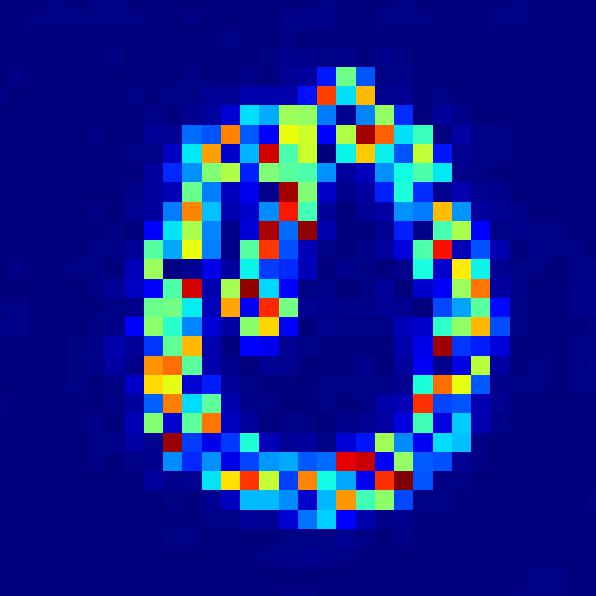}
		\end{minipage}
	}
	\subfigure{
		\begin{minipage}[t]{0.081\linewidth}
			\centering
			\includegraphics[width=1\linewidth]{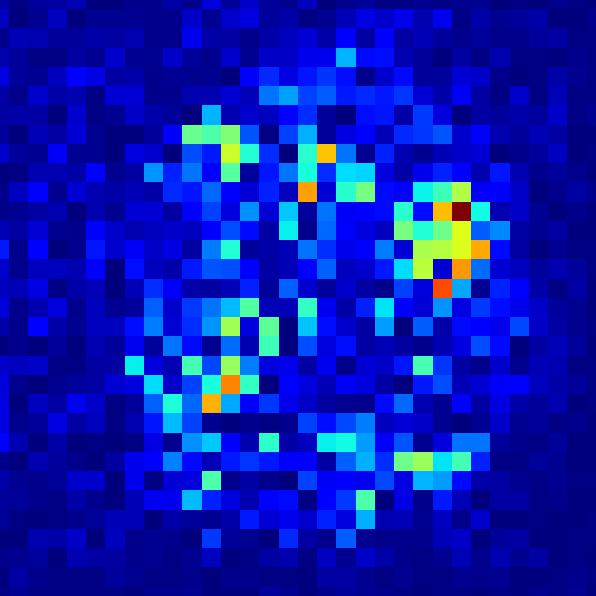}
		\end{minipage}
	}
	\subfigure{
		\begin{minipage}[t]{0.081\linewidth}
			\centering
			\includegraphics[width=1\linewidth]{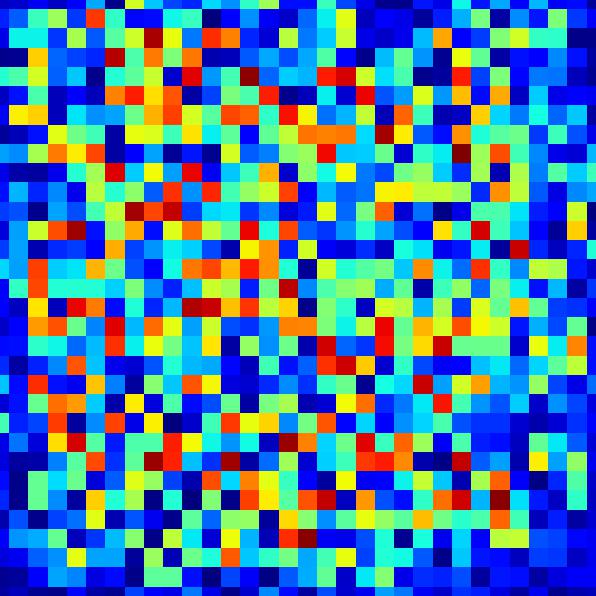}
		\end{minipage}
	}
	\subfigure{
		\begin{minipage}[t]{0.081\linewidth}
			\centering
			\includegraphics[width=1\linewidth]{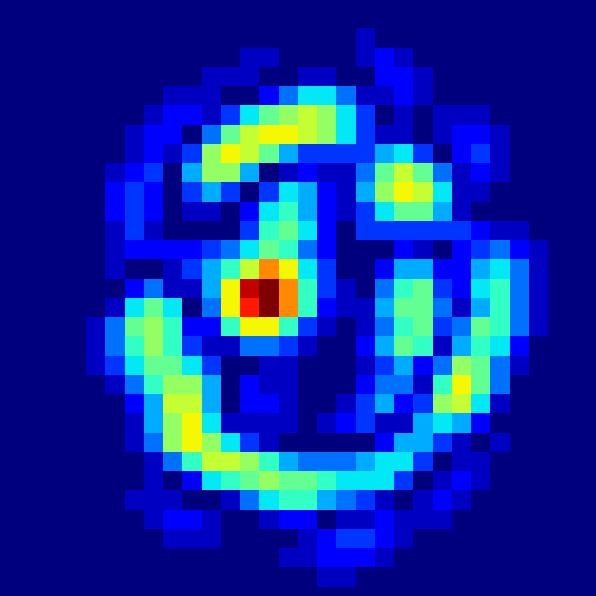}
		\end{minipage}
	}

	\setcounter{subfigure}{0}
	
	    %2 row
	\subfigure{
		\begin{minipage}[t]{0.081\linewidth}
			\centering
			\includegraphics[width=1\linewidth]{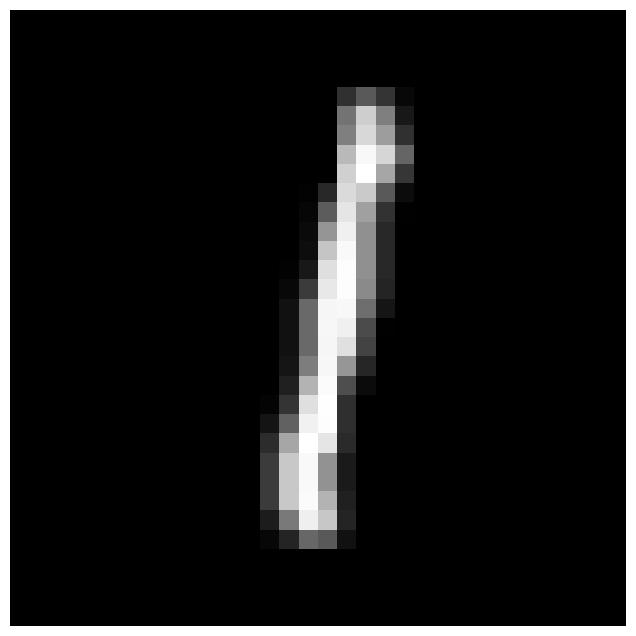}
		\end{minipage}
	}
	\subfigure{
		\begin{minipage}[t]{0.081\linewidth}
			\centering
			\includegraphics[width=1\linewidth]{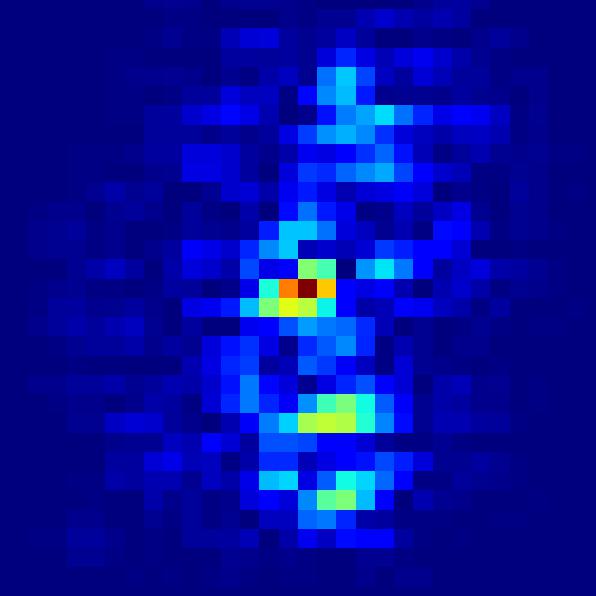}
		\end{minipage}
	}
	\subfigure{
		\begin{minipage}[t]{0.081\linewidth}
			\centering
			\includegraphics[width=1\linewidth]{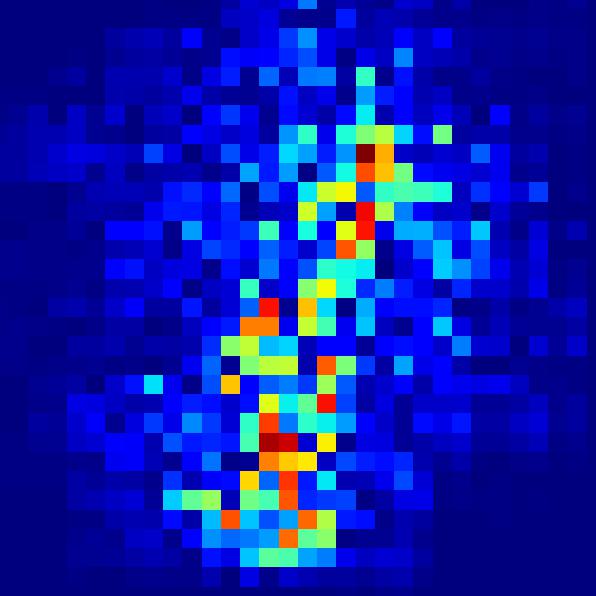}
		\end{minipage}
	}
	\subfigure{
		\begin{minipage}[t]{0.081\linewidth}
			\centering
			\includegraphics[width=1\linewidth]{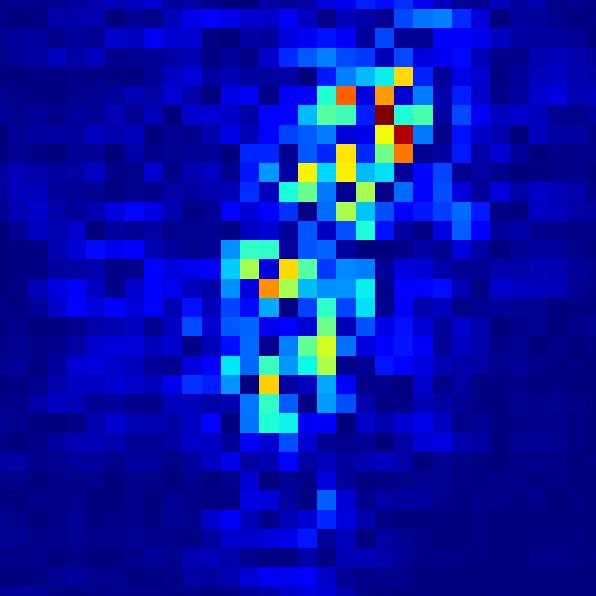}
		\end{minipage}
	}
	\subfigure{
		\begin{minipage}[t]{0.081\linewidth}
			\centering
			\includegraphics[width=1\linewidth]{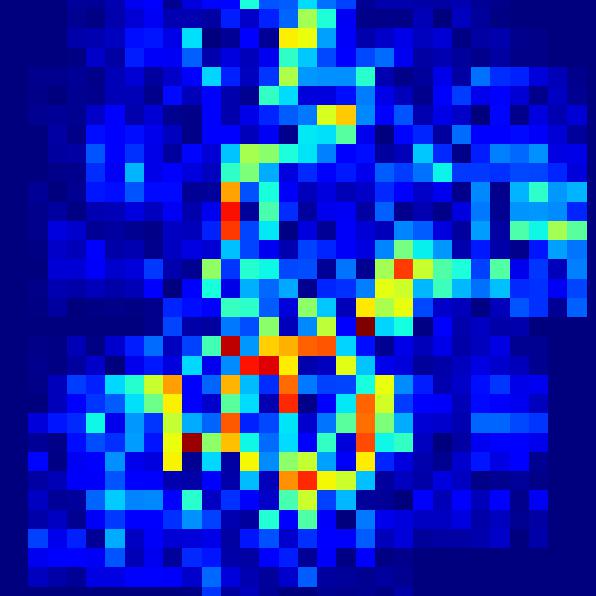}
		\end{minipage}
	}
	\subfigure{
		\begin{minipage}[t]{0.081\linewidth}
			\centering
			\includegraphics[width=1\linewidth]{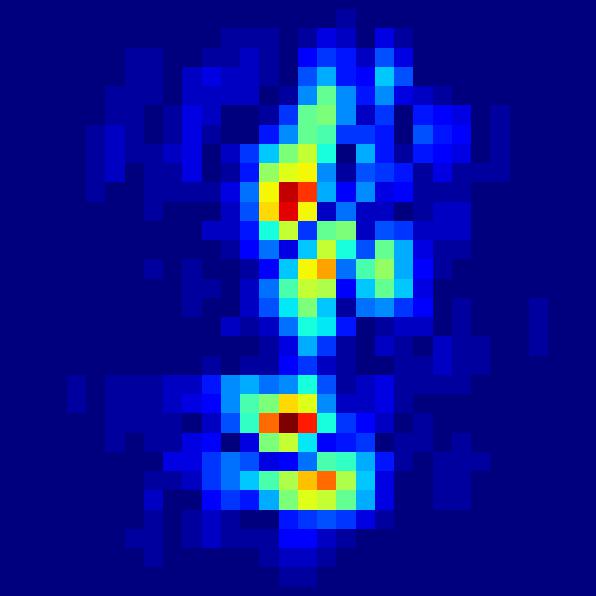}
		\end{minipage}
	}
	\subfigure{
		\begin{minipage}[t]{0.081\linewidth}
			\centering
			\includegraphics[width=1\linewidth]{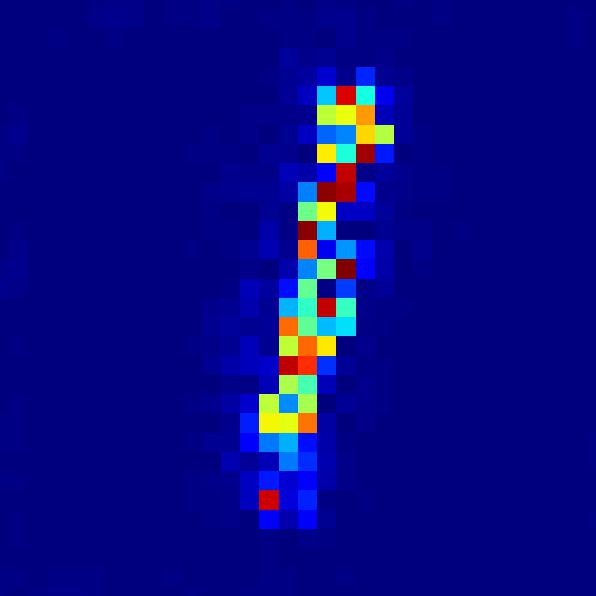}
		\end{minipage}
	}
	\subfigure{
		\begin{minipage}[t]{0.081\linewidth}
			\centering
			\includegraphics[width=1\linewidth]{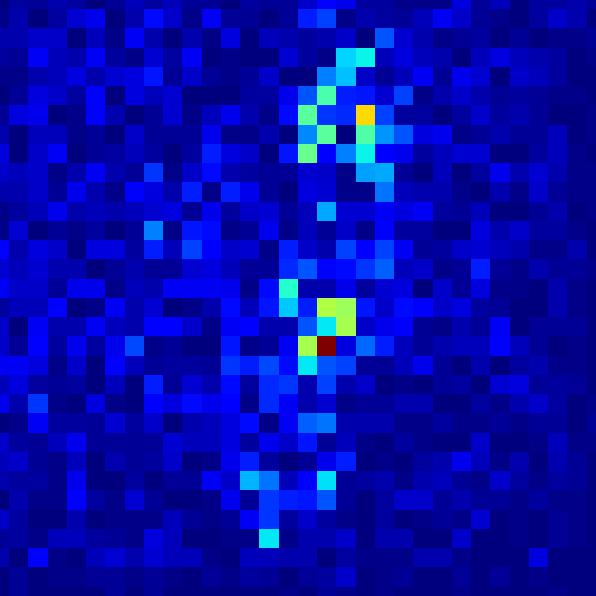}
		\end{minipage}
	}
	\subfigure{
		\begin{minipage}[t]{0.081\linewidth}
			\centering
			\includegraphics[width=1\linewidth]{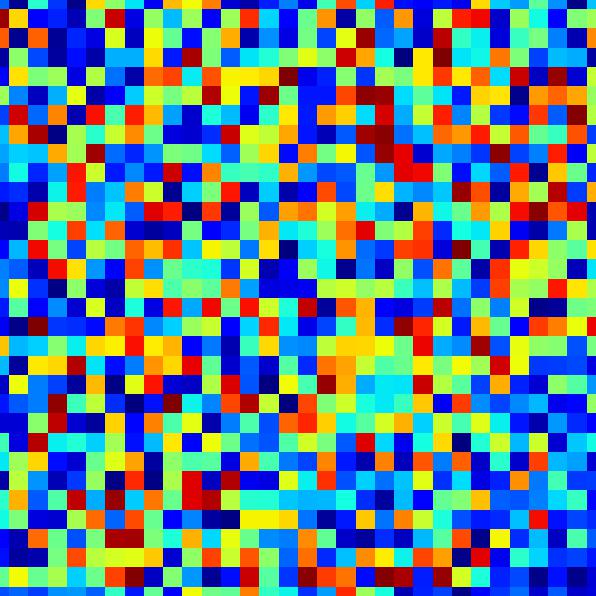}
		\end{minipage}
	}
	\subfigure{
		\begin{minipage}[t]{0.081\linewidth}
			\centering
			\includegraphics[width=1\linewidth]{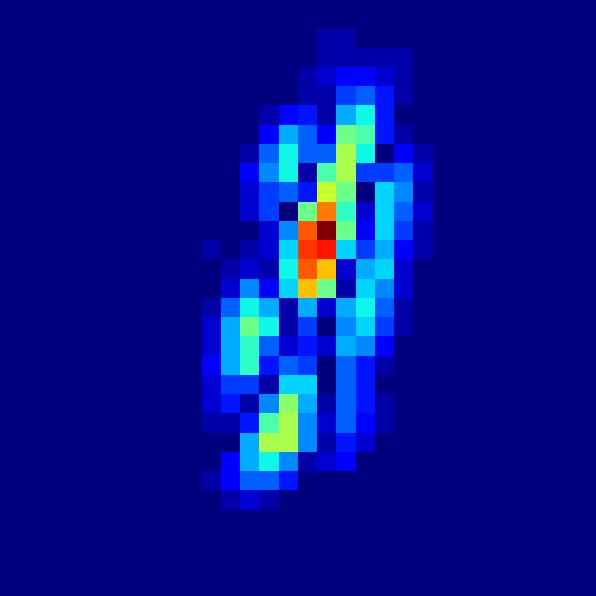}
		\end{minipage}
	}

	\setcounter{subfigure}{0}
	
	    %3 row
	\subfigure{
		\begin{minipage}[t]{0.081\linewidth}
			\centering
			\includegraphics[width=1\linewidth]{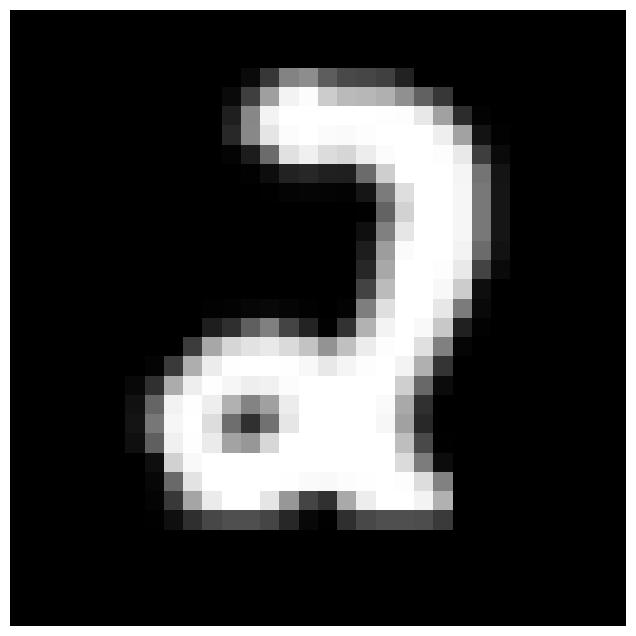}
		\end{minipage}
	}
	\subfigure{
		\begin{minipage}[t]{0.081\linewidth}
			\centering
			\includegraphics[width=1\linewidth]{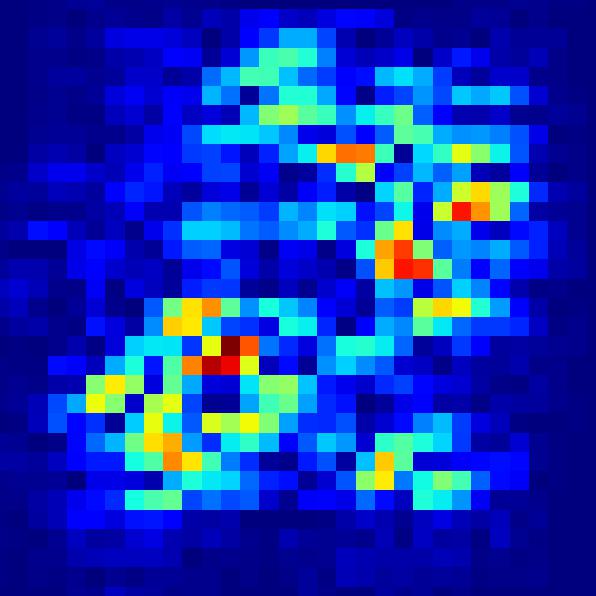}
		\end{minipage}
	}
	\subfigure{
		\begin{minipage}[t]{0.081\linewidth}
			\centering
			\includegraphics[width=1\linewidth]{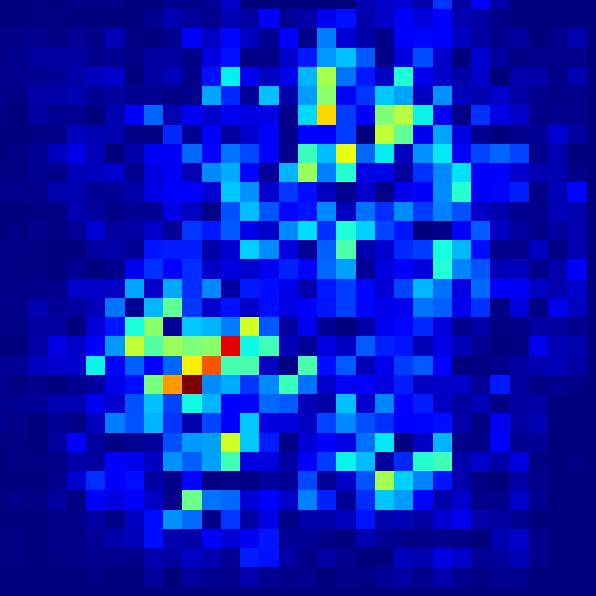}
		\end{minipage}
	}
	\subfigure{
		\begin{minipage}[t]{0.081\linewidth}
			\centering
			\includegraphics[width=1\linewidth]{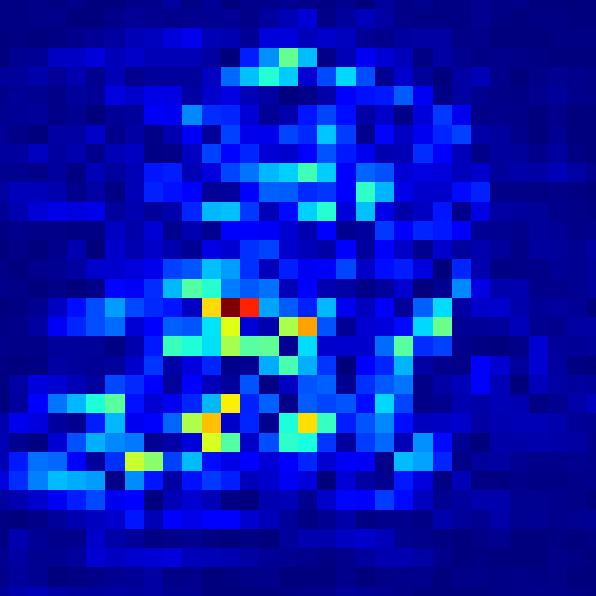}
		\end{minipage}
	}
	\subfigure{
		\begin{minipage}[t]{0.081\linewidth}
			\centering
			\includegraphics[width=1\linewidth]{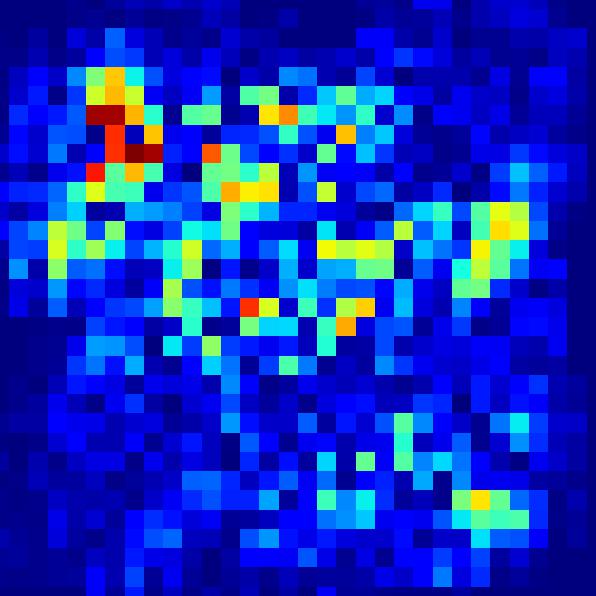}
		\end{minipage}
	}
	\subfigure{
		\begin{minipage}[t]{0.081\linewidth}
			\centering
			\includegraphics[width=1\linewidth]{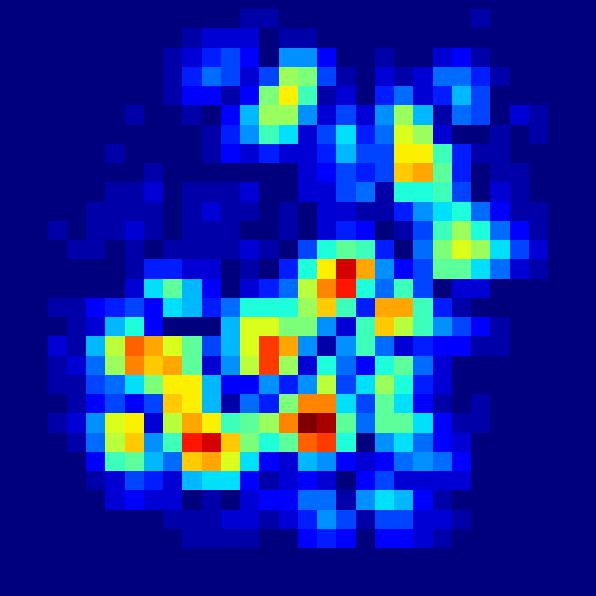}
		\end{minipage}
	}
	\subfigure{
		\begin{minipage}[t]{0.081\linewidth}
			\centering
			\includegraphics[width=1\linewidth]{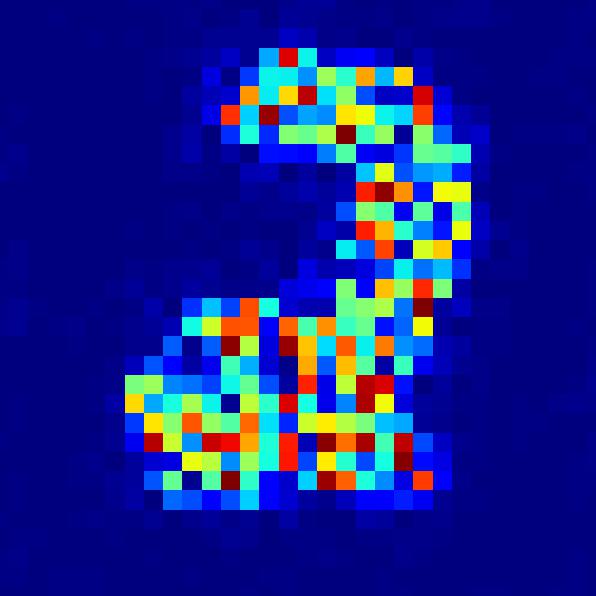}
		\end{minipage}
	}
	\subfigure{
		\begin{minipage}[t]{0.081\linewidth}
			\centering
			\includegraphics[width=1\linewidth]{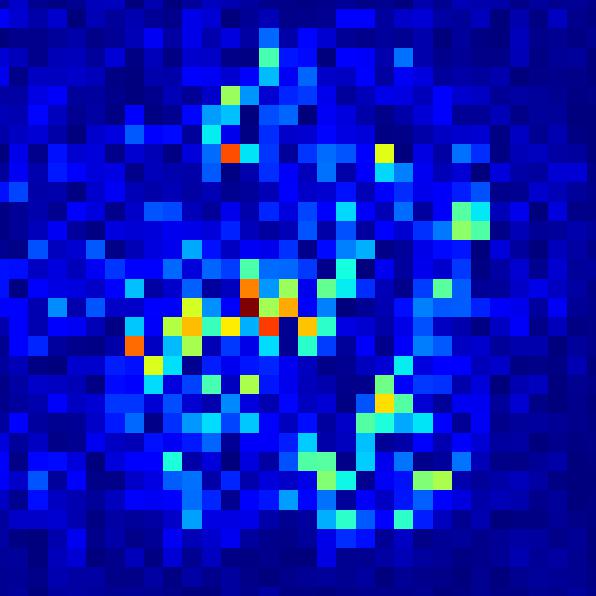}
		\end{minipage}
	}
	\subfigure{
		\begin{minipage}[t]{0.081\linewidth}
			\centering
			\includegraphics[width=1\linewidth]{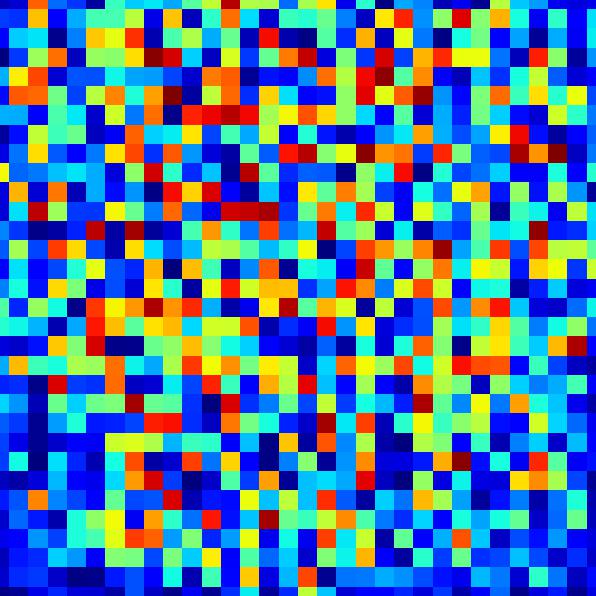}
		\end{minipage}
	}
	\subfigure{
		\begin{minipage}[t]{0.081\linewidth}
			\centering
			\includegraphics[width=1\linewidth]{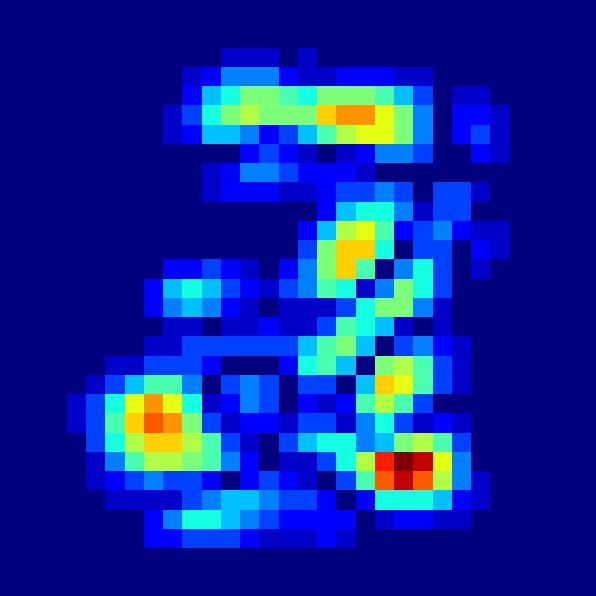}
		\end{minipage}
	}

	\setcounter{subfigure}{0}
	
	    %4row
	\subfigure{
		\begin{minipage}[t]{0.081\linewidth}
			\centering
			\includegraphics[width=1\linewidth]{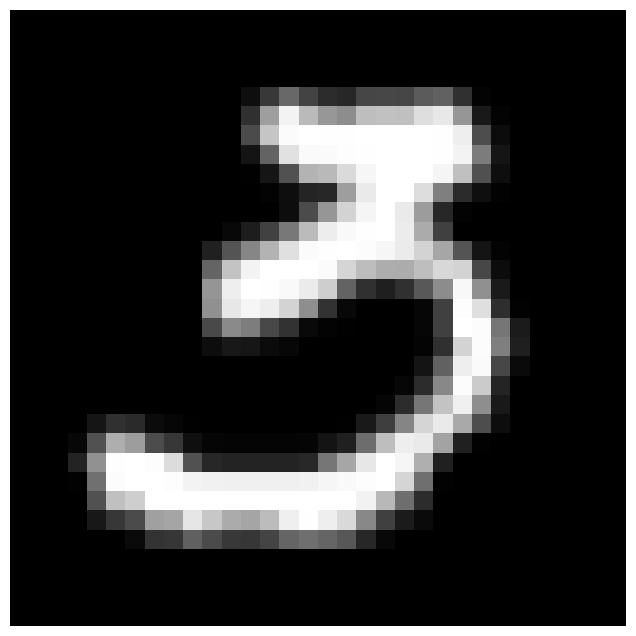}
		\end{minipage}
	}
	\subfigure{
		\begin{minipage}[t]{0.081\linewidth}
			\centering
			\includegraphics[width=1\linewidth]{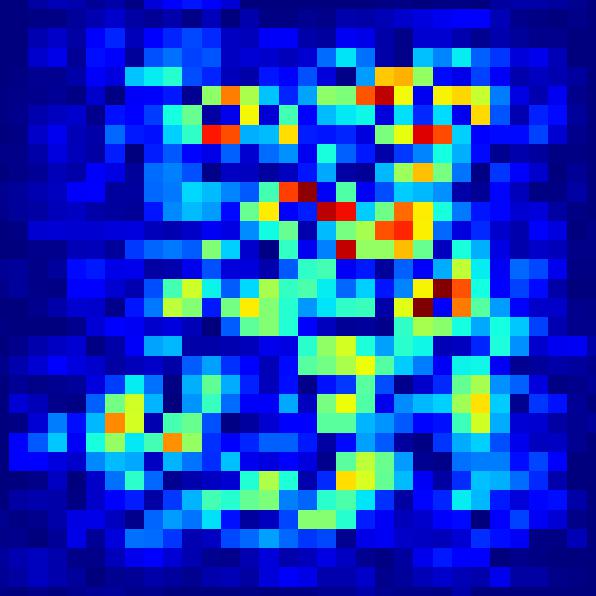}
		\end{minipage}
	}
	\subfigure{
		\begin{minipage}[t]{0.081\linewidth}
			\centering
			\includegraphics[width=1\linewidth]{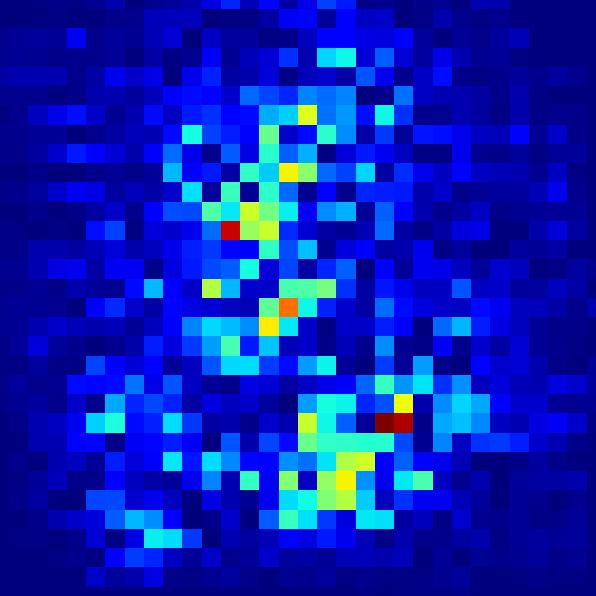}
		\end{minipage}
	}
	\subfigure{
		\begin{minipage}[t]{0.081\linewidth}
			\centering
			\includegraphics[width=1\linewidth]{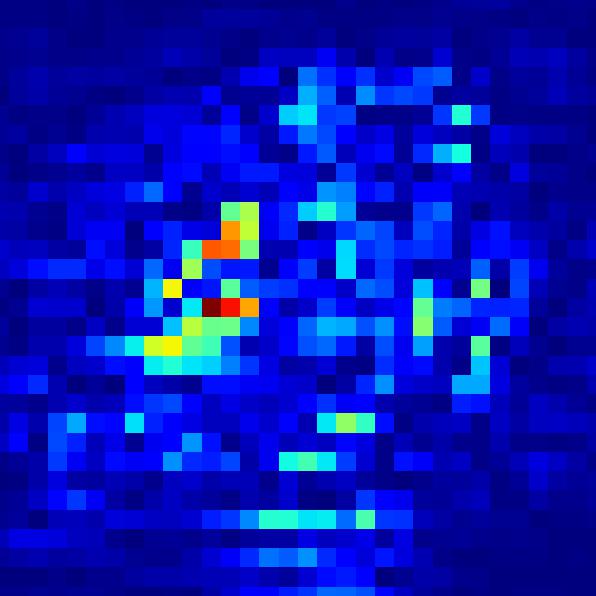}
		\end{minipage}
	}
	\subfigure{
		\begin{minipage}[t]{0.081\linewidth}
			\centering
			\includegraphics[width=1\linewidth]{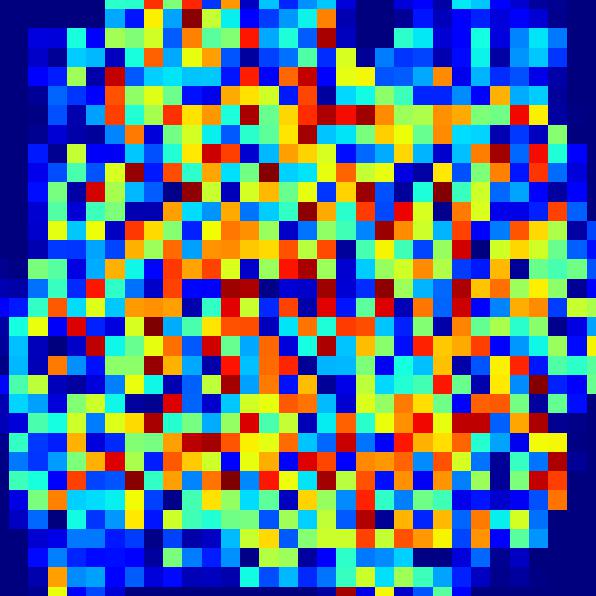}
		\end{minipage}
	}
	\subfigure{
		\begin{minipage}[t]{0.081\linewidth}
			\centering
			\includegraphics[width=1\linewidth]{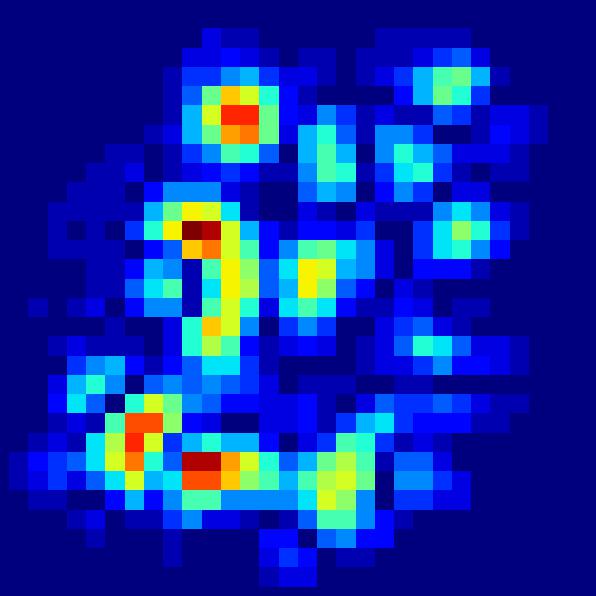}
		\end{minipage}
	}
	\subfigure{
		\begin{minipage}[t]{0.081\linewidth}
			\centering
			\includegraphics[width=1\linewidth]{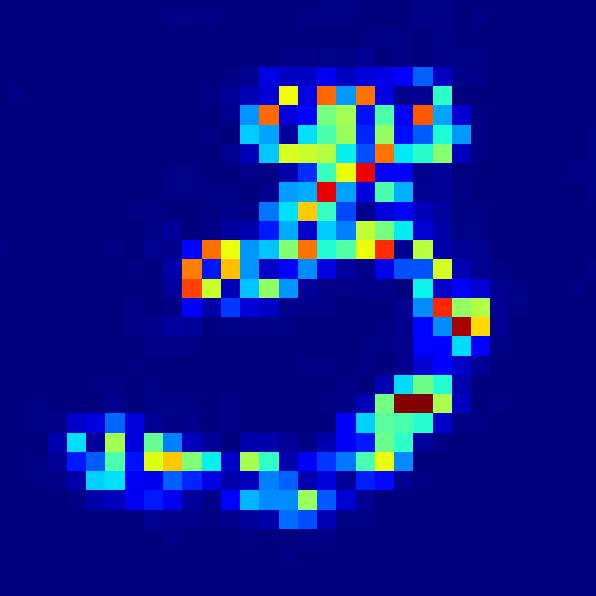}
		\end{minipage}
	}
	\subfigure{
		\begin{minipage}[t]{0.081\linewidth}
			\centering
			\includegraphics[width=1\linewidth]{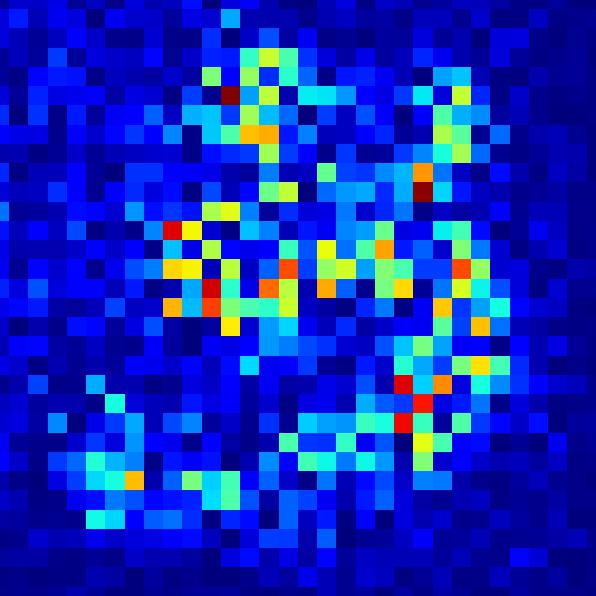}
		\end{minipage}
	}
	\subfigure{
		\begin{minipage}[t]{0.081\linewidth}
			\centering
			\includegraphics[width=1\linewidth]{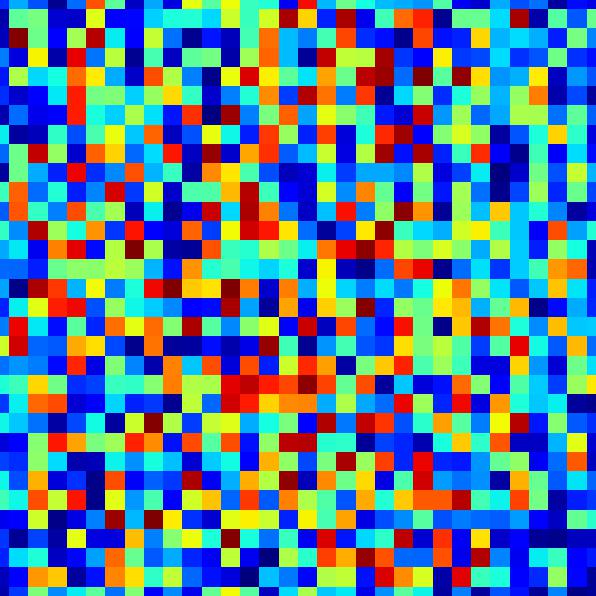}
		\end{minipage}
	}
	\subfigure{
		\begin{minipage}[t]{0.081\linewidth}
			\centering
			\includegraphics[width=1\linewidth]{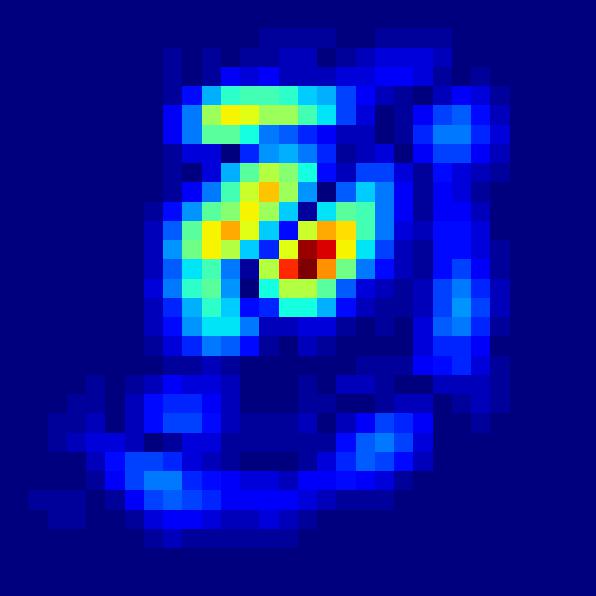}
		\end{minipage}
	}

	\setcounter{subfigure}{0}
	
	    %5 row
	\subfigure{
		\begin{minipage}[t]{0.081\linewidth}
			\centering
			\includegraphics[width=1\linewidth]{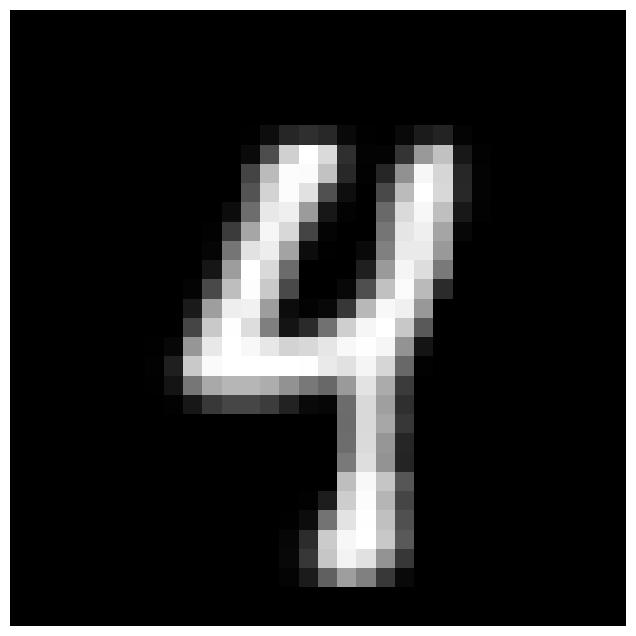}
		\end{minipage}
	}
	\subfigure{
		\begin{minipage}[t]{0.081\linewidth}
			\centering
			\includegraphics[width=1\linewidth]{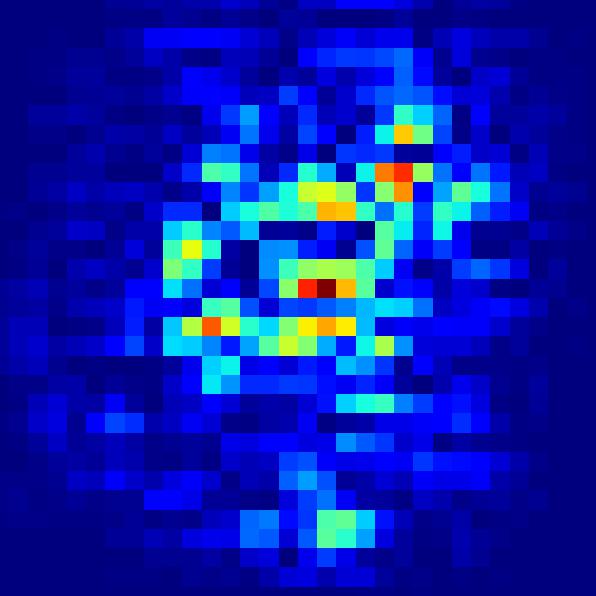}
		\end{minipage}
	}
	\subfigure{
		\begin{minipage}[t]{0.081\linewidth}
			\centering
			\includegraphics[width=1\linewidth]{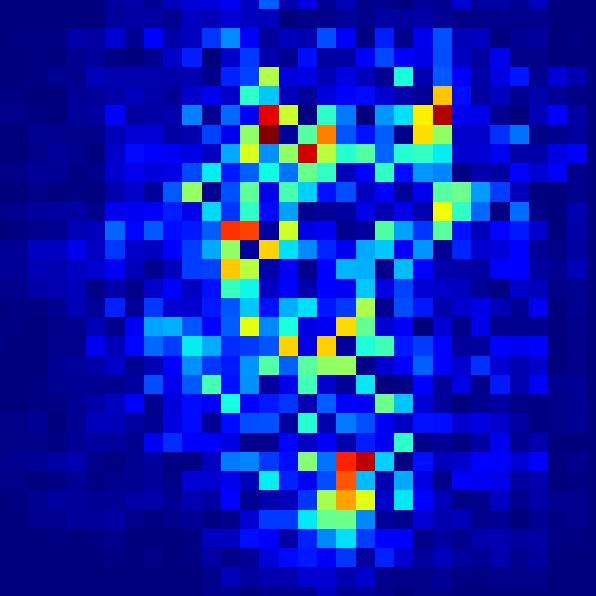}
		\end{minipage}
	}
	\subfigure{
		\begin{minipage}[t]{0.081\linewidth}
			\centering
			\includegraphics[width=1\linewidth]{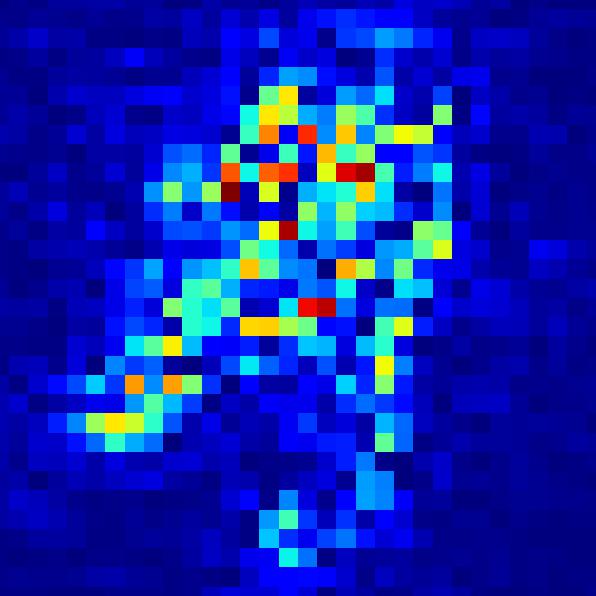}
		\end{minipage}
	}
	\subfigure{
		\begin{minipage}[t]{0.081\linewidth}
			\centering
			\includegraphics[width=1\linewidth]{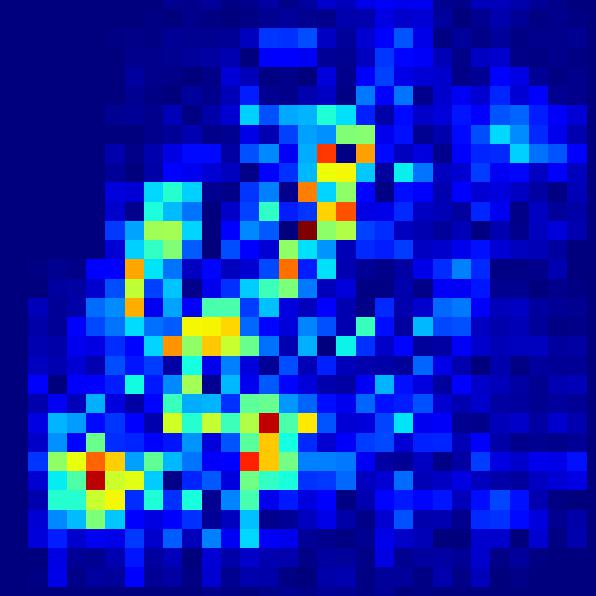}
		\end{minipage}
	}
	\subfigure{
		\begin{minipage}[t]{0.081\linewidth}
			\centering
			\includegraphics[width=1\linewidth]{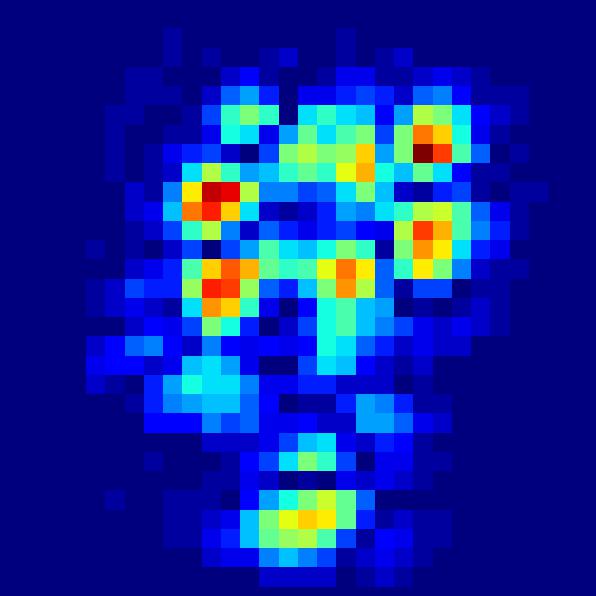}
		\end{minipage}
	}
	\subfigure{
		\begin{minipage}[t]{0.081\linewidth}
			\centering
			\includegraphics[width=1\linewidth]{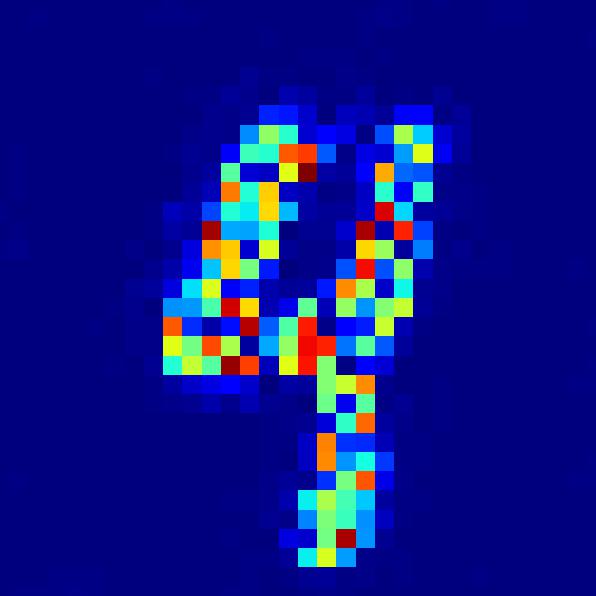}
		\end{minipage}
	}
	\subfigure{
		\begin{minipage}[t]{0.081\linewidth}
			\centering
			\includegraphics[width=1\linewidth]{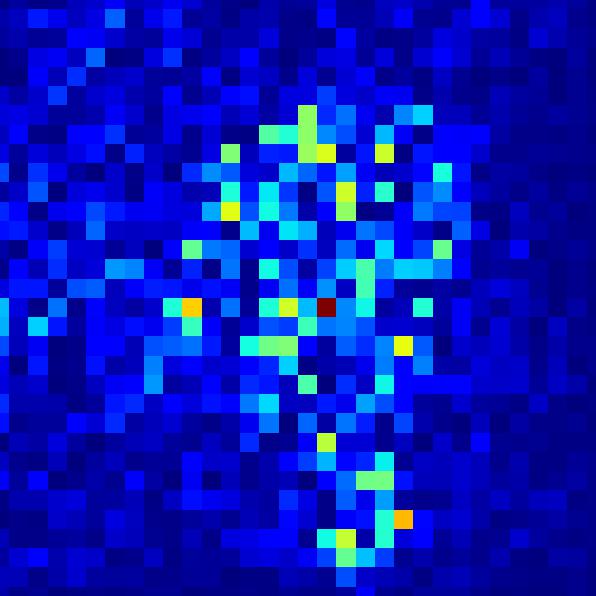}
		\end{minipage}
	}
	\subfigure{
		\begin{minipage}[t]{0.081\linewidth}
			\centering
			\includegraphics[width=1\linewidth]{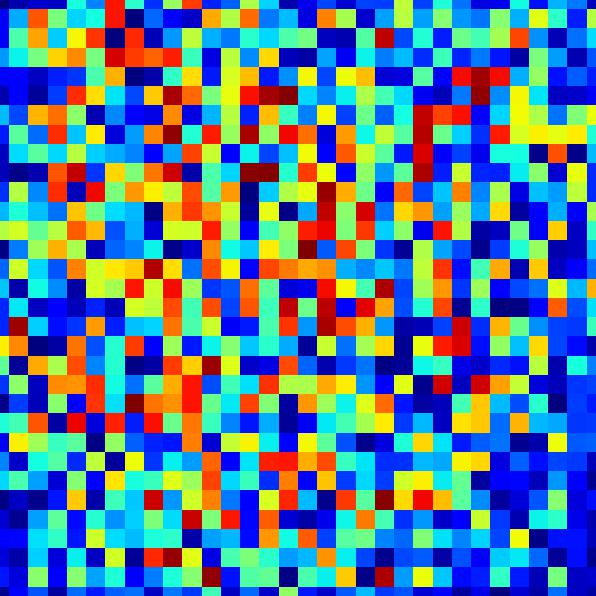}
		\end{minipage}
	}
	\subfigure{
		\begin{minipage}[t]{0.081\linewidth}
			\centering
			\includegraphics[width=1\linewidth]{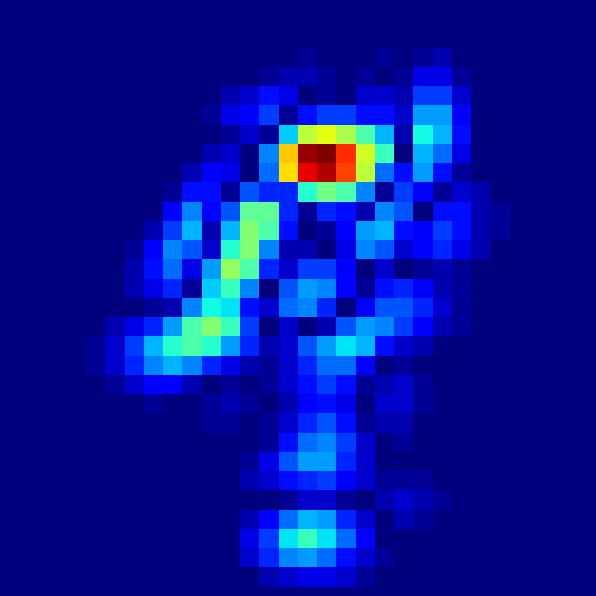}
		\end{minipage}
	}

	\setcounter{subfigure}{0}
	
	    %6 row
	\subfigure{
		\begin{minipage}[t]{0.081\linewidth}
			\centering
			\includegraphics[width=1\linewidth]{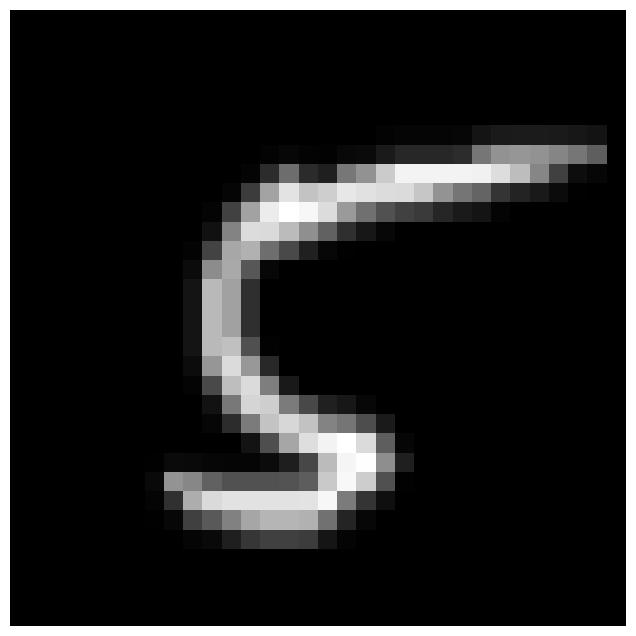}
		\end{minipage}
	}
	\subfigure{
		\begin{minipage}[t]{0.081\linewidth}
			\centering
			\includegraphics[width=1\linewidth]{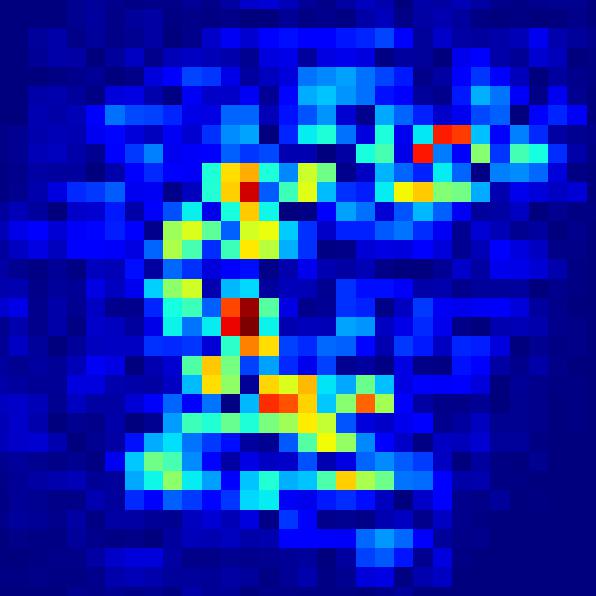}
		\end{minipage}
	}
	\subfigure{
		\begin{minipage}[t]{0.081\linewidth}
			\centering
			\includegraphics[width=1\linewidth]{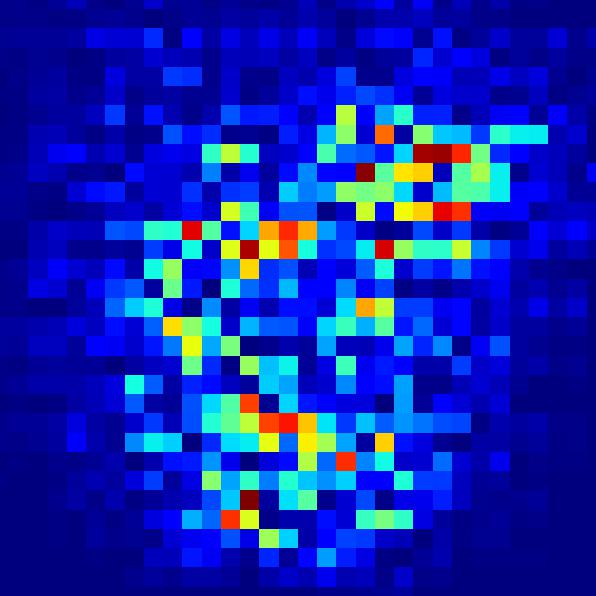}
		\end{minipage}
	}
	\subfigure{
		\begin{minipage}[t]{0.081\linewidth}
			\centering
			\includegraphics[width=1\linewidth]{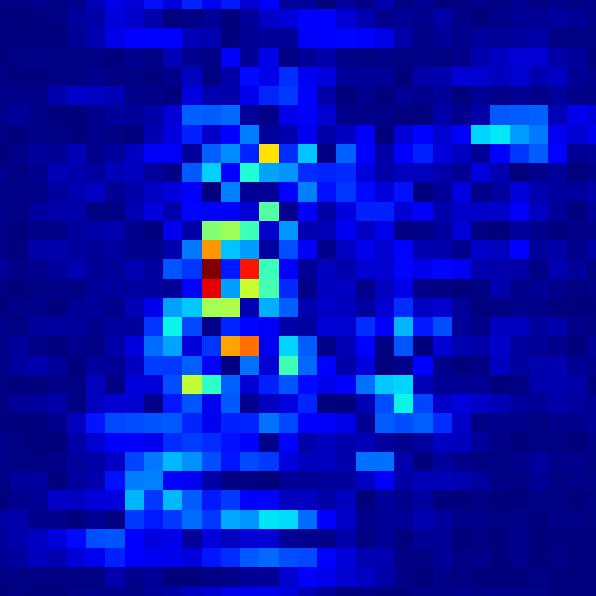}
		\end{minipage}
	}
	\subfigure{
		\begin{minipage}[t]{0.081\linewidth}
			\centering
			\includegraphics[width=1\linewidth]{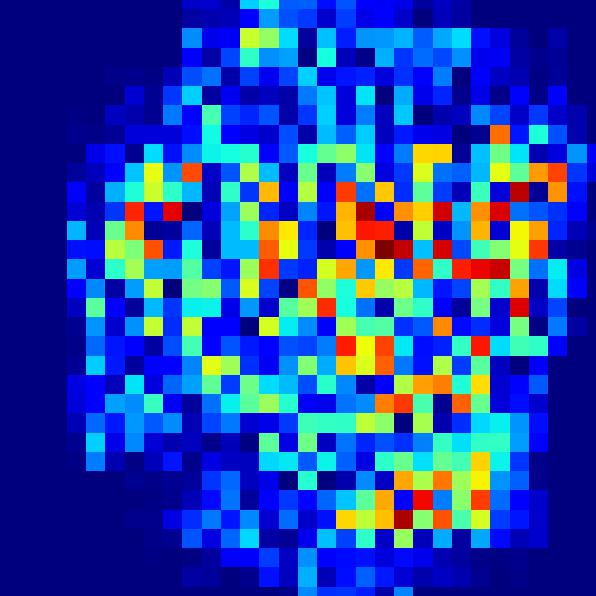}
		\end{minipage}
	}
	\subfigure{
		\begin{minipage}[t]{0.081\linewidth}
			\centering
			\includegraphics[width=1\linewidth]{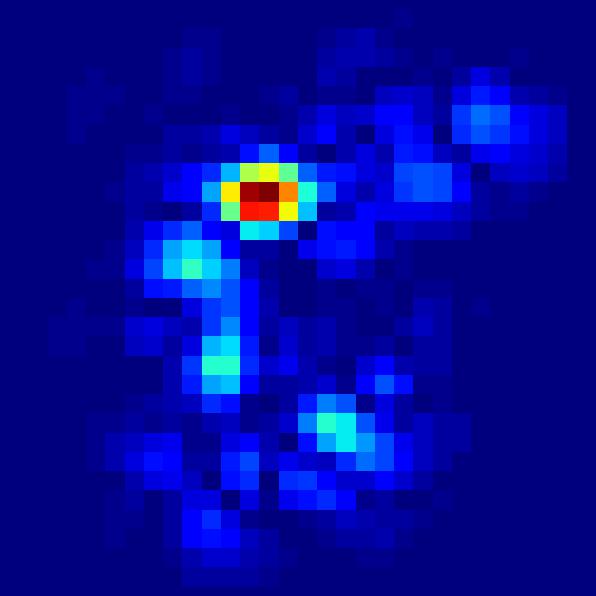}
		\end{minipage}
	}
	\subfigure{
		\begin{minipage}[t]{0.081\linewidth}
			\centering
			\includegraphics[width=1\linewidth]{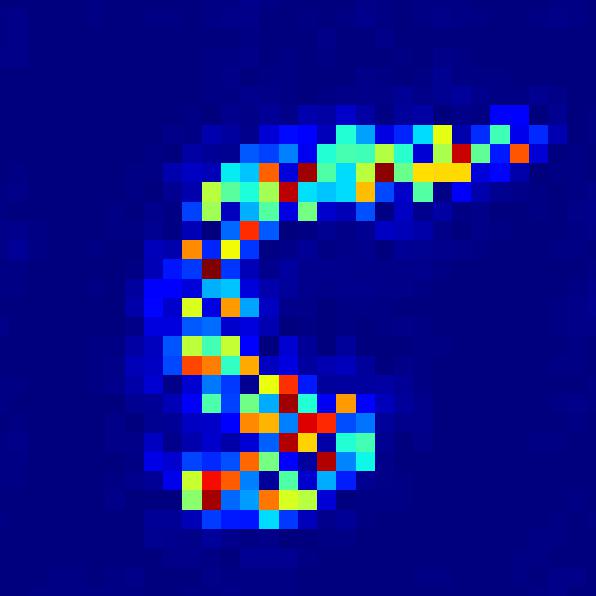}
		\end{minipage}
	}
	\subfigure{
		\begin{minipage}[t]{0.081\linewidth}
			\centering
			\includegraphics[width=1\linewidth]{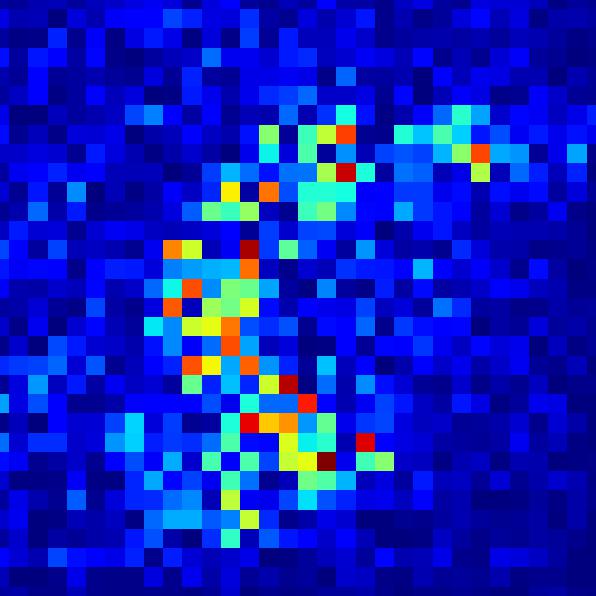}
		\end{minipage}
	}
	\subfigure{
		\begin{minipage}[t]{0.081\linewidth}
			\centering
			\includegraphics[width=1\linewidth]{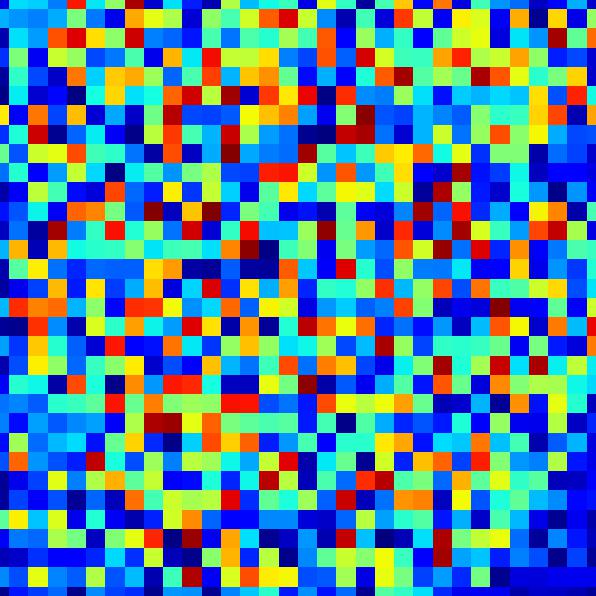}
		\end{minipage}
	}
	\subfigure{
		\begin{minipage}[t]{0.081\linewidth}
			\centering
			\includegraphics[width=1\linewidth]{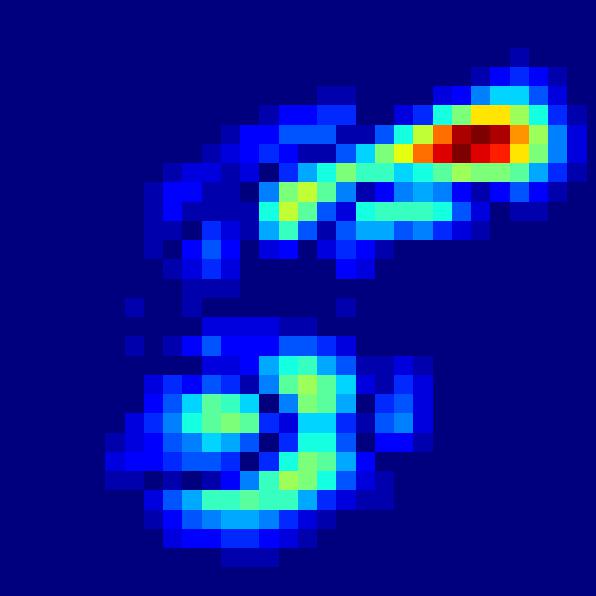}
		\end{minipage}
	}

	\setcounter{subfigure}{0}
	
	    %7 row
	\subfigure{
		\begin{minipage}[t]{0.081\linewidth}
			\centering
			\includegraphics[width=1\linewidth]{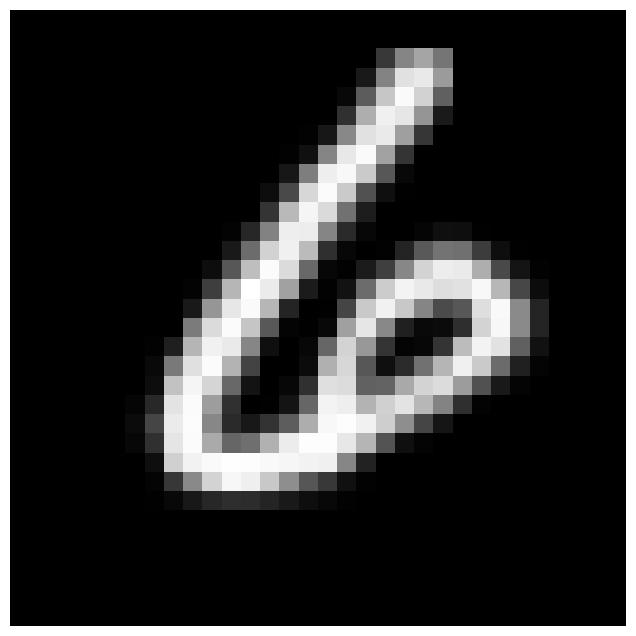}
		\end{minipage}
	}
	\subfigure{
		\begin{minipage}[t]{0.081\linewidth}
			\centering
			\includegraphics[width=1\linewidth]{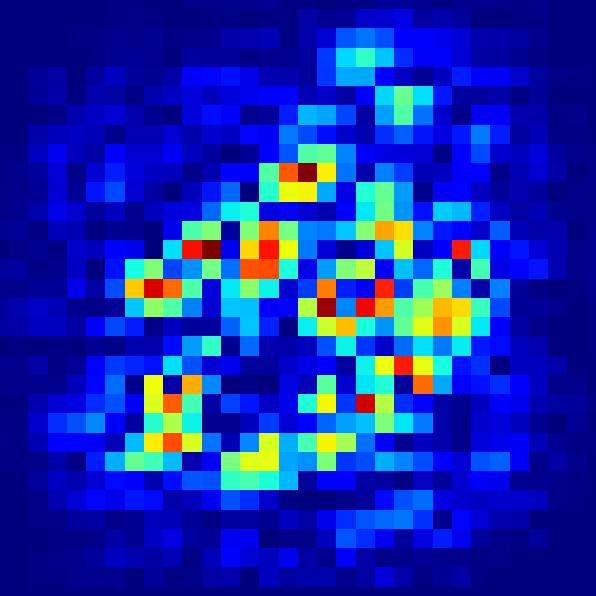}
		\end{minipage}
	}
	\subfigure{
		\begin{minipage}[t]{0.081\linewidth}
			\centering
			\includegraphics[width=1\linewidth]{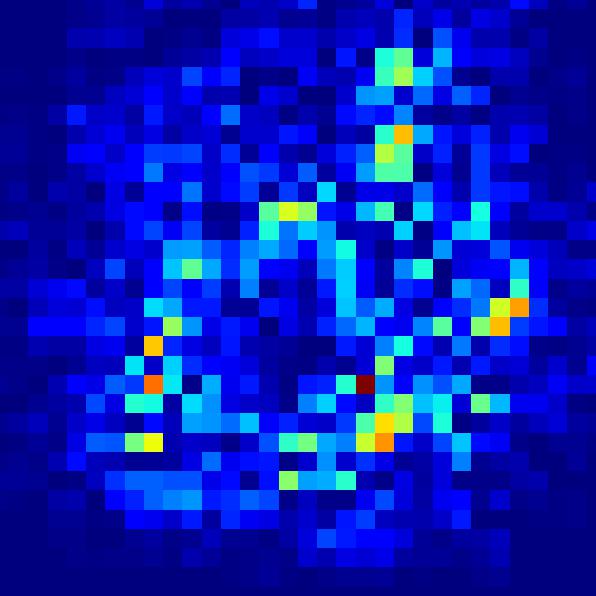}
		\end{minipage}
	}
	\subfigure{
		\begin{minipage}[t]{0.081\linewidth}
			\centering
			\includegraphics[width=1\linewidth]{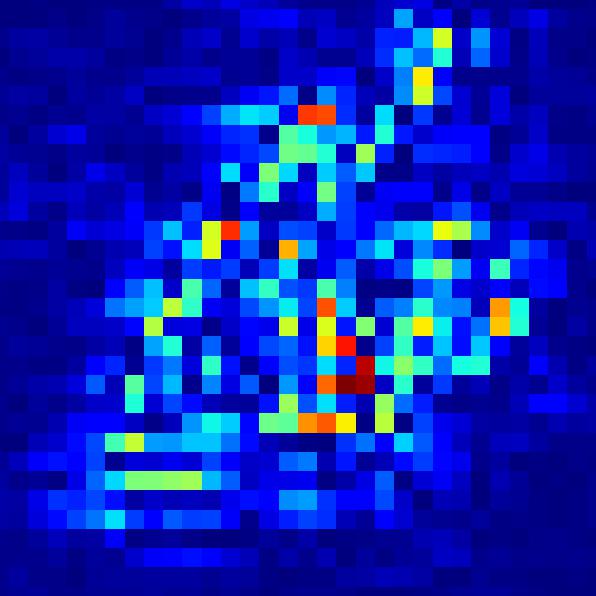}
		\end{minipage}
	}
	\subfigure{
		\begin{minipage}[t]{0.081\linewidth}
			\centering
			\includegraphics[width=1\linewidth]{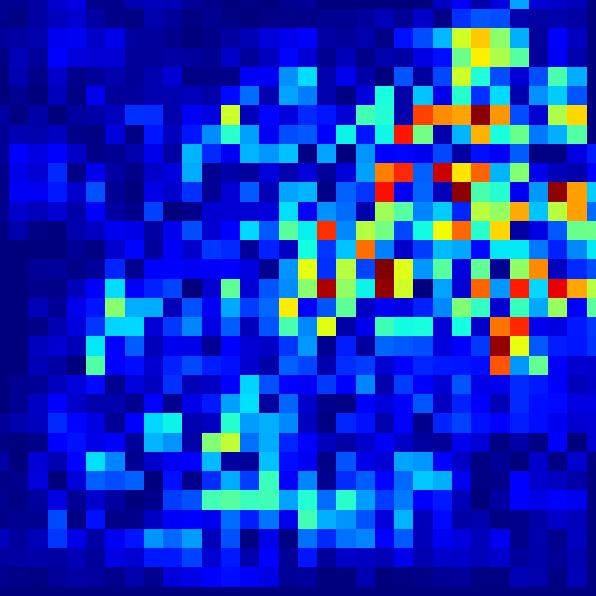}
		\end{minipage}
	}
	\subfigure{
		\begin{minipage}[t]{0.081\linewidth}
			\centering
			\includegraphics[width=1\linewidth]{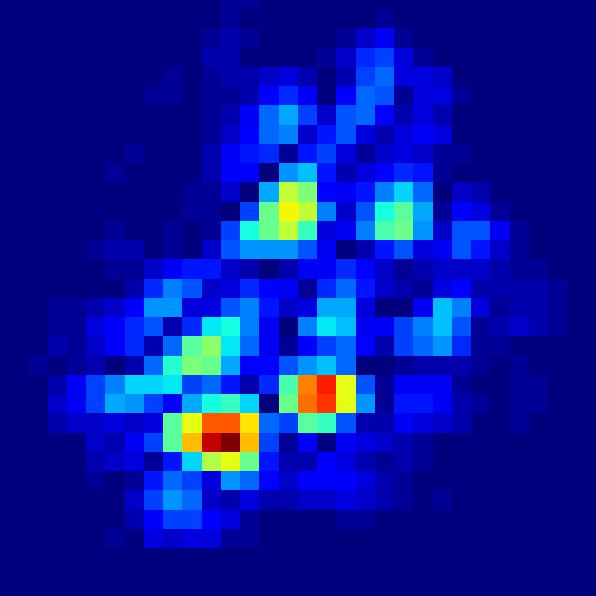}
		\end{minipage}
	}
	\subfigure{
		\begin{minipage}[t]{0.081\linewidth}
			\centering
			\includegraphics[width=1\linewidth]{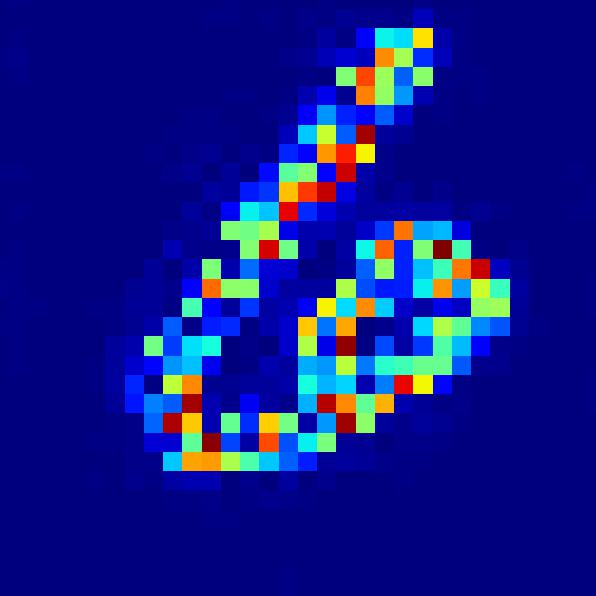}
		\end{minipage}
	}
	\subfigure{
		\begin{minipage}[t]{0.081\linewidth}
			\centering
			\includegraphics[width=1\linewidth]{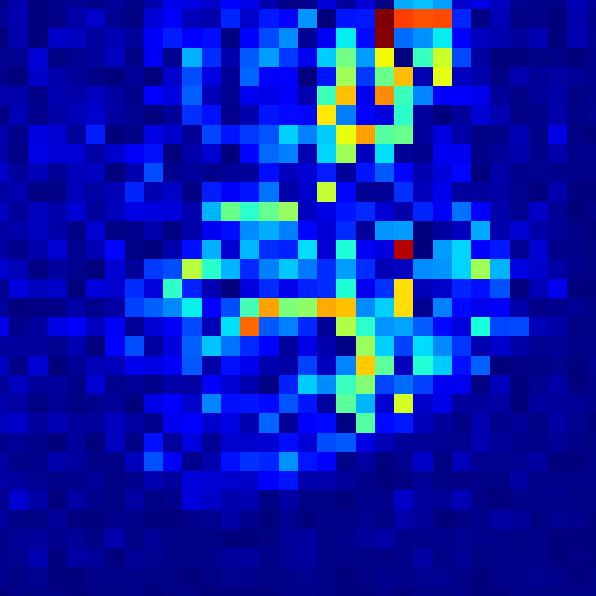}
		\end{minipage}
	}
	\subfigure{
		\begin{minipage}[t]{0.081\linewidth}
			\centering
			\includegraphics[width=1\linewidth]{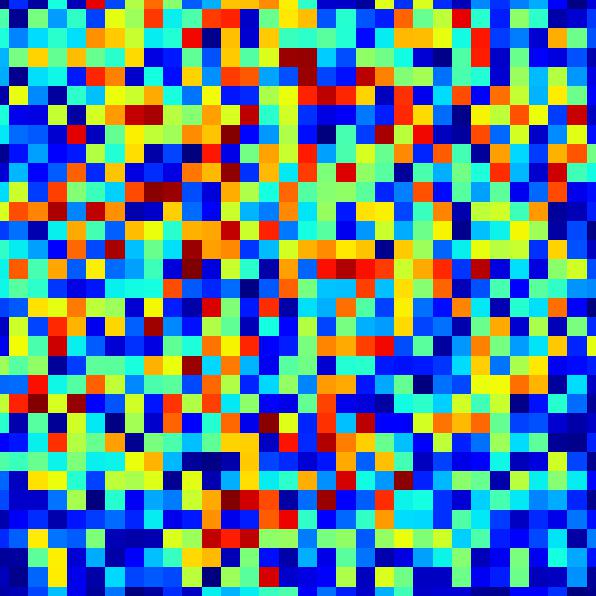}
		\end{minipage}
	}
	\subfigure{
		\begin{minipage}[t]{0.081\linewidth}
			\centering
			\includegraphics[width=1\linewidth]{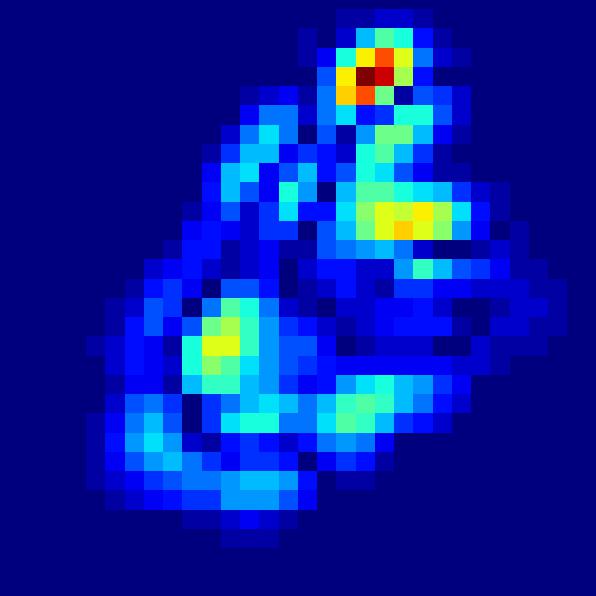}
		\end{minipage}
	}
	
	\setcounter{subfigure}{0}
	
	    %8 row
	\subfigure{
		\begin{minipage}[t]{0.081\linewidth}
			\centering
			\includegraphics[width=1\linewidth]{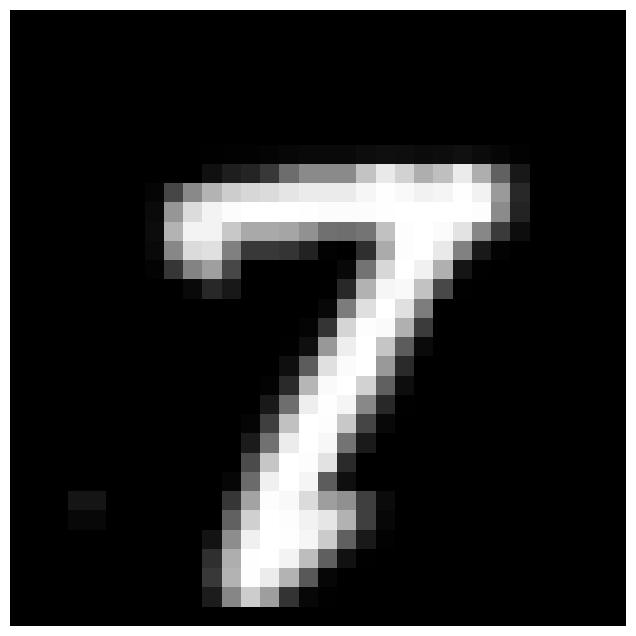}
		\end{minipage}
	}
	\subfigure{
		\begin{minipage}[t]{0.081\linewidth}
			\centering
			\includegraphics[width=1\linewidth]{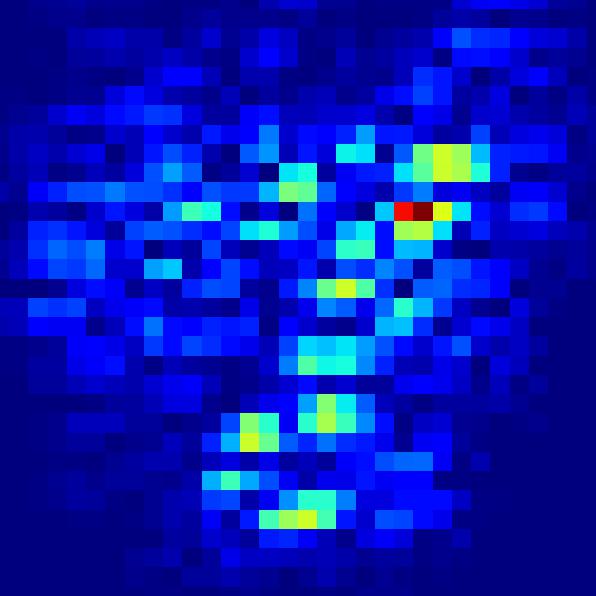}
		\end{minipage}
	}
	\subfigure{
		\begin{minipage}[t]{0.081\linewidth}
			\centering
			\includegraphics[width=1\linewidth]{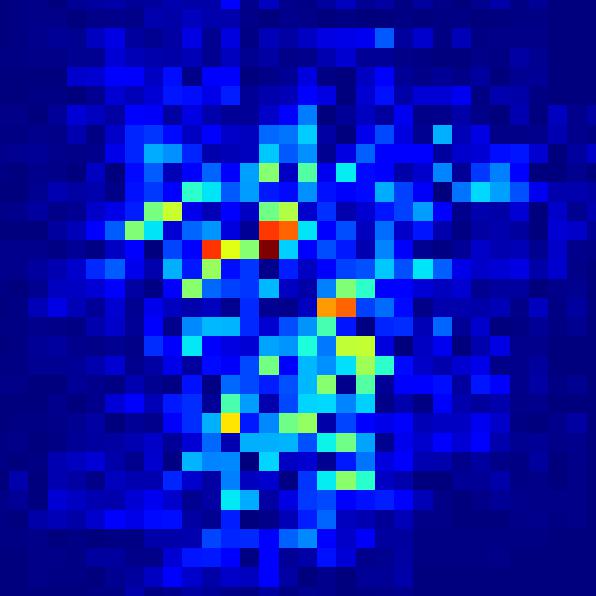}
		\end{minipage}
	}
	\subfigure{
		\begin{minipage}[t]{0.081\linewidth}
			\centering
			\includegraphics[width=1\linewidth]{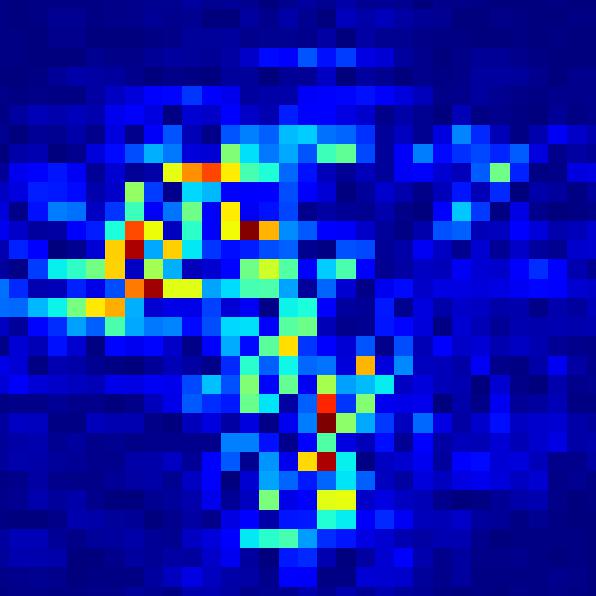}
		\end{minipage}
	}
	\subfigure{
		\begin{minipage}[t]{0.081\linewidth}
			\centering
			\includegraphics[width=1\linewidth]{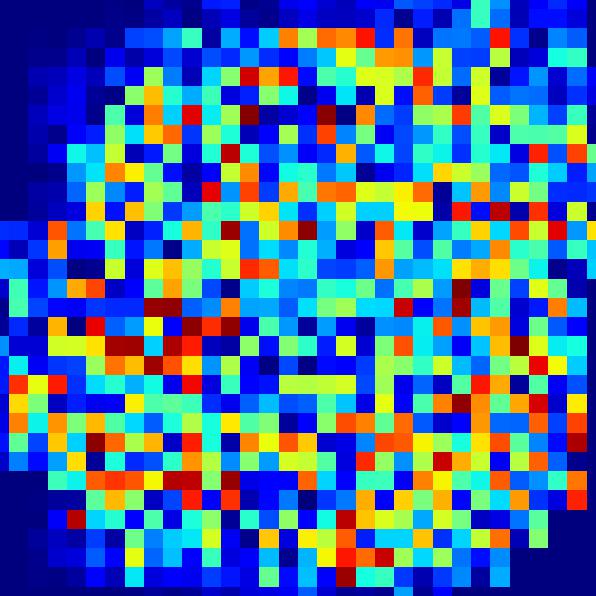}
		\end{minipage}
	}
	\subfigure{
		\begin{minipage}[t]{0.081\linewidth}
			\centering
			\includegraphics[width=1\linewidth]{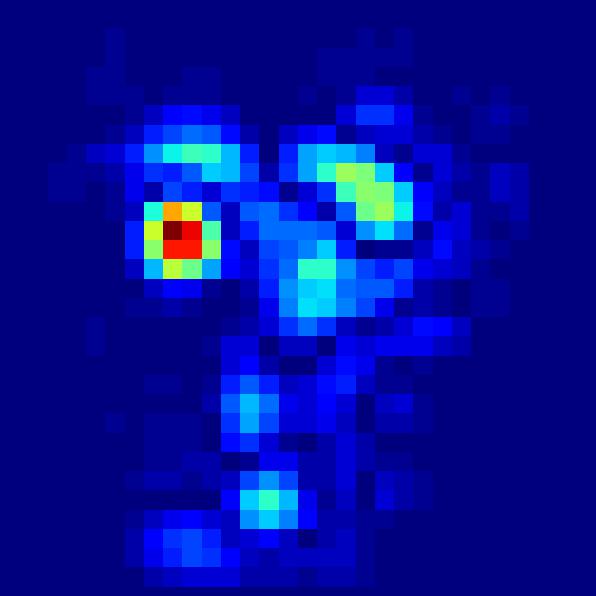}
		\end{minipage}
	}
	\subfigure{
		\begin{minipage}[t]{0.081\linewidth}
			\centering
			\includegraphics[width=1\linewidth]{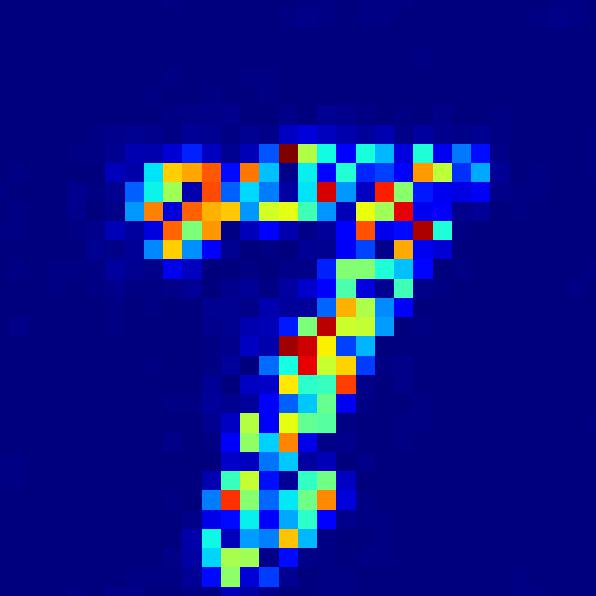}
		\end{minipage}
	}
	\subfigure{
		\begin{minipage}[t]{0.081\linewidth}
			\centering
			\includegraphics[width=1\linewidth]{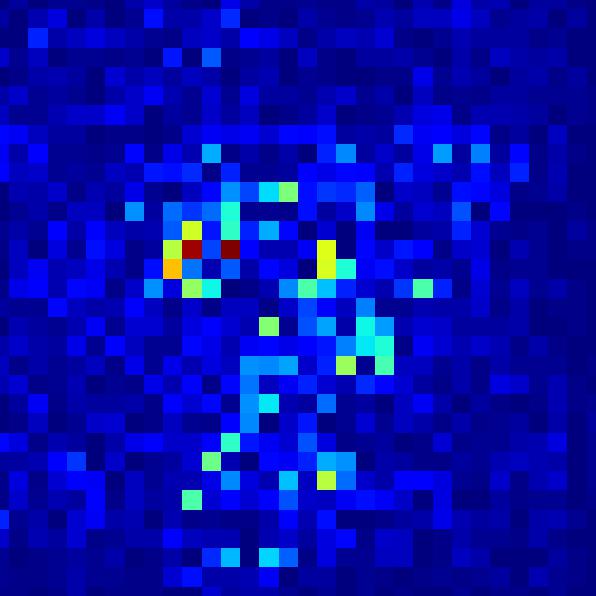}
		\end{minipage}
	}
	\subfigure{
		\begin{minipage}[t]{0.081\linewidth}
			\centering
			\includegraphics[width=1\linewidth]{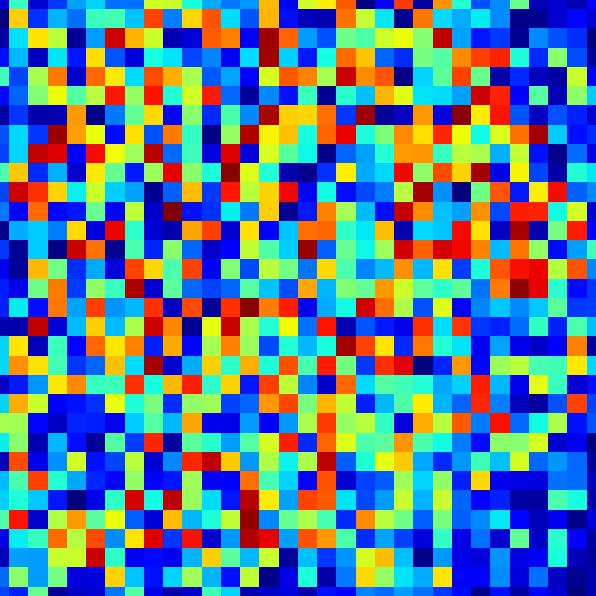}
		\end{minipage}
	}
	\subfigure{
		\begin{minipage}[t]{0.081\linewidth}
			\centering
			\includegraphics[width=1\linewidth]{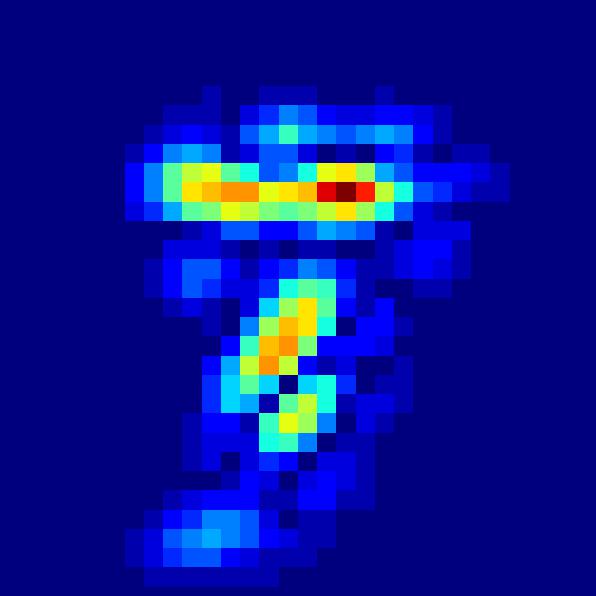}
		\end{minipage}
	}

	\setcounter{subfigure}{0}
	
	    %9 row
	\subfigure{
		\begin{minipage}[t]{0.081\linewidth}
			\centering
			\includegraphics[width=1\linewidth]{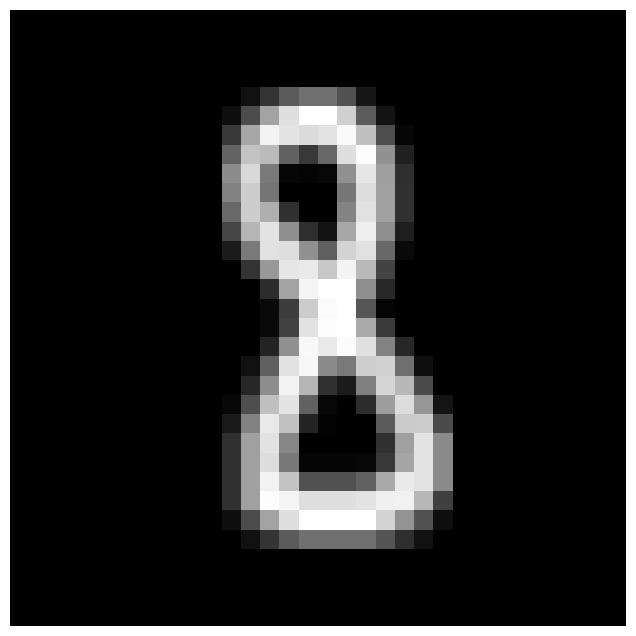}
		\end{minipage}
	}
	\subfigure{
		\begin{minipage}[t]{0.081\linewidth}
			\centering
			\includegraphics[width=1\linewidth]{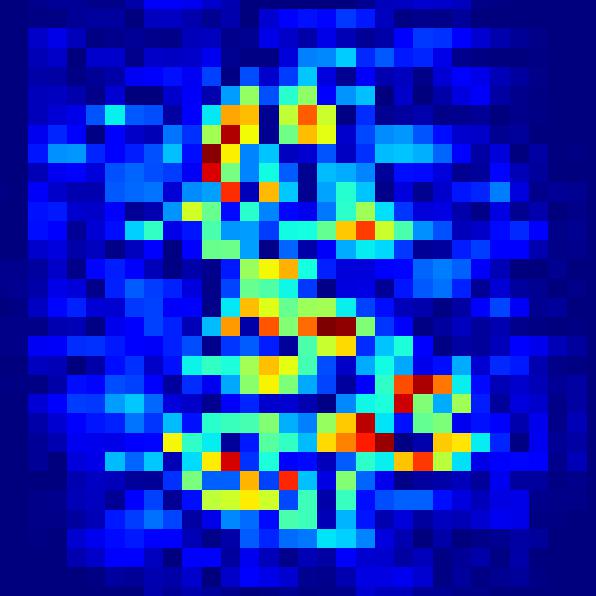}
		\end{minipage}
	}
	\subfigure{
		\begin{minipage}[t]{0.081\linewidth}
			\centering
			\includegraphics[width=1\linewidth]{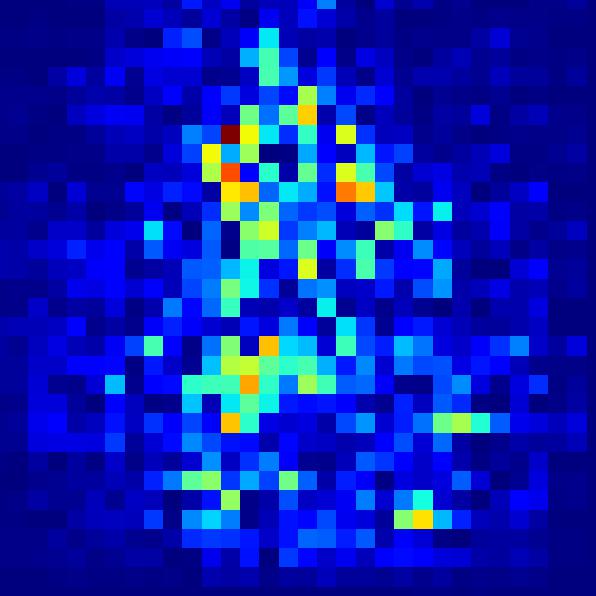}
		\end{minipage}
	}
	\subfigure{
		\begin{minipage}[t]{0.081\linewidth}
			\centering
			\includegraphics[width=1\linewidth]{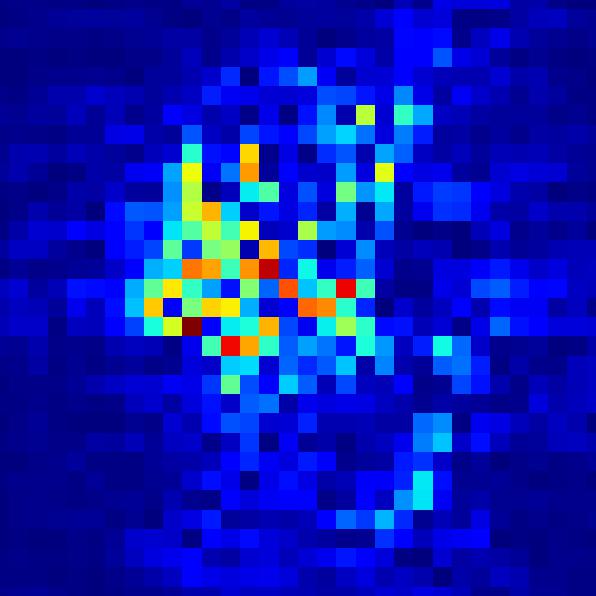}
		\end{minipage}
	}
	\subfigure{
		\begin{minipage}[t]{0.081\linewidth}
			\centering
			\includegraphics[width=1\linewidth]{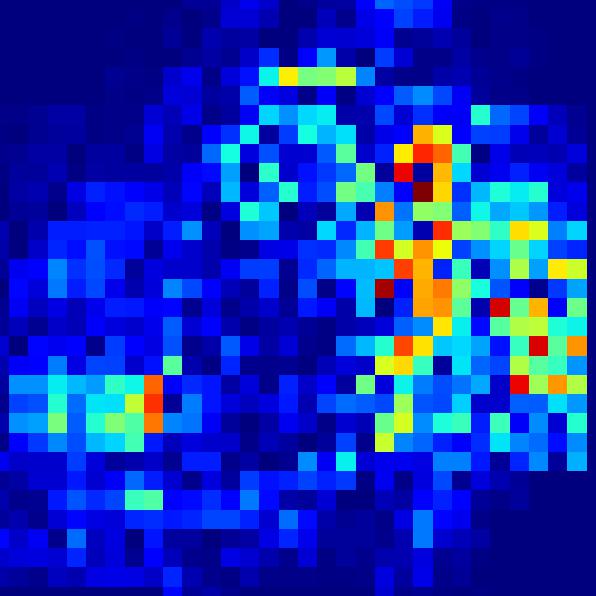}
		\end{minipage}
	}
	\subfigure{
		\begin{minipage}[t]{0.081\linewidth}
			\centering
			\includegraphics[width=1\linewidth]{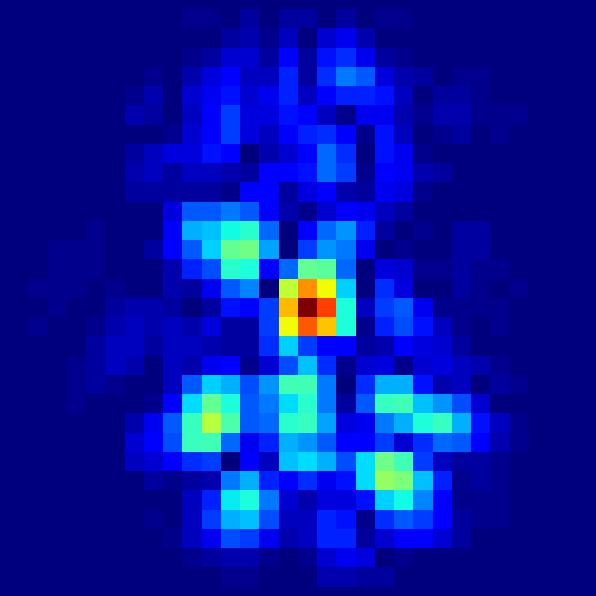}
		\end{minipage}
	}
	\subfigure{
		\begin{minipage}[t]{0.081\linewidth}
			\centering
			\includegraphics[width=1\linewidth]{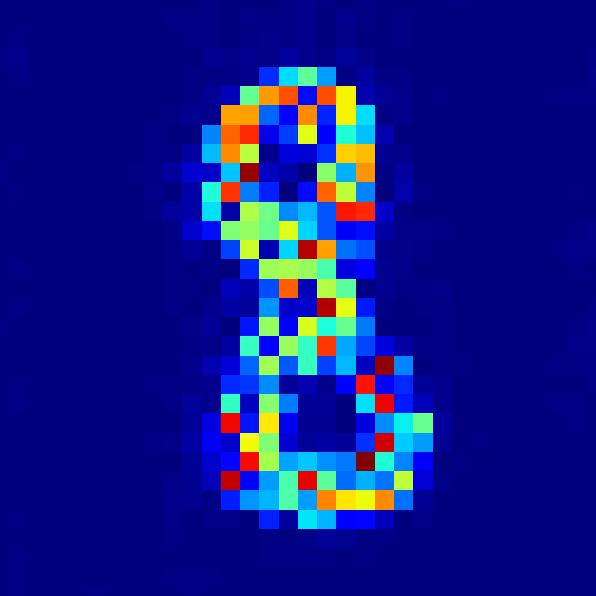}
		\end{minipage}
	}
	\subfigure{
		\begin{minipage}[t]{0.081\linewidth}
			\centering
			\includegraphics[width=1\linewidth]{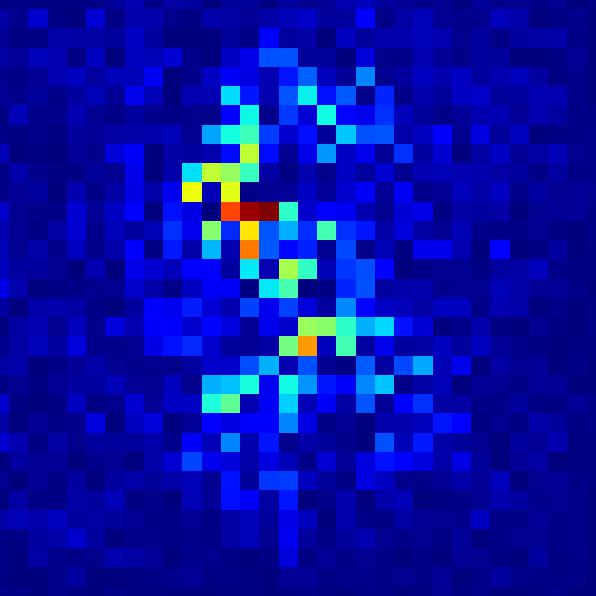}
		\end{minipage}
	}
	\subfigure{
		\begin{minipage}[t]{0.081\linewidth}
			\centering
			\includegraphics[width=1\linewidth]{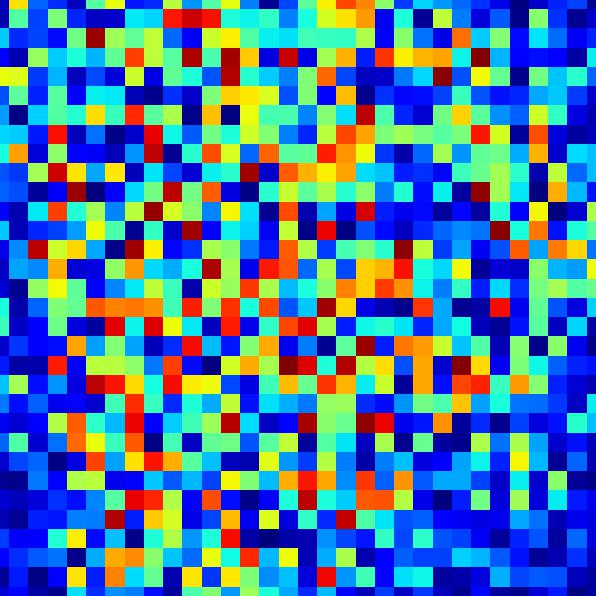}
		\end{minipage}
	}
	\subfigure{
		\begin{minipage}[t]{0.081\linewidth}
			\centering
			\includegraphics[width=1\linewidth]{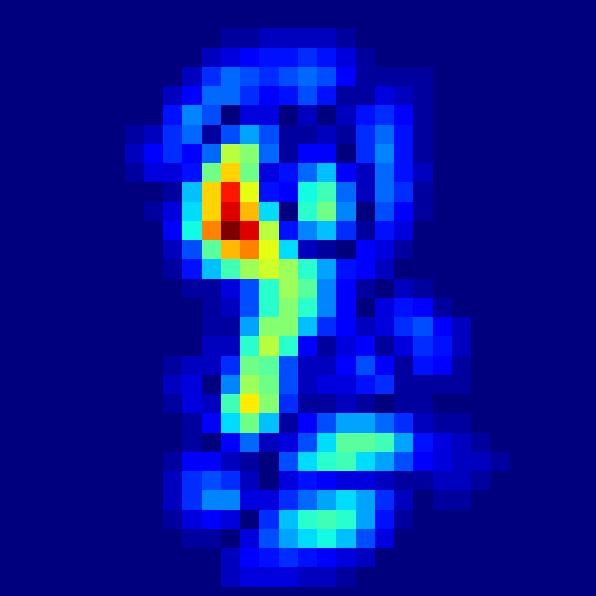}
		\end{minipage}
	}

	\setcounter{subfigure}{0}

    % last row
    \subfigure[Image]{
		\begin{minipage}[t]{0.081\linewidth}
			\centering
			\includegraphics[width=1\linewidth]{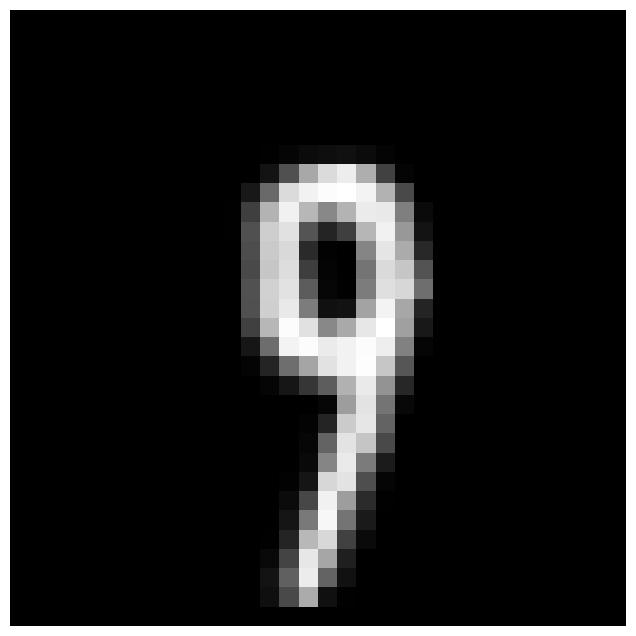}
		\end{minipage}
	}
	\subfigure[BL]{
		\begin{minipage}[t]{0.081\linewidth}
			\centering
			\includegraphics[width=1\linewidth]{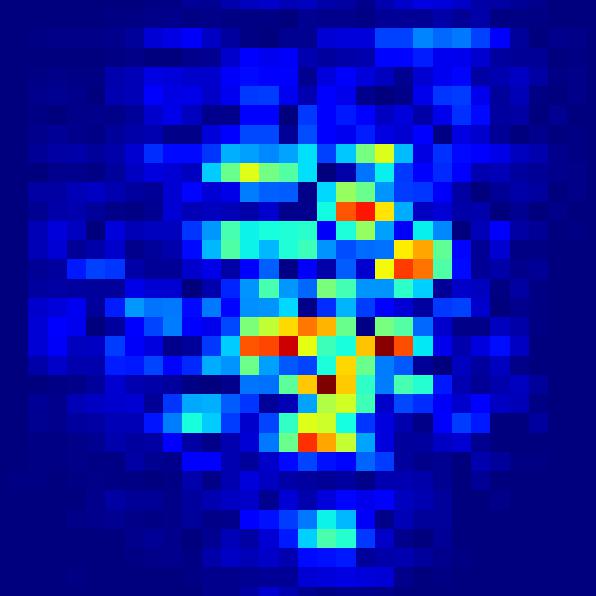}
		\end{minipage}
	}
	\subfigure[SR]{
		\begin{minipage}[t]{0.081\linewidth}
			\centering
			\includegraphics[width=1\linewidth]{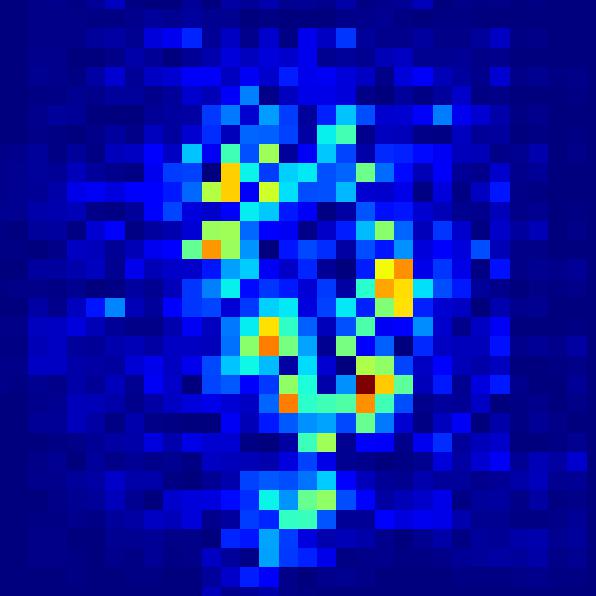}
		\end{minipage}
	}
	\subfigure[CT]{
		\begin{minipage}[t]{0.081\linewidth}
			\centering
			\includegraphics[width=1\linewidth]{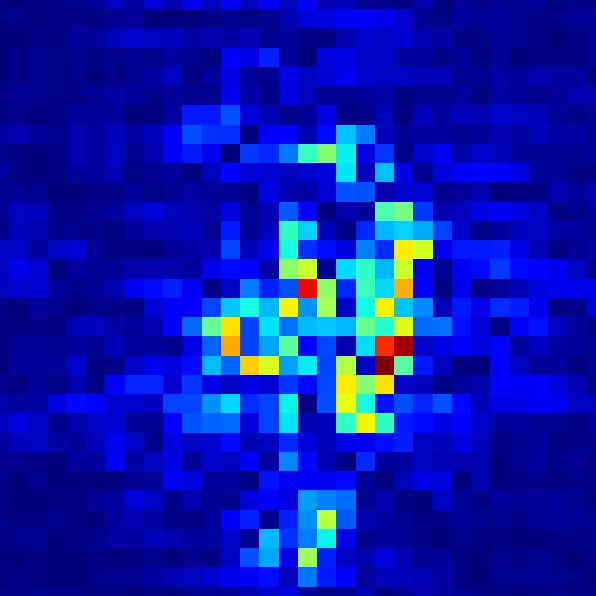}
		\end{minipage}
	}
	\subfigure[CT+]{
		\begin{minipage}[t]{0.08\linewidth}
			\centering
			\includegraphics[width=1\linewidth]{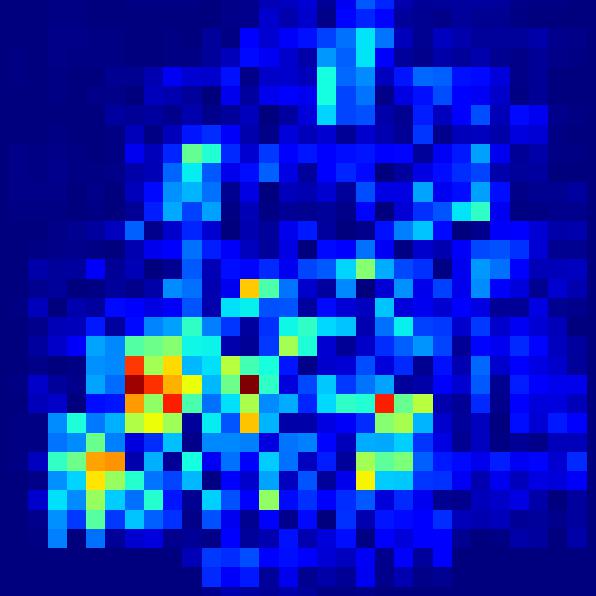}
		\end{minipage}
	}
	\subfigure[ES]{
		\begin{minipage}[t]{0.08\linewidth}
			\centering
			\includegraphics[width=1\linewidth]{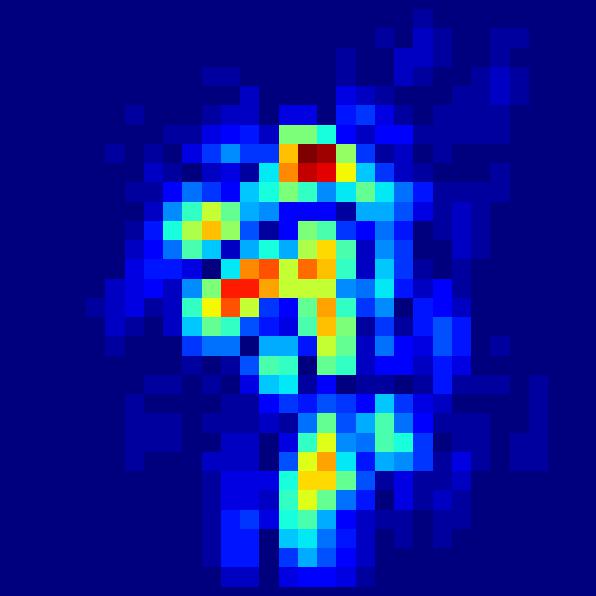}
		\end{minipage}
	}
	\subfigure[ELR]{
		\begin{minipage}[t]{0.08\linewidth}
			\centering
			\includegraphics[width=1\linewidth]{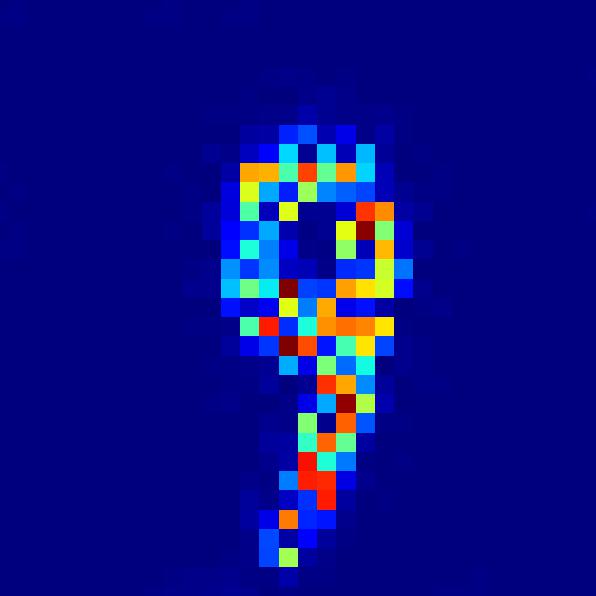}
		\end{minipage}
	}
	\subfigure[PES]{
		\begin{minipage}[t]{0.08\linewidth}
			\centering
			\includegraphics[width=1\linewidth]{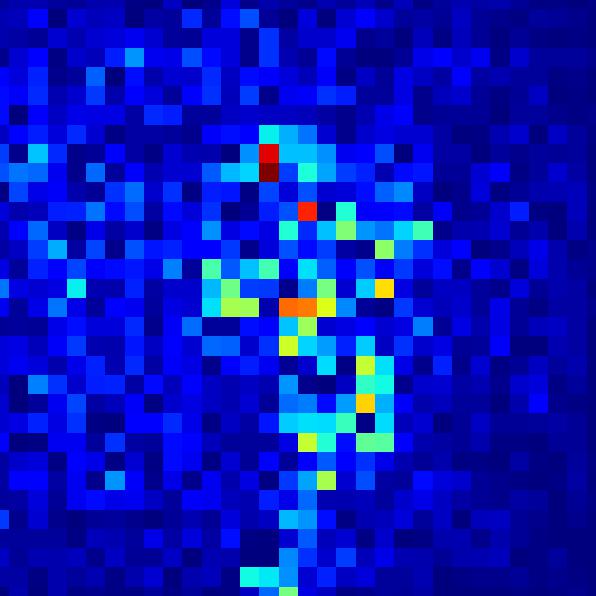}
		\end{minipage}
	}
	\subfigure[NCT]{
		\begin{minipage}[t]{0.081\linewidth}
			\centering
			\includegraphics[width=1\linewidth]{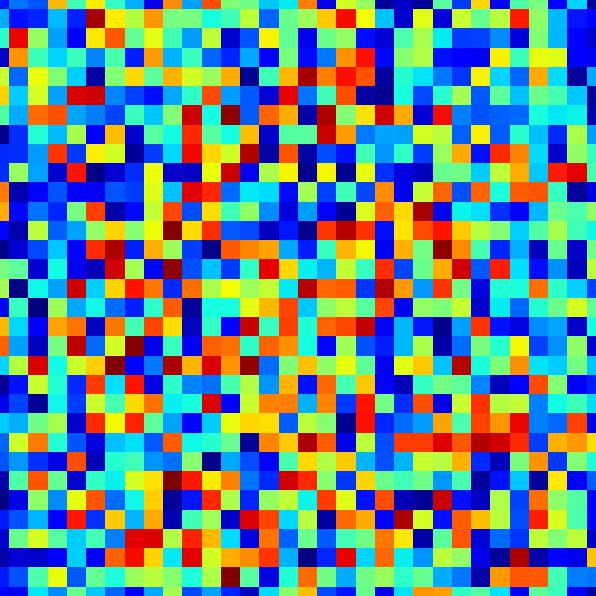}
		\end{minipage}
	}
	\subfigure[FN]{
		\begin{minipage}[t]{0.081\linewidth}
			\centering
			\includegraphics[width=1\linewidth]{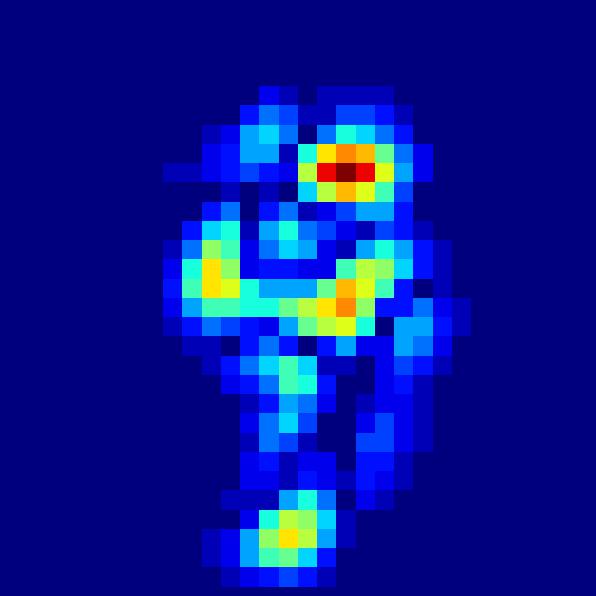}
		\end{minipage}
	}
	\caption{The saliency maps of images from MNIST, visualized at the converged epochs of all methods under $\Delta=0.8$.}
	\label{fig:fullpatterns}
\end{figure*}

\subsection{Datasets}\label{ap:data}
Detailed descriptions of all datasets are listed as follows:
\begin{itemize}
    \item MNIST \footnote{Downloaded from the torchvision datasets.} consists of 70,000 images of handwritten digits from ten classes. The features of each image are presented in $28 \times 28$ gray-scale pixels.
    \item FashionMNIST \footnote{Downloaded from the torchvision datasets.} consists of 70,000 images of fashion products from ten classes. The features of each image are presented in $28 \times 28$ gray-scale pixels.
    \item miniImageNet \footnote{Downloaded from the MLclf toolbox. \cite{xin_cao_2022_7233094}} consists of 60,000 images  sampled from the ImageNet dataset. It is evenly distributed across 100 classes, and we randomly select two classes to form the binary classification. The features of each image are presented in $84 \times 84$ RGB pixels.
    \item ADNI \footnote{Downloaded from http://www.loni.usc.edu/ADNI.} contains 412 brain scan images from two classes, \ie 186 Alzheimer's disease patients and 226 normal controls. Each brain scan image is pre-processed to obtain 90 features to form the features of a node in the graph.
\end{itemize}

\subsection{Comarison methods}\label{ap:cm}
Detailed descriptions of all comparison methods are listed as follows:
\begin{itemize}
    \item Convolutional neural network (CNN) and graph convolutional network (GCN) \cite{lecun2015lenet,welling2016semi} are the baseline (BL)  methods that aggregate the adjacent features to generate new feature representations for classification on image and graph datasets, respectively.
    \item Sample reweighting (SR) \cite{ren2018learning}) is a clean-sample-based method that denoises the noisy labels by reweighting their loss based on the gradient provided by a meta-network.
    \item Co-teaching (CT) \cite{han2018co} is a small-loss-sample-based method that trains two networks simultaneously and updates each network using the correctly labeled samples selected by the other network.
    \item Co-teaching+ \cite{yu2019does}  is a  small-loss-sample-based method that trains two networks simultaneously and updates each network by using the correctly labeled samples with prediction disagreement selected by the other network.
    \item Early stopping (ES) \cite{li2020gradient} is a  small-loss-sample-based method that implements the early stopping technique to stop the training process early and reserve the early robustness of DNNs. 
    \item Early learning regularization (ELR) \cite{liu2020early}  is a  small-loss-sample-based  method that denoises the noisy labels by regularizing the samples with large prediction variations between the early training stage and the late training stage.  ELR+ is a more advanced ELR algorithm  with two networks and mixup data augmentation \cite{zhang2018mixup}.
    \item Progressive early stopping (PES) \cite{bai2021understanding}  is a small-loss-sample-based  method that progressively stops each DNN layer's training to accommodate their different sensitivities to the noisy labels.
    \item Nested co-teaching (NCT) \cite{chen2021boosting} is a  small-loss-sample-based  method that combines the nested dropout technique with the co-teaching framework to discard the unimportant feature representations during training.
    % \item DivideMix (DM) \cite{li2019dividemix} is a robustness-based method that leverages the early-robustness of DNNs to divide the samples into a labeled set and an unlabeled set and trains two networks simultaneously using MixMatch technique.
    \item Limiting label information memorization in training (LIMIT) \cite{harutyunyan2020improving} is a bound-reducing method that constrains the mutual information between the model weights and the dataset by replacing the model weights with the gradient of the loss w.r.t the weights.
\end{itemize}

\subsection{Interpretability}\label{ss:int}

%%%int
To further strengthen the superiority of the FN method, we compare the interpretability between the FN methods and eight comparison methods (including BL and seven denoising-based methods, \ie SR, CT, CT+, ES, ELR, PES, and NCT) by
visualizing the saliency maps generated at the converged epochs on MNIST under $\Delta=0.8$. We visualize the results in Fig. \ref{fig:fullpatterns}.

It shows the FN method can generate  clearer and more interpretable patterns than all the comparison methods.
Specifically, the patterns learned by the FN method focus more on the foreground, where the image's semantic information lies, instead of the background, compared to BL, CT, CT+, ES, PES, and NCT.
Although ELR can also focus on the patterns in the foreground, it pays equal attention to all pixels of the object and fails to highlight the most important parts that are distinct from similar classes. For example, the most important difference between the handwritten digits ``4'' and ``9'' lies in the upper part of the digit.

The superior interpretability of the FN method over the denoising-based methods could be the result of using all samples in the training process, which makes full use of the features in the mislabeled samples and thus enhances the interpretability of DNNs in an unsupervised manner \cite{yuksel2021latentclr}.

\end{document}